\documentclass[10pt]{article} %
\usepackage[accepted]{tmlr}

\usepackage{amsmath,amsfonts,bm}

\def\eqref#1{equation~\ref{#1}}

\def\1{\bm{1}}

\def\vt{{\bm{t}}}

\def\mA{{\bm{A}}}
\def\mB{{\bm{B}}}

\def\mF{{\bm{F}}}

\def\mH{{\bm{H}}}

\def\mK{{\bm{K}}}

\def\mS{{\bm{S}}}
\def\mT{{\bm{T}}}

\def\mX{{\bm{X}}}

\def\mZ{{\bm{Z}}}

\DeclareMathAlphabet{\mathsfit}{\encodingdefault}{\sfdefault}{m}{sl}
\SetMathAlphabet{\mathsfit}{bold}{\encodingdefault}{\sfdefault}{bx}{n}

\newcommand{\KL}{D_{\mathrm{KL}}}
\newcommand{\Var}{\mathrm{Var}}

\newcommand{\Cov}{\mathrm{Cov}}

\DeclareMathOperator*{\argmax}{arg\,max}
\DeclareMathOperator*{\argmin}{arg\,min}

\usepackage{hyperref}
\usepackage{url}

\usepackage{color}
\usepackage{array}
\usepackage{float}
\usepackage{booktabs}
\usepackage{multirow}
\usepackage{amsmath}
\usepackage{amsthm}
\usepackage{ulem} 
\usepackage{caption}
\usepackage{subfig}
\usepackage{graphicx}
\usepackage{microtype}

\usepackage{authblk}

\providecommand{\tabularnewline}{\\}
\floatstyle{ruled}
\newfloat{algorithm}{tbp}{loa}
\providecommand{\algorithmname}{Algorithm}
\floatname{algorithm}{\protect\algorithmname}

\usepackage{thmtools}
\usepackage{thm-restate}
\usepackage{cleveref}

\usepackage{amsfonts}%
\usepackage{nicefrac}%
\usepackage{xcolor}%
\usepackage{amsthm}
\usepackage{amssymb}
\usepackage{bbm}
\usepackage{algorithmic}

\usepackage{listings}

\definecolor{codegreen}{rgb}{0,0.6,0}
\definecolor{codegray}{rgb}{0.5,0.5,0.5}
\definecolor{codepurple}{rgb}{0.58,0,0.82}
\definecolor{backcolour}{rgb}{0.95,0.95,0.92}

\lstdefinestyle{mystyle}{
    backgroundcolor=\color{backcolour},   
    commentstyle=\color{codegreen},
    keywordstyle=\color{magenta},
    numberstyle=\tiny\color{codegray},
    stringstyle=\color{codepurple},
    basicstyle=\ttfamily\footnotesize,
    breakatwhitespace=false,         
    breaklines=true,                 
    captionpos=b,                    
    keepspaces=true,                 
    numbers=left,                    
    numbersep=5pt,                  
    showspaces=false,                
    showstringspaces=false,
    showtabs=false,                  
    tabsize=2
}

\newcommand{\REVISE}[1]{\textcolor{black}{#1}}

\DeclareUnicodeCharacter{00B0}{ }

\title{Generating Adversarial Examples with Task Oriented Multi-Objective Optimization}

\author{\name Anh Bui \email tuananh.bui@monash.edu \\
   \addr Monash University
   \COAUTHOR
   \name Trung Le \email trunglm@monash.edu \\
   \addr Monash University
   \COAUTHOR
   \name He Zhao \email he.zhao@ieee.org \\
   \addr CSIRO's Data61, Australia 
   \COAUTHOR 
   \name Quan Tran \email qtran@adobe.com \\
   \addr Adobe Research 
   \COAUTHOR
   \name Paul Montague \email paul.montague@dst.defence.gov.au \\
   \addr Defence Science and Technology Group, Australia
   \COAUTHOR
   \name Dinh Phung \email dinh.phung@monash.edu \\
   \addr Monash University, VinAI Research 
}

\newcommand{\new}{\marginpar{NEW}}

\def\month{05}  %
\def\year{2023} %
\def\openreview{\url{https://openreview.net/forum?id=XXXX}} %

\begin{document}

\maketitle

\global\long\def\sidenote#1{\marginpar{\small\emph{{\color{Medium}#1}}}}%
\global\long\def\se{\hat{\text{se}}}%
\global\long\def\interior{\text{int}}%
\global\long\def\boundary{\text{bd}}%
\global\long\def\ML{\textsf{ML}}%
\global\long\def\GML{\mathsf{GML}}%
\global\long\def\HMM{\mathsf{HMM}}%
\global\long\def\support{\text{supp}}%
\global\long\def\new{\text{*}}%
\global\long\def\stir{\text{Stirl}}%
\global\long\def\mA{\mathcal{A}}%
\global\long\def\mB{\mathcal{B}}%
\global\long\def\expect{\mathbb{E}}%
\global\long\def\mF{\mathcal{F}}%
\global\long\def\mK{\mathcal{K}}%
\global\long\def\mH{\mathcal{H}}%
\global\long\def\mX{\mathcal{X}}%
\global\long\def\mZ{\mathcal{Z}}%
\global\long\def\mS{\mathcal{S}}%
\global\long\def\Ical{\mathcal{I}}%
\global\long\def\mT{\mathcal{T}}%
\global\long\def\Pcal{\mathcal{P}}%
\global\long\def\dist{d}%
\global\long\def\HX{\entro\left(X\right)}%
\global\long\def\entropyX{\HX}%
\global\long\def\HY{\entro\left(Y\right)}%
\global\long\def\entropyY{\HY}%
\global\long\def\HXY{\entro\left(X,Y\right)}%
\global\long\def\entropyXY{\HXY}%
\global\long\def\mutualXY{\mutual\left(X;Y\right)}%
\global\long\def\mutinfoXY{\mutualXY}%
\global\long\def\given{\mid}%
\global\long\def\gv{\given}%
\global\long\def\goto{\rightarrow}%
\global\long\def\asgoto{\stackrel{a.s.}{\longrightarrow}}%
\global\long\def\pgoto{\stackrel{p}{\longrightarrow}}%
\global\long\def\dgoto{\stackrel{d}{\longrightarrow}}%
\global\long\def\lik{\mathcal{L}}%
\global\long\def\logll{\mathit{l}}%
\global\long\def\bigcdot{\raisebox{-0.5ex}{\scalebox{1.5}{\ensuremath{\cdot}}}}%
\global\long\def\sig{\textrm{sig}}%
\global\long\def\likelihood{\mathcal{L}}%
\global\long\def\vectorize#1{\mathbf{#1}}%
\global\long\def\vt#1{\mathbf{#1}}%
\global\long\def\gvt#1{\boldsymbol{#1}}%
\global\long\def\idp{\ \bot\negthickspace\negthickspace\bot\ }%
\global\long\def\cdp{\idp}%
\global\long\def\das{}%
\global\long\def\id{\mathbb{I}}%
\global\long\def\idarg#1#2{\id\left\{  #1,#2\right\}  }%
\global\long\def\iid{\stackrel{\text{iid}}{\sim}}%
\global\long\def\bzero{\vt 0}%
\global\long\def\bone{\mathbf{1}}%
\global\long\def\a{\mathrm{a}}%
\global\long\def\ba{\mathbf{a}}%
\global\long\def\b{\mathrm{b}}%
\global\long\def\bb{\mathbf{b}}%
\global\long\def\B{\mathrm{B}}%
\global\long\def\boldm{\boldsymbol{m}}%
\global\long\def\c{\mathrm{c}}%
\global\long\def\C{\mathrm{C}}%
\global\long\def\d{\mathrm{d}}%
\global\long\def\D{\mathrm{D}}%
\global\long\def\N{\mathrm{N}}%
\global\long\def\h{\mathrm{h}}%
\global\long\def\H{\mathrm{H}}%
\global\long\def\bH{\mathbf{H}}%
\global\long\def\K{\mathrm{K}}%
\global\long\def\M{\mathrm{M}}%
\global\long\def\bff{\vt f}%
\global\long\def\bx{\mathbf{\mathbf{x}}}%
\global\long\def\bl{\boldsymbol{l}}%
\global\long\def\s{\mathrm{s}}%
\global\long\def\T{\mathrm{T}}%
\global\long\def\bu{\mathbf{u}}%
\global\long\def\v{\mathrm{v}}%
\global\long\def\bv{\mathbf{v}}%
\global\long\def\bo{\boldsymbol{o}}%
\global\long\def\bh{\mathbf{h}}%
\global\long\def\bs{\boldsymbol{s}}%
\global\long\def\x{\mathrm{x}}%
\global\long\def\bx{\mathbf{x}}%
\global\long\def\bz{\mathbf{z}}%
\global\long\def\hbz{\hat{\bz}}%
\global\long\def\z{\mathrm{z}}%
\global\long\def\y{\mathrm{y}}%
\global\long\def\bxnew{\boldsymbol{y}}%
\global\long\def\bX{\boldsymbol{X}}%
\global\long\def\tbx{\tilde{\bx}}%
\global\long\def\by{\mathbf{y}}%
\global\long\def\bY{\boldsymbol{Y}}%
\global\long\def\bZ{\boldsymbol{Z}}%
\global\long\def\bU{\boldsymbol{U}}%
\global\long\def\bn{\boldsymbol{n}}%
\global\long\def\bV{\boldsymbol{V}}%
\global\long\def\bI{\boldsymbol{I}}%
\global\long\def\J{\mathrm{J}}%
\global\long\def\bJ{\mathbf{J}}%
\global\long\def\w{\mathrm{w}}%
\global\long\def\bw{\vt w}%
\global\long\def\bW{\mathbf{W}}%
\global\long\def\balpha{\gvt{\alpha}}%
\global\long\def\bdelta{\boldsymbol{\delta}}%
\global\long\def\bsigma{\gvt{\sigma}}%
\global\long\def\bbeta{\gvt{\beta}}%
\global\long\def\bmu{\gvt{\mu}}%
\global\long\def\btheta{\boldsymbol{\theta}}%
\global\long\def\blambda{\boldsymbol{\lambda}}%
\global\long\def\bgamma{\boldsymbol{\gamma}}%
\global\long\def\bpsi{\boldsymbol{\psi}}%
\global\long\def\bphi{\boldsymbol{\phi}}%
\global\long\def\bpi{\boldsymbol{\pi}}%
\global\long\def\bomega{\boldsymbol{\omega}}%
\global\long\def\bepsilon{\boldsymbol{\epsilon}}%
\global\long\def\btau{\boldsymbol{\tau}}%
\global\long\def\bxi{\boldsymbol{\xi}}%
\global\long\def\realset{\mathbb{R}}%
\global\long\def\realn{\realset^{n}}%
\global\long\def\integerset{\mathbb{Z}}%
\global\long\def\natset{\integerset}%
\global\long\def\integer{\integerset}%
\global\long\def\natn{\natset^{n}}%
\global\long\def\rational{\mathbb{Q}}%
\global\long\def\rationaln{\rational^{n}}%
\global\long\def\complexset{\mathbb{C}}%
\global\long\def\comp{\complexset}%
\global\long\def\compl#1{#1^{\text{c}}}%
\global\long\def\and{\cap}%
\global\long\def\compn{\comp^{n}}%
\global\long\def\comb#1#2{\left({#1\atop #2}\right) }%
\global\long\def\param{\vt w}%
\global\long\def\Param{\Theta}%
\global\long\def\meanparam{\gvt{\mu}}%
\global\long\def\Meanparam{\mathcal{M}}%
\global\long\def\meanmap{\mathbf{m}}%
\global\long\def\logpart{A}%
\global\long\def\simplex{\Delta}%
\global\long\def\simplexn{\simplex^{n}}%
\global\long\def\dirproc{\text{DP}}%
\global\long\def\ggproc{\text{GG}}%
\global\long\def\DP{\text{DP}}%
\global\long\def\ndp{\text{nDP}}%
\global\long\def\hdp{\text{HDP}}%
\global\long\def\gempdf{\text{GEM}}%
\global\long\def\rfs{\text{RFS}}%
\global\long\def\bernrfs{\text{BernoulliRFS}}%
\global\long\def\poissrfs{\text{PoissonRFS}}%
\global\long\def\grad{\gradient}%
\global\long\def\gradient{\nabla}%
\global\long\def\partdev#1#2{\partialdev{#1}{#2}}%
\global\long\def\partialdev#1#2{\frac{\partial#1}{\partial#2}}%
\global\long\def\partddev#1#2{\partialdevdev{#1}{#2}}%
\global\long\def\partialdevdev#1#2{\frac{\partial^{2}#1}{\partial#2\partial#2^{\top}}}%
\global\long\def\closure{\text{cl}}%
\global\long\def\cpr#1#2{\Pr\left(#1\ |\ #2\right)}%
\global\long\def\var{\text{Var}}%
\global\long\def\Var#1{\text{Var}\left[#1\right]}%
\global\long\def\cov{\text{Cov}}%
\global\long\def\Cov#1{\cov\left[ #1 \right]}%
\global\long\def\COV#1#2{\underset{#2}{\cov}\left[ #1 \right]}%
\global\long\def\corr{\text{Corr}}%
\global\long\def\sst{\text{T}}%
\global\long\def\SST{\sst}%
\global\long\def\ess{\mathbb{E}}%
\global\long\def\Ess#1{\ess\left[#1\right]}%
\global\long\def\fisher{\mathcal{F}}%
\global\long\def\bfield{\mathcal{B}}%
\global\long\def\borel{\mathcal{B}}%
\global\long\def\bernpdf{\text{Bernoulli}}%
\global\long\def\betapdf{\text{Beta}}%
\global\long\def\dirpdf{\text{Dir}}%
\global\long\def\gammapdf{\text{Gamma}}%
\global\long\def\gaussden#1#2{\text{Normal}\left(#1, #2 \right) }%
\global\long\def\gauss{\mathbf{N}}%
\global\long\def\gausspdf#1#2#3{\text{Normal}\left( #1 \lcabra{#2, #3}\right) }%
\global\long\def\multpdf{\text{Mult}}%
\global\long\def\poiss{\text{Pois}}%
\global\long\def\poissonpdf{\text{Poisson}}%
\global\long\def\pgpdf{\text{PG}}%
\global\long\def\wshpdf{\text{Wish}}%
\global\long\def\iwshpdf{\text{InvWish}}%
\global\long\def\nwpdf{\text{NW}}%
\global\long\def\niwpdf{\text{NIW}}%
\global\long\def\studentpdf{\text{Student}}%
\global\long\def\unipdf{\text{Uni}}%
\global\long\def\transp#1{\transpose{#1}}%
\global\long\def\transpose#1{#1^{\mathsf{T}}}%
\global\long\def\mgt{\succ}%
\global\long\def\mge{\succeq}%
\global\long\def\idenmat{\mathbf{I}}%
\global\long\def\trace{\mathrm{tr}}%
\global\long\def\argmax#1{\underset{_{#1}}{\text{argmax}} }%
\global\long\def\argmin#1{\underset{_{#1}}{\text{argmin}\ } }%
\global\long\def\diag{\text{diag}}%
\global\long\def\norm{}%
\global\long\def\spn{\text{span}}%
\global\long\def\vtspace{\mathcal{V}}%
\global\long\def\field{\mathcal{F}}%
\global\long\def\ffield{\mathcal{F}}%
\global\long\def\inner#1#2{\left\langle #1,#2\right\rangle }%
\global\long\def\iprod#1#2{\inner{#1}{#2}}%
\global\long\def\dprod#1#2{#1 \cdot#2}%
\global\long\def\norm#1{\left\Vert #1\right\Vert }%
\global\long\def\entro{\mathbb{H}}%
\global\long\def\entropy{\mathbb{H}}%
\global\long\def\Entro#1{\entro\left[#1\right]}%
\global\long\def\Entropy#1{\Entro{#1}}%
\global\long\def\mutinfo{\mathbb{I}}%
\global\long\def\relH{\mathit{D}}%
\global\long\def\reldiv#1#2{\relH\left(#1||#2\right)}%
\global\long\def\KL{KL}%
\global\long\def\KLdiv#1#2{\KL\left(#1\parallel#2\right)}%
\global\long\def\KLdivergence#1#2{\KL\left(#1\ \parallel\ #2\right)}%
\global\long\def\crossH{\mathcal{C}}%
\global\long\def\crossentropy{\mathcal{C}}%
\global\long\def\crossHxy#1#2{\crossentropy\left(#1\parallel#2\right)}%
\global\long\def\breg{\text{BD}}%
\global\long\def\lcabra#1{\left|#1\right.}%
\global\long\def\lbra#1{\lcabra{#1}}%
\global\long\def\rcabra#1{\left.#1\right|}%
\global\long\def\rbra#1{\rcabra{#1}}%

\begin{abstract}
Deep learning models, even the-state-of-the-art ones, are highly vulnerable to adversarial examples. 
Adversarial training is one of the most efficient methods to improve the model's robustness. 
The key factor for the success of adversarial training is the capability to generate 
qualified and divergent adversarial examples which satisfy some objectives/goals 
(e.g., finding adversarial examples that maximize the model losses for simultaneously 
attacking multiple models). 
Therefore, multi-objective optimization (MOO) is a natural tool 
for adversarial example generation to achieve multiple objectives/goals simultaneously.
However, we observe that a naive application of MOO tends to maximize all objectives/goals equally, without caring 
if an objective/goal has been achieved yet. This leads to useless effort to further 
improve the goal-achieved tasks, while putting less focus on the goal-unachieved 
tasks. In this paper, we propose \emph{Task Oriented MOO} to address this issue, 
in the context where we can explicitly define the goal achievement for a task. 
Our principle is to only maintain the goal-achieved tasks, while letting the 
optimizer spend more effort on improving the goal-unachieved tasks. We conduct 
comprehensive experiments for our Task Oriented MOO on various adversarial 
example generation schemes. The experimental results firmly demonstrate 
the merit of our proposed approach. 
Our code is available at \url{https://github.com/tuananhbui89/TAMOO}.

\end{abstract}
\vspace{-5mm}
\section{Introduction}

Deep neural networks are powerful models that achieve impressive performance across various domains such as bioinformatics \citep{bio_infomatics}, speech recognition \citep{hinton_speech}, computer vision \citep{he2016deep}, and natural language processing \citep{vaswani2017attention}. Despite achieving state-of-the-art performance, these models are extremely fragile, as one can easily craft small and imperceptible adversarial perturbations of input data to fool them, hence resulting in high misclassifications \citep{szegedy2013intriguing, goodfellow2014explaining}. Accordingly, adversarial training (AT) \citep{madry2017towards, trades} has been proven to be one of the most efficient approaches to strengthen model robustness \citep{athalye2018obfuscated}. 
AT requires challenging models with divergent and qualified adversarial examples \citep{madry2017towards, trades, bui2021improving} so that the robustified models can defend against adversarial examples. 
Therefore, generating adversarial examples is an important research topic in Adversarial Machine Learning (AML).  
Several perturbation based attacks have been proposed, 
notably 
PGD \citep{madry2017towards}, 
CW \citep{carlini2017towards}, and AutoAttack \citep{croce2020reliable}. 
Most of them aim to optimize a single objective/goal, e.g., 
maximizing the cross-entropy (CE) loss w.r.t. the ground-truth label 
\citep{goodfellow2014explaining, madry2017towards}, maximizing the Kullback-Leibler 
(KL) divergence w.r.t. the predicted probabilities of a benign example \citep{trades}, 
or minimizing a combination of perturbation size and predicted loss to a targeted class as in \cite{carlini2017towards}. 

However, in many contexts, we need to find qualified adversarial examples satisfying multiple objectives/goals, e.g., 
finding an adversarial example that can \textit{attack simultaneously multiple models} in an 
ensemble model \citep{pang19a,bui2021improving}, finding an universal perturbation that 
can \textit{attack simultaneously multiple benign examples} \citep{moosavi2017universal}.
Obviously, these adversarial generations have a nature of multi-objective problem rather than a single-objective one. 
Consequently, using \textit{single-objective} adversarial examples leads to a much less adversarial robustness in 
ensemble learning as discussed in Section \ref{subsec:Adv-Train} and Appendix \ref{subsec:sup-adv-training}.

Multi-Objective Optimization (MOO) \citep{desideri2012multiple} is an optimization problem to find a Pareto optimality that aims to optimize multiple objective functions. In a nutshell, MOO is a natural tool for the aforementioned multi-objective adversarial generations. 
However, a direct and naive application of MOO to generating robust adversarial examples for multiple models or ensemble of transformations does not work satisfactorily (cf. Appendix \ref{sec:sup-discussions}).
Concretely, it can be observed that the tasks are not optimized equally. The optimizing process focuses too much on one dominating task and can be trapped easily by it, hence leading to downgraded attack performances. 

Intuitively, for multi-objective adversarial generations, we can explicitly investigate if an objective or a task achieves or fails to achieve its goal (e.g., the current adversarial example can fool a model successfully or unsuccessfully in multiple models). To avoid some tasks dominating others
during the optimization process,
we can favour more the tasks that are failing and pay less attention to the tasks that are performing well. For example, in the context of attacking multiple models, we update an adversarial example $x^a$ to favor the models that $x^a$ has not attacked successfully yet, while trying to maintain the attack capability of $x^a$ on the already successful models. In this way, we expect that no task really dominates others and all tasks can be updated equally to fulfill their goals.

Bearing this in mind, we propose a new framework named \textbf{\emph{TA}}\emph{sk Oriented }\textbf{\emph{M}}\emph{ulti-}\textbf{\emph{O}}\emph{bjective
}\textbf{\emph{O}}\emph{ptimization} (TA-MOO) with multi-objective adversarial generations as the demonstrating applications. Specifically, we learn a weight vector (i.e., each dimension is the  weight for a task) lying on a simplex corresponding to all tasks. 
To favor the unsuccessful tasks while maintaining the success of the successful ones, we propose a geometry-based regularization term that represents the distance between the original simplex and a reduced simplex which involves the weight vectors for the currently unsuccessful tasks only. Furthermore, along with the original quadratic term of the standard MOO helping to improve all tasks, minimizing our geometry-based regularization term encourages the weights of the goal-achieved tasks to be as small as possible, while inspiring those for the goal-unachieved ones to have a sum close to 1. 
By doing so, we aim to focus more on improving the goal-unachieved tasks, while still maintain the performance of goal-achieved tasks.

Most related work to ours is \cite{wang2021adversarial}, which considers the worst-case performance across all tasks. However, this original principle reduces the generalizability to other tasks. To mitigate this issue, a specific regularization was proposed to balance all tasks' weights. 
Our work, which casts an adversarial generation task as a multi-objective optimization problem, is conceptually different from that work, although both methods can be applied to similar tasks. 
Further discussion about relate work can be found in Appendix \ref{sec:related_work}. 

To summarize, our contributions in this work include:

\REVISE{\textbf{(C1)} We propose a novel framework called TA-MOO, which addresses the shortcomings of the original MOO when applied to multi-objective adversarial generation. Specifically, the TA-MOO framework incorporates a geometry-based regularization term that favors unsuccessful tasks, while simultaneously maintaining the performance of successful tasks.
This innovative approach improves the efficiency and efficacy of adversarial generation by promoting a more balanced exploration of the solution space.
}

\REVISE{
    \textbf{(C2)} We conduct comprehensive experiments for three adversarial generation tasks and one adversarial training task including attacking multiple models, learning universal perturbation, attacking over many data transformations, and adversarial training on ensemble learning setting. 
    The experimental results show that our TA-MOO outperforms the baselines by a wide margin on the three aforementioned adversarial generation tasks. 
    More importantly, our adversary brings a great benefit on improving adversarial robustness, highlighting the potential of our TA-MOO framework in adversarial machine learning.
}

\REVISE{\textbf{(C3)} Additionally, we provide a comprehensive analysis on different aspects of applying MOO and TA-MOO to adversarial generation tasks, such as the impact of the dominating issue in Appendix \ref{subsec:sup-discussion-dominating-issue}, the importance of the Task-Oriented regularization in Appendix \ref{subsec:sup-discussion-task-oriented-reg}, the impact of initialization of MOO in Appendix {subsec:optimal-init-moo}, and the limitations of MOO solver in Appendix {sec:sup-gradient-des-discuss}. We believe that our analysis would be beneficial for future research in this area.
}

\vspace{-3mm}
\section{Background}

We revisit the background of multi-objective optimization (MOO),
which lays the foundation for our task-oriented MOO in the sequel. Given
multiple objective functions $f\left(\delta\right):=\left[f_{1}\left(\delta\right),...,f_{m}\left(\delta\right)\right]$
where each $f_{i}:\mathbb{R}^{d}\goto\mathbb{R}$, we aim to find
the Pareto optimal solution that simultaneously 
maximizes all objective functions:
\begin{equation}
\max_{\delta}f(\delta):=\left[f_{1}\left(\delta\right),...,f_{m}\left(\delta\right)\right].\label{eq:moo}
\end{equation}

While there are a variety of MOO solvers \citep{miettinen2012nonlinear, ehrgott2005multicriteria}, in this paper, we adapt from the multi-gradient descent algorithm (MGDA) that was proposed suitably for end-to-end learning by \cite{desideri2012multiple}. 
Specifically, MGDA combines the gradients of individual objectives to a single optimal direction that increases all objectives simultaneously. 
The optimal direction corresponds to the minimum-norm point that can be found by solving the quadratic programming problem:
\begin{equation}
w^{*}=\text{argmin}_{w\in\simplex_{m}}w^{T}Qw, \label{eq:op_moo}
\end{equation}
where $\simplex_{m}=\left\{ \pi\in\mathbb{R}_{+}^{m}:\norm{\pi}_{1}=1\right\} $
is the $m$-simplex and $Q\in\mathbb{R}^{m\times m}$ is the matrix
with $Q_{ij}=\nabla_{\delta}f_{i}\left(\delta\right)^{T}\nabla_{\delta}f_{j}\left(\delta\right)$.
Finally, the solution of the problem \ref{eq:moo} can be found iteratively with each update step $\delta=\delta+\eta g$ where $g$ is the combined gradient $ g=\sum_{i=1}^{m}w_{i}^{*}\nabla_{\delta}f_{i}\left(\delta\right)$ and $\eta>0$ is a sufficiently small learning rate.
Furthermore, \cite{desideri2012multiple} also proved that by using
an appropriate learning rate at each step, we reach the Pareto optimality
point $\delta^{*}$ at which there exist $w\in\simplex_{m}$ such
that $\sum_{i=1}^{m}w_{i}\nabla_{\delta}f_{i}\left(\delta^{*}\right)=\bzero$.

\vspace{-2mm}
\section{Our Proposed Method}

\subsection{Task Oriented Multi-Objective Optimization}

We now present our \textbf{\emph{TA}}\emph{sk Oriented }\textbf{\emph{M}}\emph{ulti-}\textbf{\emph{O}}\emph{bjective
}\textbf{\emph{O}}\emph{ptimization }(TA-MOO). We consider the MOO problem in (1) 
where each task $\mathcal{T}_{i}\,(i=1,...,m)$ corresponds to the
objective function $f_{i}\left(\delta\right)\,(i=1,...,m)$. 
Additionally, assume that given a task $\mathcal{T}_{i}$, we can explicitly observe if this task has currently achieved its goal (e.g., the current adversarial example $x$ can fool successfully the model $f_{i}$), which is named a \emph{goal-achieved} task. We also name a task that has not achieved its goal a \emph{goal-unachieved} task.
Different from the standard MOO, which equally pays equal attention to all tasks, our TA-MOO focuses on improving the currently \emph{goal-unachieved} tasks, while trying to maintain the performance of the \emph{goal-achieved} tasks. By this principle, we expect all tasks would be equally improved to simultaneously achieve their goals. 

To be more precise, we depart from $\delta_{0}$ and consecutively
update in $L$ steps to obtain the sequence $\delta_{1},\delta_{2},...,\delta_{L}$
that approaches the optimal solution. Considering the $t$-th step (i.e., $1\le t \le L$),
we currently have $\delta_{t}$ and need to update it to obtain $\delta_{t+1}$.
We examine the tasks that have achieved their goals already and denote
them as $\mathcal{T}_{1},\mathcal{T}_{2},...,\mathcal{T}_{s}$ without
the loss of generalization. Here we note that the list of \emph{goal-achieved}
tasks is empty if $s=0$ and the list of \emph{goal-unachieved}
tasks is empty if $s=m$. 
Specifically, to find $\delta_{t+1}$, we first solve the following
optimization problem (OP):
\begin{equation}
w^{*}=\text{argmin}_{w\in\simplex_{m}}\left\{ w^{T}Qw+\lambda\Omega\left(w\right)\right\} ,\label{eq:op_reg}
\end{equation}
where $Q\in\mathbb{R}^{m\times m}$ with $Q_{ij}=\nabla_{\delta}f_{i}\left(\delta_{t}\right)^{T}\nabla_{\delta}f_{j}\left(\delta_{t}\right)$,
$\lambda>0$ is a trade-off parameter, and $\Omega\left(w\right)$
is a regularization term to let the weights focus more on the \emph{goal-unachieved}
tasks. We next compute the combined gradient $g_{t}$ and update $\delta_{t}$
as:
\[
g_{t}=\sum_{i=1}^{m}w_{i}^{*}\nabla_{\delta}f_{i}\left(\delta_{t}\right) \text{ and } \delta_{t+1}=\delta_{t}+\eta g_{t}.
\]

The OP in (\ref{eq:op_reg}) consists of two terms.
The first term $w^{T}Qw$ ensures that all tasks are improving, while
the second term $\Omega\left(w\right)$ serves as the regularization
to restrict the goal-achieved tasks $\mathcal{T}_{1},...,\mathcal{T}_{s}$
by setting the corresponding weights $w_{1},...,w_{s}$ as small
as possible. 

Before getting into the details of the regularization, we emphasize that to impose the constraint $w\in\simplex_m$, we parameterize $w=\text{softmax}\left(\alpha\right)$
with $\alpha\in\mathbb{R}^{m}$ and solve the OP in
(\ref{eq:op_reg}) using gradient descent. In what follows, we discuss
our proposed 
geometry-based regularization term $\Omega\left(w\right)$.

\paragraph{Simplex-based regularization.}

Let $\mathcal{S}_{u}=\left\{ \beta=\left[\beta_{i}\right]_{i=s+1}^{m}\in\mathbb{R}_{+}^{m-s}:\sum_{i=s+1}^{m}\beta_{i}=1\right\} $
be a simplex w.r.t. the \emph{goal-unachieved} tasks and $\mathcal{S}=\left\{ \bzero_{s}\right\} \times\mathcal{S}_{u}$
be the extended simplex, where $\bzero_{s}$ is the $s$-dimensional vector of all
zeros. We define the regularization term $\Omega\left(w\right)$ as
the distance from $w$ to the extended simplex $\mathcal{S}$:
\begin{equation}
\Omega\left(w\right)=d\left(w,\mathcal{S}\right)=\text{min}_{\pi\in\mathcal{S}}\norm{w-\pi}_{2}^{2}.\label{eq:simplex_reg}
\end{equation}

Because $\mathcal{S}$ is a compact and convex set and $\norm{w-\pi}_{2}^{2}$
is a differentiable and convex function, the optimization problem in
(\ref{eq:simplex_reg}) has a unique global minimizer $\Omega\left(w\right)=\norm{w-\text{proj}_{\mathcal{S}}\left(w\right)}_{2}^{2}$,
where the projection $\text{proj}_{\mathcal{S}}\left(w\right)$ is
defined as
\[
\text{proj}_{\mathcal{S}}\left(w\right)=\text{argmin}_{\pi\in\mathcal{S}}\norm{w-\pi}_{2}^{2}.
\]
The following lemma shows us how to find the projection $\text{proj}_{\mathcal{S}}\left(w\right)$
and evaluate $\Omega\left(w\right)$.
\begin{restatable}{lem}{lemmaone}\label{lem:projection}
    Sorting $w_{s+1:m}$ into $u_{s+1:m}$ such that $u_{s+1}\geq u_{s+2}\geq...\geq u_{m}$. Defining $\rho=\max\left\{ s+1\leq i\leq m:u_{i}+\frac{1}{i-s}\left(1-\sum_{j=s+1}^{i}u_{j}\right)>0\right\} $.
    Denoting $\gamma=\frac{1}{\rho}\left(1-\sum_{i=s+1}^{\rho}u_{i}\right)$,
    the projection $\text{proj}_{\mathcal{S}}\left(w\right)$ can be computed
    as
    \[
    \text{proj}_{\mathcal{S}}\left(w\right)_{i}=\begin{cases}
    0 & 1\leq i\leq s\\
    \max\left\{ w_{i}+\gamma,0\right\}  & \text{otherwise}
    \end{cases}
    \]
    \vspace{-2mm}
    Furthermore, the regularization $\Omega\left(w\right)$ has the form:
    \begin{equation}
    \Omega\left(w\right)=\sum_{i=1}^{s}w_{i}^{2}+\sum_{i=s+1}^{m}\left(w_{i}-\max\left\{ w_{i}+\gamma,0\right\} \right)^{2}.\label{eq:reg_simplex}
    \end{equation}    
\end{restatable}
\vspace{-0mm}
With further algebraic manipulations, $\Omega\left(w\right)$ can be significantly simplified as shown in Theorem \ref{lem:sim_op}.
\begin{restatable}{thm}{lemmatwo}\label{lem:sim_op}
    The regularization $\Omega\left(w\right)$ has the following closed-form:
    \begin{equation}
    \Omega\left(w\right)=\sum_{i=1}^{s}w_{i}^{2}+\frac{1}{m-s}\left(1-\sum_{i=s+1}^{m}w_{i}\right)^{2}.\label{eq:reg_simplex-1}
    \end{equation} 
\end{restatable}

The proof of Lemma \ref{lem:projection} and Theorem \ref{lem:sim_op} can be found in Appendix \ref{sec:sup-proof}.
Evidently, the regularization term in Eq. (\ref{eq:reg_simplex-1})
in Theorem \ref{lem:sim_op} encourages the weights $w_{1:s}$ associated with
the \emph{goal-achieved} tasks to be as small as possible and
the weights $w_{s+1:m}$associated with the \emph{goal-unachieved} tasks
to move closer to the simplex $\mathcal{S}_{u}$ (i.e., $\sum_{i=s+1}^{m}w_{i}$
is closer to $1$).

\paragraph{Parameterized TA-MOO.}
Algorithm \ref{alg:bilevel-tamoo} summarizes the key steps of our TA-MOO. 
We use gradient descent to find solution $\delta$ for the OP \ref{eq:moo} in $L$ steps and at each iteration we solve the OP in \ref{eq:op_reg} in $K$ steps using gradient descent solver with the parameterization $w=\text{softmax}\left(\alpha\right)$. 
To reduce computational cost, at each iteration we reuse the previous solution $\alpha$ and use a few steps $K$ (i.e., $K \leq 10$) to get new solution. 
We then compute the combined gradient $g_t$ and finally update $\delta_{t}$ to $\delta_{t+1}$ using the combined gradient $g_t$ (or $\text{sign}(g_t)$ in the case of $L_\infty$ norm).
The projecting operation in step 13 is to project $\delta$ to a valid space specifying to applications that we introduce hereon.

\begin{algorithm}[h]
\begin{algorithmic}[1]

\REQUIRE Multi-objective functions $f_{1:m}\left(\delta\right)$. $\delta$'s solver with L update steps and learning rate $\eta_\delta$. 
$w$'s Gradient Descent Solver (GD) with K update steps and learning rate $\eta_w$ and variable $\alpha$. The softmax function denotes by $\sigma$. Tradeoff parameter $\lambda$. 

\ENSURE The optimal solution $\delta^{*}$.

\STATE Initialize $\delta_{0}$ (e.g., $\delta_{0} \sim \mathcal{U}(-\epsilon, \epsilon)$).

\STATE Initialize $\alpha_0=[\alpha_0^i]_{i=1}^m$ with $\alpha_0^i=1/m$.

\FOR {$t=0$ to $L-1$}

\STATE Collect list of tasks' gradients $\{ \nabla_{\delta} f_i (\delta_t) \}_{i=1}^m$.

\STATE Compute $Q$ with $Q_{ij}=\nabla_{\delta}f_{i}\left(\delta_{t}\right)^{T}\nabla_{\delta}f_{j}\left(\delta_{t}\right)$.

\STATE Initialize $\alpha_{t+1}=\alpha_t$
\FOR {$k=0$ to $K-1$}
\STATE Compute $\mathcal{L}(\alpha_{t+1})=\sigma(\alpha_{t+1})^T Q \sigma(\alpha_{t+1}) + \lambda \Omega (\sigma(\alpha_{t+1}))$.  

\STATE Update $\alpha_{t+1} = \alpha_{t+1} - \eta_w \nabla_\alpha \mathcal{L}(\alpha_{t+1})$.
\ENDFOR

\STATE Compute the combined gradient $g_{t}=\sum_{i=1}^{m} \sigma(\alpha_{t+1,i}) \nabla_{\delta}f_{i}\left(\delta_{t}\right)$.

\STATE Update $\delta_{t+1}=\delta_{t} + \eta_\delta g_{t}$.

\STATE Project $\delta_{t+1}$ to a valid space (specific to domain, e.g., $\norm{\delta} \leq \epsilon$). 

\ENDFOR

\STATE Output $\delta^{*}=\delta_{L}$.

\end{algorithmic}

\caption{Pseudocode for Parameterized TA-MOO.\label{alg:bilevel-tamoo}}
\end{algorithm}

\subsection{Applications in Adversarial Generation}
Although TA-MOO is a general framework, we in this paper focus on its applications in adversarial generation. Following~\cite{wang2021adversarial}, we consider three tasks of generating adversarial examples.
\paragraph{Generating adversarial examples for an ensemble model.}
Considering an ensemble classifier with multiple classification models $h_{1},h_{2},...,h_{m}$,
where $h_{i}\left(x\right)\in\simplex_{M}=\left\{ \pi\in\mathbb{R}_{+}^{M}:\norm{\pi}_{1}=1\right\} $
with the number of classes $M$.
Given a data sample $x$, our aim is to find an adversarial example $x^{a}=x+\delta$ that can successfully attack all the models.
Specifically, we consider a set of tasks each of which, $\mathcal{T}_{i}$, is about whether $x+\delta$ can successfully attack model $h_{i}$, defined as:
\[
\mathbb{I}\left\{ \text{argmax}_{1\leq k\leq M}h_{i}\left(x+\delta,k\right)\neq y\right\},
\]
where $y$ is the ground truth label of $x$,
$\mathbb{I}$ is the indicator function and $h_{i}\left(x,k\right)$
returns the probability to predict $x$ to the class $k$.
To find a perturbation $\delta$ that can attack successfully all
models, we solve the following multi-objective optimization problem:
\[
\max_{\delta:\norm{\delta}\leq\epsilon}\left[f_{1}\left(\delta\right),...,f_{m}\left(\delta\right)\right],
\]
where $f_{i}\left(\delta\right)=\ell\left(h_{i}\left(x+\delta\right),y\right)$
with the loss function $\ell$ which could be the cross-entropy (CE) loss
\citep{madry2017towards}, the Kullback-Leibler (KL) loss \citep{trades},
or the Carlini-Wagner (CW) loss \citep{carlini2017towards}. 

\paragraph{Generating universal perturbations.}
Considering a single classification model $h$ with $h\left(x\right)\in\simplex_{M}$
and a batch of data samples $x_{1},x_{2},...,x_{B}$, we would like to
find a perturbation $\delta$ with $\norm{\delta}\leq\epsilon$
such that $x_{i}^{a}=x_{i}+\delta,i=1,...,B$, are adversarial examples. 
We define the task $\mathcal{T}_{i}$ as finding the adversarial example
$x_{i}^{a}=x_{i}+\delta$ for data sample $x_{i}$. For each task $\mathcal{T}_{i}$,
we can define its goal as finding successfully the adversarial
example $x_{i}^{a}$:
\[
\mathbb{I}\left\{ \text{argmax}_{1\leq k\leq M}h\left(x_{i}^{a},k\right)\neq\text{argmax}_{1\leq k\leq M}h\left(x_{i},k\right)\right\} .
\]
To find the perturbation $\delta$, we solve the following
multi-objective optimization problem:
\[
\max_{\delta:\norm{\delta}\leq\epsilon}\left[f_{1}\left(\delta\right),...,f_{m}\left(\delta\right)\right],
\]
where $f_{i}\left(\delta\right)=\ell\left(h\left(x_{i}^{a}\right),y_{i}\right)=\ell\left(h\left(x_{i}+\delta\right),y_{i}\right)$
with $y_{i}$ the ground-truth label of $x_{i}$.

\paragraph{Generating adversarial examples against transformations.}
Considering a single classification model $h$ and $m$ categories
of data transformation $\mathcal{P}_{1:m}$ (e.g., rotation, lighting,
and translation). Our goal is to find an adversarial attack that
is robust to these data transformations. Specifically, given a benign example
$x$, we would like to learn a perturbation $\delta$ with $\norm{\delta}\leq\epsilon$ that can successfully attack the model after any transformation $t_{i}\sim\mathcal{P}_{i}$ is applied.
To formulate as an MOO problem, we consider the task $\mathcal{T}_{i}$
as finding the adversarial example $x_{i}^{a}=t_{i}\left(x+\delta\right)$
with $t_{i}\sim\mathcal{P}_{i}$. For each task $\mathcal{T}_{i}$,
we can define the goal as finding successfully the adversarial
example $x_{i}^{a}$:
\[
\mathbb{I}\left\{ \text{argmax}_{1\leq k\leq M}h\left(x_{i}^{a},k\right)\neq\text{argmax}_{1\leq k\leq M}h\left(x,k\right)\right\} .
\]
To find the perturbation $\delta$, we solve the following multi-objective
optimization problem:
\[
\max_{\delta:\norm{\delta}\leq\epsilon}\left[f_{1}\left(\delta\right),...,f_{m}\left(\delta\right)\right],
\]
where $f_{i}\left(\delta\right)=\mathbb{E}_{t_{i}\sim\mathcal{P}_{i}}\left[\ell\left(h\left(t_{i}\left(x+\delta\right)\right),y\right)\right]$
with $y$ the ground-truth label of $x$.

\section{Experiments}

In this section, we provide extensive experiments across four settings:
(i) generating adversarial examples for ensemble of models (ENS, Sec \ref{subsec:Ens}), (ii) generating universal perturbation (UNI, Sec \ref{subsec:Uni})
, (iii) generating robust adversarial examples against Ensemble
of Transformations (EoT, Sec \ref{subsec:EoT}), 
and (iv) adversarial training for ensemble of models (AT, Sec \ref{subsec:Adv-Train}). 
The details of each setting can be found in Appendix \ref{sec:sup-exp-setting}.

\textbf{General settings.}
Through our experiments, we use six common architectures for the classifier including
ResNet18 \citep{he2016deep}, VGG16 \citep{simonyan2014very}, GoogLeNet
\citep{szegedy2015going}, EfficientNet \citep{tan2019efficientnet}, 
MobileNet \cite{howard2017mobilenets}, and WideResNet \cite{zagoruyko2016wide} 
with the implementation \footnote{https://github.com/kuangliu/pytorch-cifar}.
We evaluate on the full testing set of two benchmark datasets
which are CIFAR10 and CIFAR100 \citep{cifar10}. We observed that the
attack performance is saturated with standard training models. Therefore,
to make the job of adversaries more challenging, we use Adversarial
Training with PGD-AT \citep{madry2017towards} to robustify the models
and use these robust models as the victim models in our experiments. 

\begin{table}
    \centering{}\caption{Evaluation of Attacking Ensemble model on the CIFAR10 and CIFAR100 datasets.}
    \label{tab:ens-main-res}
    \begin{tabular}{llcccccc}
        &  & \multicolumn{2}{c}{CW} & \multicolumn{2}{c}{CE} & \multicolumn{2}{c}{KL}\tabularnewline
        &  & \textcolor{blue}{A-All} & A-Avg & \textcolor{blue}{A-All} & A-Avg & \textcolor{blue}{A-All} & A-Avg\tabularnewline
    \midrule 
    \multirow{4}{*}{CIFAR10} & Uniform & \textcolor{blue}{26.37} & \textbf{41.13} & \textcolor{blue}{28.21} & 48.34 & \textcolor{blue}{17.44} & \uuline{32.85}\tabularnewline
        & MinMax & \textcolor{blue}{\uuline{27.53}} & \uuline{41.20} & \textcolor{blue}{\uuline{35.75}} & \textbf{51.56} & \textcolor{blue}{\uuline{19.97}} & \textbf{33.13}\tabularnewline
        & MOO & \textcolor{blue}{18.87} & 34.24 & \textcolor{blue}{25.16} & 44.76 & \textcolor{blue}{15.69} & 29.54\tabularnewline
        & TA-MOO & \textbf{\textcolor{blue}{30.65}} & 40.41 & \textbf{\textcolor{blue}{38.01}} & \uuline{51.10} & \textbf{\textcolor{blue}{20.56}} & 31.42\tabularnewline
    \midrule
    \multirow{4}{*}{CIFAR100} & Uniform & \textcolor{blue}{52.82} & \textbf{67.39} & \textcolor{blue}{55.86} & 72.62 & \textcolor{blue}{38.57} & \textbf{54.88}\tabularnewline
        & MinMax & \textcolor{blue}{\uuline{54.96}} & 66.92 & \textcolor{blue}{\uuline{63.70}} & \uuline{75.44} & \textcolor{blue}{\uuline{40.67}} & \uuline{53.83}\tabularnewline
        & MOO & \textcolor{blue}{51.16} & 65.87 & \textcolor{blue}{58.17} & 73.19 & \textcolor{blue}{39.18} & 53.44\tabularnewline
        & TA-MOO & \textbf{\textcolor{blue}{55.73}} & \uuline{67.02} & \textbf{\textcolor{blue}{64.89}} & \textbf{75.85} & \textbf{\textcolor{blue}{41.97}} & 53.76\tabularnewline
    \bottomrule
    \end{tabular}
    \vspace{-0mm}
\end{table}

\textbf{Evaluation metrics.}
We use three metrics to evaluate the attack performance including
(i) $\text{A-All}$:
the Attack Success Rate (ASR) when an adversarial example can achieve goals
in all tasks. This is considered as the most important metric to indicate how well one method can achieve in all tasks; 
(ii)$\text{A-Avg}$: the average Attack Success Rate over all
tasks which indicate the average attacking performance; 
(iii)$\{\text{A-i}\}_{i=1}^{K}$: Attack Success Rate in each
individual task. 
For reading comprehension purposes, if necessary  the highest/second highest performance 
in each experimental setting is highlighted in \textbf{Bold}/\uuline{Underline} 
and the \textcolor{blue}{most important metric(s)} is emphasized in blue color.

\textbf{Baseline methods.}
We compare our method with the \textbf{Uniform} strategy which assigns the same weight for all tasks and the \textbf{MinMax} method \citep{wang2021adversarial} which examines only the worst-case performance across all tasks. To increase the generality to other tasks, MinMax requires a regularization to balance between the average and the worst-case performance. 
We use the same attack setting for all methods: the attack is the $L_\infty$ untargeted attack with 100 steps, step size $\eta_\delta=2/255$ and perturbation limitation $\epsilon=8/255$. 
The GD solver in TA-MOO uses 10 steps with learning rate $\eta_w=0.005$. 
Further detail can be found in Appendix \ref{sec:sup-exp-setting}. 
\subsection{Adversarial Examples for Ensemble of Models (ENS) \label{subsec:Ens}}

\textbf{Experimental setting.}
In our experiment, we use an ensemble of four adversarially trained
models: ResNet18, VGG16, GoogLeNet, and EfficientNet. The architecture
is the same for both the CIFAR10 and CIFAR100 datasets except for
the last layer which corresponds with the number of classes in each
dataset. The final output of the ensemble is an average of the probability
outputs (i.e., output of the softmax layer). We use three different
losses as an object for generating adversarial examples including
CE \citep{madry2017towards}, 
KL \citep{trades}, and CW \citep{carlini2017towards}.

\textbf{\textit{Results 1: TA-MOO achieves the best performance.}}
Table \ref{tab:ens-main-res} shows the results of attacking the ensemble model on the CIFAR10 and CIFAR100 datasets. 
It can be seen that TA-MOO significantly outperforms the baselines and achieves the best performance in all the settings. For example, the improvement over the Uniform strategy is around 10\% on both datasets with the CE loss. Comparing to the MinMax method, the biggest improvement is around 3\% for CIFAR10 with CW loss and the lowest one is around 0.6\% with the KL loss. The improvement can be observed in all the settings, showing the generality of the proposed method.

\textbf{\textit{Results 2: When does not MOO work?}}
It can be observed that MOO falls behind all other methods, even compared with the Uniform strategy. Our hypothesis for the failure of MOO is that in the original setting with an ensemble of 4 diverse architectures (i.e., ResNet18, VGG16, GoogLeNet, and EfficientNet) there is one task that dominates the others and makes MOO become trapped (i.e., focusing on improving the dominant task). 
To verify our hypothesis, we measure the gradient norm $\| \nabla_{\delta} f_i(\delta) \|$ corresponding to each model and the final weight $w$ of 1000 samples and report the results in Table \ref{tab:ENS-non-diverse}. 
It can be seen that the EfficientNet has a much lower gradient strength, therefore, it has a much higher weight. This explains the highest ASR observed in EfficientNet and the large gap of ~19\% (56.11\% in EfficientNet and 37.05\% in GoogLeNet). 
To further confirm our hypothesis, we provide an additional experiment on a non-diverse ensemble model which consists of 4 individual ResNet18 models. It can be observed that in the non-diverse setting, the gradient strengths are more balanced across models, indicating that no task dominates others. As a result, MOO shows its effectiveness by outperforming the Uniform strategy by 4.3\% in A-All. 

\textbf{\textit{Results 3: The importance of the Task-Oriented regularization.}}
It can be observed from Table \ref{tab:ENS-non-diverse} that in the diverse setting, TA-MOO has a much lower gap (4\%) between the highest ASR (53.4\% at EfficientNet) and the lowest one (49.29\% at GoogLeNet) compared to MOO (~19\%). Moreover, while the ASR of EfficientNet is lower by ~2.7\%, the ASRs of all other architectures have been improved considerably (i.e., ~12\% in GoogLeNet). This improvement shows the importance of the Task-Oriented regularization, which helps to avoid being trapped by one dominating task, as happened in MOO. For the non-diverse setting, when no task dominates others, TA-MOO still shows its effectiveness when improving the ASR in all tasks by around 5\%. The significant improvement can be observed in all settings (except the setting on EfficientNet with the CIFAR10 dataset) as shown in Table \ref{tab:ens-main-res}, and demonstrates the generality of the Task-Oriented regularization.

\begin{table}
\centering{}\caption{Attacking Ensemble model with a diverse set D=\{R-ResNet18, V-VGG16, G-GoogLeNet, E-EfficientNet\} and non-diverse
set ND=\{4 ResNets\}. $w$ represents the final $w$ of MOO (mean $\pm$ std). $\| \nabla_{\delta} f_i(\delta) \|$ represents the gradient norm of each model (mean $\pm$ std).}
\label{tab:ENS-non-diverse}
\begin{tabular}{llcccccc}
 &  & \textcolor{blue}{A-All} & A-Avg & R/R1 & V/R2 & G/R3 & E/R4\tabularnewline
\midrule 
\multirow{3}{*}{D} 
& $\| \nabla_{\delta} f_i(\delta) \| $ & - & - & 7.15 $\pm$ 6.87 & 4.29 $\pm$ 4.64 & 7.35 $\pm$ 7.21 & 0.98 $\pm$ 0.72\tabularnewline
& $w$ & - & - & 0.15 $\pm$ 0.14 & 0.17 $\pm$0.13 & 0.15 $\pm$ 0.14 & 0.53 $\pm$ 0.29\tabularnewline
& Uniform & \textcolor{blue}{28.21} & 48.34 & 48.89 & 49.08 & 48.38 & 47.03\tabularnewline
 & MOO & \textcolor{blue}{25.16} & 44.76 & 39.06 & 46.83 & 37.05 & 56.11\tabularnewline
 & TA-MOO & \textcolor{blue}{38.01} & 51.10 & 49.55 & 52.15 & 49.29 & 53.40\tabularnewline
 \midrule 
\multirow{3}{*}{ND} 
& $\| \nabla_{\delta} f_i(\delta) \| $ & - & - & 8.41 $\pm$ 8.22 & 6.68$\pm$ 6.95 & 7.36 $\pm$ 6.03 & 5.67 $\pm$ 6.09\tabularnewline
& $w$ & - & - & 0.23 $\pm$ 0.21 & 0.24$\pm$0.17 & 0.23 $\pm$ 0.19 & 0.30 $\pm$ 0.21\tabularnewline
& Uniform & \textcolor{blue}{28.17} & 48.75 & 51.94 & 45.55 & 54.15 & 43.34\tabularnewline
 & MOO & \textcolor{blue}{32.50} & 52.21 & 53.25 & 49.05 & 56.80 & 49.76\tabularnewline
 & TA-MOO & \textcolor{blue}{41.01} & 57.33 & 58.88 & 55.32 & 60.81 & 54.29\tabularnewline
\bottomrule
\end{tabular}
\end{table}

\textbf{\textit{Results 4: TA-MOO achieves the best transferability on a diverse set of ensembles.}}

Table \ref{tab:ens-transferability} reports the SAR-All metric of transferred adversarial examples 
crafted from a source ensemble (RME) on attacking target ensembles (e.g., RMEVW is an ensemble of 5 models).
A higher number indicates a higher success rate of attacking a target model, therefore, 
also implies a higher transferability of adversarial examples. 
It can be seen that our TA-MOO adversary achieves the highest attacking performance on the whitebox attack setting, 
with a huge gap of 9.24\% success rate over the Uniform strategy. Our method also achieves 
the highest transferability regardless diversity of a target ensemble. 
More specifically, on target models such as REV, MEV, and RMEV, where members in the source ensemble (RME) 
are also in the target ensemble, our TA-MOO significantly outperforms the Uniform strategy, 
with the highest improvement is 5.19\% observed on target model RMEV. 
On the target models EVW and MVW which are less similar to the source model, 
our method still outperforms the Uniform strategy by 1.46\% and 1.65\%. 
The superior performance of our adversary on the transferability shows another benefit 
of using multi-objective optimization in generating adversarial examples. 
By reaching the intersection of all members' adversarial regions, our adversary is capable  
to generate a common vulnerable pattern on an input image shared across architectures, 
therefore, increasing the transferability of adversarial examples. 
More discussion can be found in Appendix \ref{subsec:sup-ens-exp}. 

    \begin{table}
    \caption{Evaluation on the Transferability of adversarial examples. Each cell
    (row-ith, column-jth) reports SAR (higher is better) of adversarial
    examples from the same source architecture (RME) with an adversary
    at row-ith to attack an ensemble at column-jth. Each architecture
    has been denoted by symbols such as R: ResNet18, M: MobileNet, E:
    EfficientNet, V: VGG16, W: WideResNet. For examples, RME represents
    for an ensemble of ResNet18, MobileNet and EfficientNet.\label{tab:ens-transferability}}
    
    \centering{}%
    \begin{tabular}{lcccccccc}
     & RME & RVW & EVW & MVW & REV & MEV & RMEV & RMEVW\tabularnewline
    \midrule 
    Uniform & 31.73 & \uuline{25.03} & 22.13 & 22.73 & 29.50 & 28.44 & 26.95 & 20.50\tabularnewline
    MinMax & \uuline{40.01} & 23.75 & 22.39 & 23.34 & \uuline{32.57} & \uuline{32.75} & \uuline{31.85} & \uuline{21.99}\tabularnewline
    MOO & 35.20 & 24.25 & \uuline{22.94} & \uuline{23.76} & 30.65 & 32.28 & 29.49 & 21.77\tabularnewline
    TA-MOO & \textbf{40.97} & \textbf{25.13} & \textbf{23.59} & \textbf{24.38} & \textbf{33.00} & \textbf{33.05} & \textbf{32.14} & \textbf{23.04}\tabularnewline
    \bottomrule
    \end{tabular}
\end{table}

\subsection{Adversarial Training with TA-MOO for Ensemble of Models (ENS) \label{subsec:Adv-Train}}

We conduct adversarial training with adversarial examples generated
by MOO and TA-MOO attacks to verify the quality of these adversarial
examples and report results on Table \ref{tab:ens-adv-training}. The detailed setting and more experimental results can be found in Appendix \ref{subsec:sup-adv-training}.
\textbf{Result 1: Reducing transferability. } It can be seen that the SAR-All of MOO-AT and 
TA-MOO-AT are much lower than that on other methods. More specifically, the gap of SAR-All between 
PGD-AT and TA-MOO-AT is (5.33\%) 6.13\% on the (non) diverse setting. 
The lower SAR-All indicating that adversarial examples are harder to transfer among ensemble members on the TA-MOO-AT model than on the PGD-AT model.
\textbf{Result 2: Producing more robust single members. } The comparison of average SAR shows that adversarial training with TA-MOO 
produces more robust single models than PGD-AT does. More specifically, the average robust accuracy (measured by 100\% - \textit{A-Avg}) 
of TA-MOO-AT is 32.17\%, an improvement of 6.06\% over PGD-AT in the non-diverse setting, while 
there is an improvement of 4.66\% in the diverse setting.
\textbf{Result 3: Adversarial training with TA-MOO achieves the best robustness. }
More specifically, on the non-divese setting, TA-MOO-AT achives 38.22\% robust accuracy, an improvement 
of 1\% over MinMax-AT and 5.44\% over standard PGD-AT. On the diverse setting, the improvement
over MinMax-AT and PGD-AT are 0.9\% and 4\%, respectively. The root of the improvement is the ability to generate stronger adversarial examples in the
the sense that they can challenge not only the entire ensemble model but also
all single members. These adversarial examples lie in the joint insecure
region of members (i.e., the low confidence region of multiple classes),
therefore, making the decision boundaries more separate. 
As a result, adversarial training with TA-MOO produces more robust 
single models (i.e., lower SAR-Avg) and significantly reduces 
the transferability of adversarial examples among members (i.e., lower SAR-All). 
These two conditions explain the best ensemble adversarial robustness achieved by TA-MOO. 

\begin{table}
\caption{Robustness evaluation of Adversarial Training methods on the CIFAR10
dataset. RME represents an ensemble of ResNet18 (R), MobileNet (M)
and EfficientNet E), while MobiX3 represents an ensemble of three
MobileNets. NAT and ADV measure the natural accuracy and the robust
accuracy against PGD-Linf attack ($\uparrow$the higher the better).
Other metrics measure the success attack rate (SAR) of adversarial
examples generated by the same PGD-Linf attack on fooling each single
member and all members of the ensemble ($\downarrow$the lower the
better). \label{tab:ens-adv-training}}

\centering{}%
\begin{tabular}{lrrrrrrrrr}
 & \multicolumn{4}{c}{MobiX3} & \multicolumn{5}{c}{RME}\tabularnewline
\cline{2-5} \cline{3-5} \cline{4-5} \cline{5-5} \cline{7-10} \cline{8-10} \cline{9-10} \cline{10-10} 
 & NAT$\uparrow$ & \textcolor{blue}{ADV}$\uparrow$ & A-All$\downarrow$ & A-Avg$\downarrow$ &  & NAT$\uparrow$ & \textcolor{blue}{ADV}$\uparrow$ & A-All$\downarrow$ & A-Avg$\downarrow$\tabularnewline
\cline{1-5} \cline{2-5} \cline{3-5} \cline{4-5} \cline{5-5} \cline{7-10} \cline{8-10} \cline{9-10} \cline{10-10} 
PGD-AT & \textbf{80.43} & \textcolor{blue}{32.78} & 54.34 & 73.89 &  & \textbf{86.52} & \textcolor{blue}{37.36} & 49.01 & 69.75\tabularnewline
MinMax-AT & 79.01 & \textcolor{blue}{\uuline{37.28}} & 50.28 & \textbf{66.77} &  & \uuline{83.16} & \textcolor{blue}{\uuline{40.40}} & 46.91 & \uuline{65.73}\tabularnewline
MOO-AT & \uuline{79.38} & \textcolor{blue}{33.04} & \textbf{46.28} & 74.36 &  & 82.04 & \textcolor{blue}{37.48} & \uuline{45.24} & 70.11\tabularnewline
TA-MOO-AT & 79.22 & \textbf{\textcolor{blue}{38.22}} & \uuline{48.21} & \uuline{67.83} &  & 82.59 & \textbf{\textcolor{blue}{41.32}} & \textbf{43.68} & \textbf{65.09}\tabularnewline
\hline 
\end{tabular}
\end{table}

\subsection{Universal Perturbation (UNI)}
\label{subsec:Uni}
\paragraph{Experimental setting. }

We follow the experimental setup in \cite{wang2021adversarial}, where the full test set (10k images) is randomly divided into equal-size
groups (K images per group). The comparison has been conducted on
the CIFAR10 and CIFAR100 datasets, with an adversarially trained ResNet18
model and CW loss. We observed that the $\text{ASR-All}$ was mostly
zero, indicating that it is difficult to generate a general perturbation for all data points. Therefore, in Table \ref{tab:cf10-cf100-uni} we use $\text{ASR-Avg}$ to compare the performances of the methods. 
More experiments on VGG16 and EfficientNet models can be found in Appendix \ref{subsec:sup-uni-exp}.

\textbf{Results.}
Table \ref{tab:cf10-cf100-uni} shows the evaluation of generating universal perturbations on the CIFAR10 and CIFAR100 datasets, respectively. $K$ represents the number of images that are using the same perturbation. The larger the value of $K$, the harder it is to generate a universal perturbation that can be applied successfully to all images. 
It can be seen that with a small number of tasks (i.e., $K$=4), MOO and TA-MOO achieve lower performance than the MinMax method. However, with a large number of tasks (i.e, $K\geq8$), MOO and TA-MOO show their effectiveness and achieve the best performance. More specifically, on the CIFAR10 dataset, the improvements of MOO over the Uniform strategy are 5.6\%, 4\%, 3.2\%, and 2.5\% with $K=8$, $K=12$, $K=16$, and $K=20$, respectively. On the same setting, TA-MOO significantly improves MOO by around 4\% in all the $K$ settings and consistently achieves the best performance. 
Unlike the ENS setting, in the UNI setting, MOO consistently achieves better performance than the Uniform strategy 
. This improvement can be explained by the fact that in the UNI setting with the same architecture and data transformation, no task dominates the others. There will be a case (a group) when one sample is extremely close to/far from the decision boundary, and hence easier/harder to fool. However, in the entire test set with a large number of groups, the issue of dominating tasks is lessened.

\begin{table}
\begin{centering}
\caption{Evaluation of generating Universal Perturbation on the CIFAR10 and
CIFAR100 datasets.\label{tab:cf10-cf100-uni} }
\par\end{centering}
\centering{}\resizebox{\textwidth}{!}{ %
\begin{tabular}{lccccccccccc}
 & \multicolumn{5}{c}{CIFAR10} &  & \multicolumn{5}{c}{CIFAR100}\tabularnewline
 & K=4 & K=8 & K=12 & K=16 & K=20 &  & K=4 & K=8 & K=12 & K=16 & K=20\tabularnewline
\midrule 
Uniform & 37.52 & 30.34 & 27.41 & 25.52 & 24.31 &  & 65.40 & 58.99 & 55.33 & 53.02 & 51.49\tabularnewline
MinMax & \textbf{50.13} & 33.68 & 20.46 & 15.74 & 14.73 &  & \textbf{74.73} & 62.29 & 52.05 & 45.26 & 42.33\tabularnewline
MOO & 43.80 & \uuline{35.92} & \uuline{31.41} & \uuline{28.75} & \uuline{26.83} &  & 69.35 & \uuline{62.72} & \uuline{57.72} & \uuline{54.12} & \uuline{52.25}\tabularnewline
TA-MOO & \uuline{48.00} & \textbf{39.31} & \textbf{34.96} & \textbf{31.84} & \textbf{30.12} &  & \uuline{72.74} & \textbf{68.06} & \textbf{62.33} & \textbf{57.48} & \textbf{54.12}\tabularnewline
\bottomrule
\end{tabular}}\vspace{-2mm}
\end{table}

    \begin{table}
    \caption{Robust adversarial examples against transformations evaluation. I:
    Identity, H: Horizontal flip, V: Vertical flip, C: Center crop, G:
    Adjust gamma, B: Adjust brightness, R: Rotation. \label{tab:eot-main-res}}
    
    \centering{}%
    \begin{tabular}{llccccccccc}
     &  & \textcolor{blue}{A-All} & A-Avg & I & H & V & C & G & B & R\tabularnewline
    \midrule 
    \multirow{4}{*}{C10} & Uniform & \textcolor{blue}{25.98} & \textbf{55.33} & \textbf{44.85} & 41.58 & 82.90 & \textbf{72.56} & \textbf{45.92} & \textbf{49.59} & \textbf{49.93}\tabularnewline
     & MinMax & \textcolor{blue}{\uuline{30.54}} & 52.20 & 43.31 & \uuline{41.59} & 78.80 & 64.83 & 44.38 & 46.53 & 45.97\tabularnewline
     & MOO & \textcolor{blue}{21.25} & 49.81 & 36.23 & 33.93 & \textbf{87.47} & 71.05 & 37.68 & 40.21 & 42.12\tabularnewline
     & TA-MOO & \textbf{\textcolor{blue}{31.10}} & \uuline{55.26} & \uuline{44.15} & \textbf{41.86} & \uuline{85.19} & \uuline{71.86} & \uuline{45.53} & \uuline{48.70} & \uuline{49.54}\tabularnewline
    \midrule
    \multirow{4}{*}{C100} & Uniform & \textcolor{blue}{56.19} & \uuline{76.23} & \textbf{70.43} & 69.01 & 87.66 & \uuline{87.36} & 71.40 & \uuline{74.25} & \uuline{73.47}\tabularnewline
     & MinMax & \textcolor{blue}{\uuline{59.75}} & 75.72 & \uuline{70.13} & \uuline{69.26} & 87.45 & 86.03 & \uuline{71.54} & 73.30 & 72.32\tabularnewline
     & MOO & \textcolor{blue}{53.17} & 74.21 & 66.96 & 65.68 & \textbf{89.16} & 87.03 & 68.49 & 71.11 & 71.06\tabularnewline
     & TA-MOO & \textbf{\textcolor{blue}{60.88}} & \textbf{76.71} & \textbf{70.43} & \textbf{69.37} & \uuline{89.11} & \textbf{87.95} & \textbf{71.70} & \textbf{74.73} & \textbf{73.69}\tabularnewline
    \bottomrule
    \end{tabular}
    \end{table}

\subsection{Robust Adversarial Examples against Transformations (EoT)\label{subsec:EoT}}

\textbf{Results.}  Table \ref{tab:eot-main-res} shows the evaluation on the CIFAR10 and CIFAR100 datasets with 7 common data transformations. 
It can be observed that (i) MOO has a lower performance than the baselines, 
(ii) the Task Oriented regularization significantly boosts the performance, 
and (iii) our TA-MOO method achieves the best performance on both settings 
and outperforms the MinMax method 0.6\% and 1.1\% in the CIFAR10 and the CIFAR100 experiments, respectively. 
The low performance of MOO in observation (i) is again caused by the issue of one task dominating others. 
In the EoT setting, it is because of the V-vertical flip transformation as shown in Table \ref{tab:eot-main-res}. 
Observation (ii) provides another piece of evidence to support the effectiveness of the Task-Oriented regularization 
for MOO. This regularization boosts the ASRs in all the tasks (except V - the dominant one), increases the average ASR 
by 5.45\% and 2.5\% in the CIFAR10 and CIFAR100 experiments, respectively.

\subsection{Additional Experiments with Multi-Task Learning Methods} \label{sec:additional-moo} 

\REVISE{
In this section we would like to provide additional experiments with recent multi-task learning 
methods to explore how better constrained approaches can improve over the naive MOO. 
We applied three recent multi-task learning methods including PCGrad \cite{yu2020gradient}, CAGrad \cite{liu2021conflict}, and HVM \cite{albuquerque2019multi}
with implementation from their official repositories into our adversarial generation task.
We apply the best practice in \cite{albuquerque2019multi} which is adaptively updated the Nadir point based on the current tasks' losses.
For PCGrad we use the \textit{mean} as the reduction mode. 
For CAGrad we use parameter $\alpha=0.5$ and $\text{rescale}=1$ as in their default setting. 
We experiment on attacking ensemble of models setting with two settings, a diverse set \textit{D} with 4 different architectures including 
R-ResNet18, V-VGG16, G-GoogLeNet, E-EfficientNet and a non-diverse set \textit{ND} with 4 ResNet18 models.
}

\REVISE{
It can be seen from the Table \ref{tab:ENS-non-diverse-add-moo-with-hyper} that in the diverse ensemble setting, the three additional methods 
HVM, PCGrad and CAGrad significantly outperform the standard MOO method with the improvement gaps of SAR-All around 4.7\%, 3\% and 5\%, 
respectively. In the non-diverse ensemble setting, while HVM and PCGrad achieve lower performances than 
the standard MOO method, CAGrad can outperform the MOO method with a 2.7\% improvement. 
On comparison to the naive uniform method, the three methods also achieve better performance in both settings.
}

\REVISE{
The improvement on the diverse set of HVM, PCGrad and CAGrad over the standard MOO method is more 
noticeable than on the non-diverse set. It can be explained by the fact that on the diverse 
set of model architectures, there is a huge difference in gradients among architectures, therefore, 
requires a better multi-task learning method to handle the constraint between tasks. 
}

\REVISE{
On the other hand, on both ensemble settings, our TA-MOO still achieves the best performance, 
with a huge gap of (5.8\%) 7.8\% compared to the second best method on the (non) diverse setting.
It is because our method can leverage a supervised signal from knowing whether a task is achieved or not 
to focus on improving unsuccessful tasks. It is a huge advantage compared to unsupervised multi-task 
learning methods as MOO, HVM, PCGrad, and CAGrad.
}

\begin{center}
	\begin{table}
	\centering{}\caption{Attacking Ensemble model with a diverse set D=\{R-ResNet18, V-VGG16,
	G-GoogLeNet, E-EfficientNet\} and non-diverse set ND=\{4 ResNets\}.
	\label{tab:ENS-non-diverse-add-moo-with-hyper}}
	\begin{tabular}{llcccccc}
	 &  & \textcolor{blue}{A-All} & A-Avg & R/R1 & V/R2 & G/R3 & E/R4\tabularnewline
	\midrule 
	\multirow{6}{*}{D} & Uniform & \textcolor{blue}{28.21} & 48.34 & 48.89 & 49.08 & 48.38 & 47.03\tabularnewline
	 & HVM & \textcolor{blue}{29.88} & 46.98 & 48.97 & 48.10 & 46.88 & 43.96\tabularnewline
	 & PCGrad & \textcolor{blue}{28.25} & 48.28 & 48.81 & 49.03 & 48.13 & 47.14\tabularnewline
	 & CAGrad & \textcolor{blue}{30.23} & 48.34 & 47.03 & 48.22 & 45.92 & 52.20\tabularnewline
	 & MOO & \textcolor{blue}{25.16} & 44.76 & 39.06 & 46.83 & 37.05 & 56.11\tabularnewline
	 & TA-MOO & \textcolor{blue}{38.01} & 51.10 & 49.55 & 52.15 & 49.29 & 53.40\tabularnewline
	\midrule 
	\multirow{6}{*}{ND} & Uniform & \textcolor{blue}{28.17} & 48.75 & 51.94 & 45.55 & 54.15 & 43.34\tabularnewline
	 & HVM & \textcolor{blue}{28.46} & 49.87 & 51.64 & 50.03 & 50.72 & 47.10\tabularnewline
	 & PCGrad & \textcolor{blue}{28.30} & 48.75 & 52.02 & 45.42 & 54.35 & 43.21\tabularnewline
	 & CAGrad & \textcolor{blue}{35.22} & 51.07 & 54.22 & 47.84 & 55.24 & 46.97\tabularnewline
	 & MOO & \textcolor{blue}{32.50} & 52.21 & 53.25 & 49.05 & 56.80 & 49.76\tabularnewline
	 & TA-MOO & \textcolor{blue}{41.01} & 57.33 & 58.88 & 55.32 & 60.81 & 54.29\tabularnewline
	\bottomrule
	\end{tabular}
	\end{table}
\par\end{center}

\section{Additional Discussion} \label{sec:main-additional-discussions}
\REVISE{
In this section, we would like to summarize some important observations through all experiments while the complete discussion with detail can be found in Appendix \ref{sec:sup-discussions}.
}

\REVISE{
\textbf{Correlation between the objective loss and attack performance.} 
It is broadly accepted that to fool a model, a feasible approach is maximizing the objective loss (i.e., CE, KL, or CW loss), and the higher the loss, the higher the attack success rate. 
While it is true when observing the same architecture, we found that it is not necessarily true when comparing different architectures. 
As shown in Figure \ref{fig:ens-cf10-diverse-main},
with CW loss as the adversarial objective, it can be observed that there is a positive correlation between the loss value and the ASR, i.e., 
the higher the loss, the higher the ASR. However, there is no clear correlation observed when using CE and KL loss. 
Therefore, the higher weighted loss does not directly imply a higher success rate for attacking an ensemble of different architectures. 
The MinMax method \citep{wang2021adversarial} which solely weighs tasks' losses, therefore, is not always appropriate to achieve a good performance in all tasks. 
More discussion can be found in Appendix \ref{subsec:correlation-loss-attack-performance}.
}

\REVISE{
\textbf{When does MOO work?}
On the one hand, the dominating issue is observed in all three settings (ENS, UNI, EoT). The issue can be recognized by the gap of attack performance among tasks or by observing the dominating of one task's weight over others which is caused by a significant small gradient strength of one task on comparison with other tasks' strength as discussed in Section \ref{subsec:Ens}. 
The root of the dominating issue can be the natural of the setting (i.e., as in EoT setting, when the large gap can be observed in all methods) or the MOO solver.
}

\REVISE{
On the other hand, if overcoming this issue, MOO can outperform the Uniform strategy as shown in Section \ref{subsec:Ens}. As discussed in Appendix \ref{subsec:EoT}, a simple memory can helps to overcome the infinite gradient issue and significantly boosts the performance of MOO or TA-MOO. Therefore, we believe that developing a technique to lessen the dominating issue might be a potential extension. 
}

\REVISE{
\textbf{More efficient MOO solvers.} 
Inspired by \cite{sener2018multi}, in this paper we use multi-gradient descent algorithm \citep{deb2011multi} as a MOO solver which casts the multi-objective problem to a single-objective problem. However, while \cite{sener2018multi} used Frank-Wolfe algorithm to project the weight into the desired simplex, we use parameterization with softmax to do the job. While this technique is much faster than Frank-Wolfe algorithm, it has some weaknesses that might be target for future works. First, it cannot handle well the edge case which is the root of the dominating issue. Second, it does not work well in the case of a non-convex objective space as similar as other MOO scalarizing methods \citep{deb2011multi}.
}

\begin{figure}
    \begin{centering}
    \subfloat[CW]{
    \includegraphics[width=0.5\textwidth]{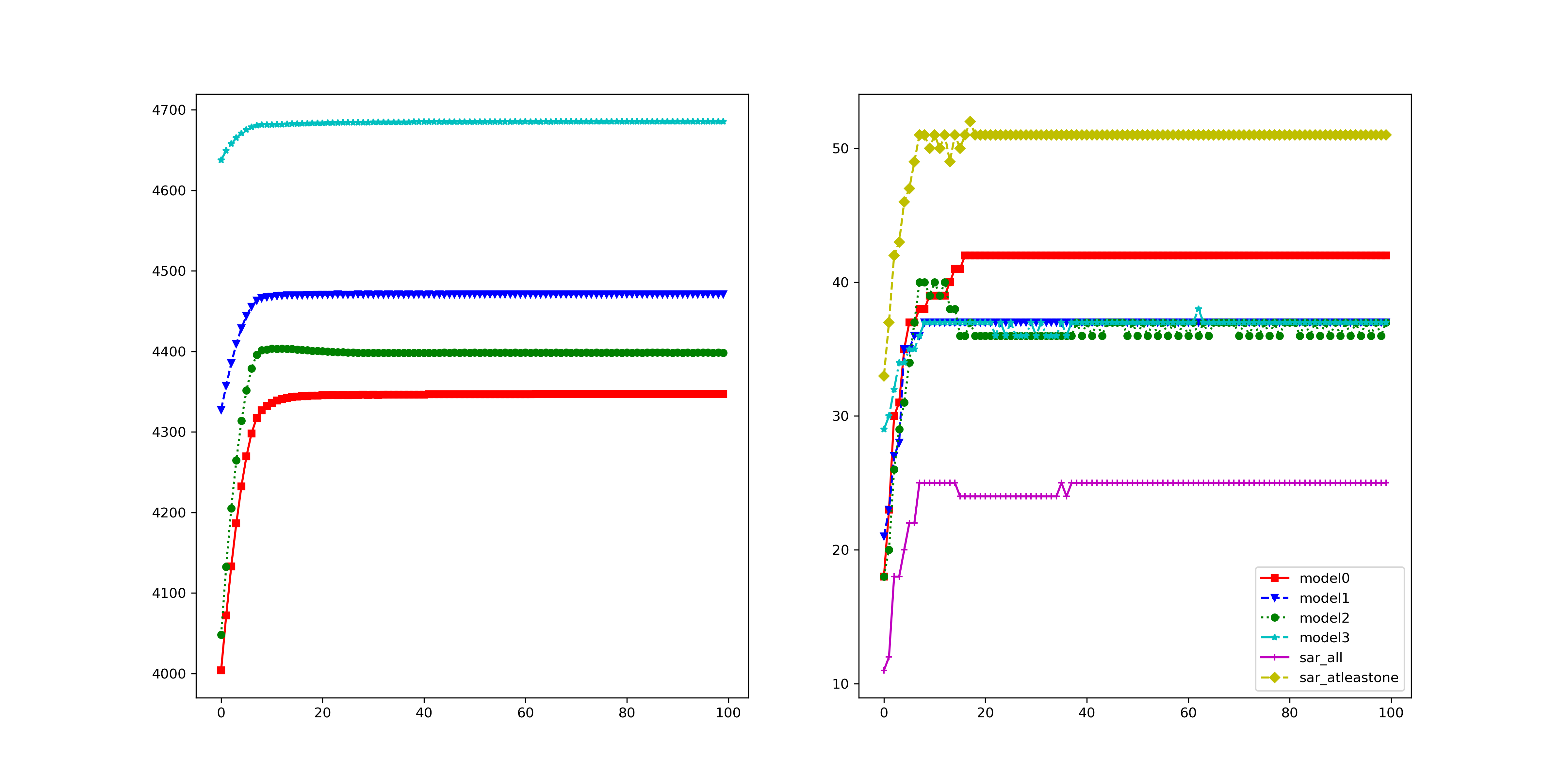}
    }
    \subfloat[CE]{
    \includegraphics[width=0.5\textwidth]{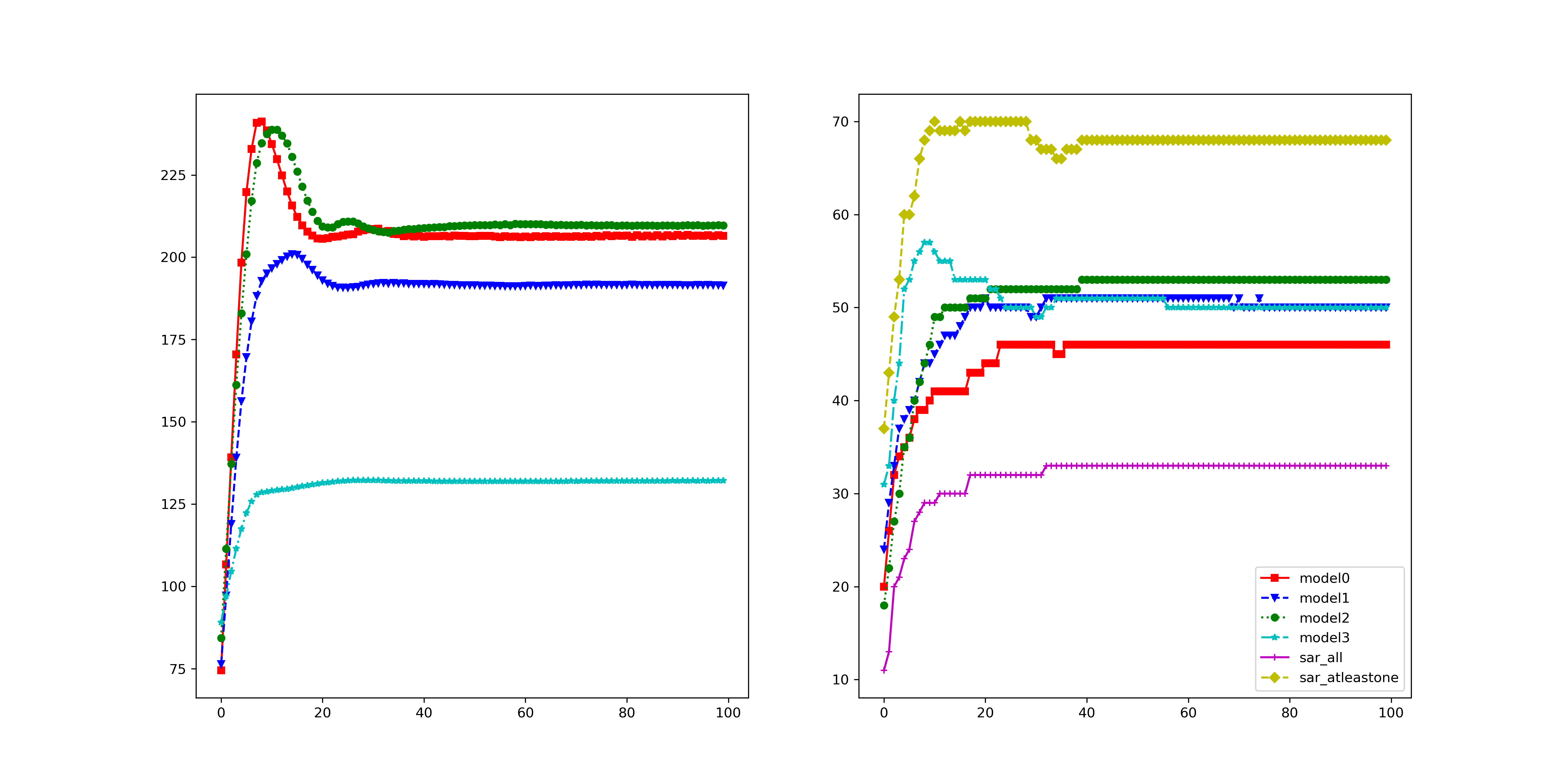}
    }
        
    \par\end{centering}
    \caption{Loss (left fig) and ASR (right fig) of each task over all attack iterations with the MinMax method. model0/1/2/3 represents R/V/G/E architecture, respectively. }
    \label{fig:ens-cf10-diverse-main}
\end{figure}

\section{Conclusion}
In this paper, we propose Task Oriented Multi-Objective Optimization (TA-MOO), with specific applications to adversarial generation tasks. We develop a geometry-based regularization term to favor the goal-unachieved tasks, while trying to maintain the the goal-achieved tasks. We conduct comprehensive experiments to showcase the merit of our proposed approach on generating adversarial examples and adversarial training. 
On the other hand, there are acknowledged limitations of our method such as weaknesses of the gradient-based solver and lacking theory on algorithm's convergence which might be target for future works. 

\section*{Acknowledgements}
This work was partially supported by the Australian Defence Science and Technology (DST) Group under the Next Generation Technology Fund (NGTF) scheme. 
The authors would like to thank the anonymous reviewers for their valuable comments and suggestions.

\bibliography{tmlr}
\bibliographystyle{tmlr}

\appendix
\clearpage
\section*{APPENDIX}

The Appendix provides technical and experimental details as well as auxiliary aspects 
to complement the main paper. Briefly, it contains the following: 

\begin{itemize}
    \item Appendix \ref{sec:related_work}: Discussion on related work. 
    \item Appendix \ref{sec:sup-method-detail}: Detailed proof and an illustration 
    of our methods.
    \item Appendix \ref{sec:sup-exp-setting}: Detailed description of experimental settings. 
    \item Appendix \ref{subsec:sup-ens-exp}: Additional experiments on transferability of adversarial examples in the ENS setting. 
    \item Appendix \ref{subsec:sup-adv-training}: Additional experiments on adversarial training with our methods.
    \item Appendix \ref{subsec:sup-uni-exp}: Additional experiments on the UNI setting. 
    \item Appendix \ref{subsec:sup-aug-exp}: Additional experiments on the EoT setting. 
    \item Appendix \ref{subsec:sup-speed-comparison}: Additional comparison on speed of generating adversarial examples. 
    \item Appendix \ref{sec:alb-sensitivity}: Additional experiments on sensitivity to hyper-parameters. 
    \item Appendix \ref{sec:alb-compare-attacks}: Additional comparison  with standard attacks on attacking performance.   
    \item Appendix \ref{subsec:attack-imagenet}: Additional experiments on attacking the ImageNet dataset. 
    \item Appendix \ref{subsec:sup-discussion-dominating-issue}: Additional discussions on the dominating issue and when MOO can work. 
    \item Appendix \ref{subsec:sup-discussion-task-oriented-reg}: A summary on the importance of Task-Oriented regularization.
    \item Appendix \ref{sec:sup-gradient-des-discuss}: Discussion on the limitation of MOO solver. 
    \item Appendix \ref{subsec:correlation-loss-attack-performance}: Discussion on correlation between the objective loss and attack performance. 
    \item Appendix \ref{subsec:conflicting}: Discussion on the conflicting between gradients in the adversarial generation task.
    \item Appendix \ref{subsec:convergence}: Discussion on the convergence of our methods.
    \item Appendix \ref{subsec:optimal-init-moo}: Additional experiments with MOO with different initializations.
    
\end{itemize}

\section{Related Work \label{sec:related_work}}

\paragraph{Multi-Objective Optimization for multi-task learning.}
\citep{desideri2012multiple} proposed a multi-gradient descent algorithm for multi-objective optimization (MOO) 
which opens the door for the applications of MOO in machine learning and deep learning. Inspired by \cite{desideri2012multiple}, 
MOO has been applied in multi-task learning (MTL) \citep{sener2018multi, mahapatra2020multi}, few-shot learning \citep{ye2021multi}, 
and knowledge distillation \citep{du2020agree}. Specifically, the work of \cite{sener2018multi} viewed multi-task learning as a 
multi-objective optimization problem, where a task network consists of a shared feature extractor and a task-specific predictor. 
The work of \cite{mahapatra2020multi} developed a gradient-based multi-objective
MTL algorithm to find a solution that satisfies the user preferences. The work of \cite{lin2019pareto} proposed a Pareto MTL to 
find a set of well-distributed Pareto solutions which
can represent different trade-offs among different tasks. Recently, the work of \cite{liu2021profiling} leveraged MOO with Stein 
Variational Gradient Descent \citep{liu2016stein} to diversify the solutions of MOO. Additionally, the work of \cite{ye2021multi} 
proposed a bi-level MOO which can be applied to few-shot learning. Finally, the work of \cite{du2020agree} applied MOO to enable 
knowledge distillation from multiple teachers.

\paragraph{Generating adversarial examples with single-objective and multi-objective optimizations.}
Generating qualified adversarial examples is crucial for adversarial training \citep{madry2017towards, trades, bui2021understanding, buiunified}. Many perturbation 
based attacks have been proposed, notably FGSM \citep{goodfellow2014explaining}, PGD \citep{madry2017towards}, TRADES \citep{trades}, 
CW \citep{carlini2017towards}, BIM \citep{kurakin2018adversarial}, and AutoAttack \citep{croce2020reliable}. Most adversarial attacks 
aim to maximize a single objective, e.g., maximizing the cross-entropy (CE) loss w.r.t. the ground-truth label \citep{madry2017towards}, 
maximizing the Kullback-Leibler (KL) divergence w.r.t. the predicted probabilities of a benign example \citep{trades}, or maximizing 
the CW loss \citep{carlini2017towards}. However, in some contexts, we need to generate adversarial examples maximizing multiple 
objectives or goals, e.g., attacking multiple models \citep{pang19a, bui2020improving} or finding universal perturbations 
\citep{moosavi2017universal}. 

The work of \cite{suzuki2019adversarial} was a pioneering attempt to consider the generation of adversarial examples as a multi-objective optimization problem. 
The authors proposed a non-adaptive method based on Evolutionary Multi-Objective Optimization (EMOO) \cite{deb2011multi} to generate sets of adversarial examples. 
However, the EMOO method is computationally expensive and requires a large number of evaluations, which limits its practicality. 
Additionally, the authors applied MOO without conducting an extensive study on the behavior of the algorithm, which could limit the effectiveness of the proposed method. 
Furthermore, the experimental results presented in the work are limited, which could weaken the evidence for the effectiveness of the proposed method.

To this end, the work of \cite{wang2021adversarial} examined the worst-case scenario by casting 
the problem of interest as a min-max problem for finding the weight of each task. However, this principle leads to a problem of 
lacking generality in other tasks. To mitigate the issue, \cite{wang2021adversarial} proposed a regularization to 
strike a balance between the average and the worst-case performance. 
The final optimization was formulated as follow: 
\begin{equation*}
    \underset{\delta : \| \delta \| \leq \epsilon }{\text{max}} \; \underset{w \in \simplex_{m}}{\text{min}} \sum_{i=1}^{K} w_i f_i(\delta) + \frac{\gamma}{2} \| w - 1/K \|_2^2,
\end{equation*}
Where $f_i(v)$ is the victim model's loss (i.e., cross entropy loss or KL divergence) and $\gamma > 0$ is the regularization parameter. 
The authors used the bisection method \citep{boyd2004convex} with project gradient descent for the inner minimization and project gradient ascent for the outer maximization. There are several major differences in comparison to MOO and TA-MOO methods: (i) In principle, MinMax considers the worst-case performance only while our methods improve performance of all tasks simultaneously. 
(ii) MinMax weighs the tasks' losses to find the minimal weighted sum loss in its inner minimization, however, as discussed in Section \ref{subsec:correlation-loss-attack-performance} the higher weighted loss does not directly imply the higher success rate in attacking multi-tasks simultaneously. In contrast, our methods use multi-gradient descent algorithm \citep{deb2011multi} in order to increase losses of all tasks simultaneously. (iii) The original principle of MinMax leads to the biasing problem to the worst-case task. The above regularization has been used to mitigate the issue, however, it considers all tasks equally. In contrast, our TA-MOO takes goal-achievement status of each task into account and focuses more on the goal-unachieved tasks.

Recently, \cite{guo2020multi} proposed a multi-task adversarial attack and demonstrated on the universal perturbation problem. 
However, while \cite{wang2021adversarial} and ours can be classified as an iterative optimization-based attack, 
\cite{guo2020multi} requires a generative model in order to generate adversarial examples. 
While this line of attack is faster than optimization-based attacks at the inference phase, it requires to 
train a generator on several tasks beforehand. 
Due to the difference in setting, we do not compare with that work in this paper.

More recently, \cite{qiu2022framework} proposed a framework to attack a generative Deepfake model using the multi-gradient descent algorithm in their backpropagation step. While their method also use 
the multi-objective optimization for generating adversarial examples, there are several major differences to ours. 
Firstly, their method aims for a generative Deepfake model while our method aims for the standard classification problem which is the most common and important setting in AML. 
Secondly, we conduct comprehensive experiments to show that a direct and naive application of MOO to adversarial generation tasks does not work satisfactorily because of the gradient dominating problem. 
Most importantly, we propose the TA-MOO method which employs a geometry-based regularization term to favor the unsuccessful tasks,
while trying to maintain the performance of the already successful tasks. 
We have conducted extensive experiments to show that our TA-MOO consistently achieves the best 
attacking performance across different settings. 
We also conducted additional experiments with SOTA multi-task learning methods which are PCGrad \citep{yu2020gradient} and CAGrad \citep{liu2021conflict} in Section \ref{sec:additional-moo}. 
Compared to these methods, our TA-MOO still achieves the best attack performance thanks to the Task Oriented regularization.

\section{Further Details of the Proposed Method \label{sec:sup-method-detail}}

\subsection{Proofs}\label{sec:sup-proof}
\lemmaone*
\begin{proof}
The proof is based on \cite{wang2013projection} with modifications.
We need to solve the following OP:
\begin{align*}
 & \min_{\pi}\frac{1}{2}\norm{w-\pi}_{2}^{2}\\
\text{s.t.}: & \pi\geq\bzero\\
 & \norm{\pi}_{1}=1.
\end{align*}

We note that $\pi_{1}=...=\pi_{s}=0$. The OP of interest reduces
to 
\begin{align*}
 & \min_{\pi_{s+1:m}}\frac{1}{2}\sum_{i=s+1}^{m}\left(\pi_{i}-w_{i}\right)^{2}\\
\text{s.t.}: & \pi_{s+1:m}\geq\bzero\\
 & \sum_{i=s+1}^{m}\pi_{i}=1.
\end{align*}

Using the Karush-Kuhn-Tucker (KKT) theorem, we construct the following
Lagrange function:
\[
\mathcal{L}\left(\pi,\gamma,\beta\right)=\frac{1}{2}\sum_{i=s+1}^{m}\left(\pi_{i}-w_{i}\right)^{2}-\gamma\left(\sum_{i=s+1}^{m}\pi_{i}-1\right)-\sum_{i=s+1}^{m}\beta_{i}\pi_{i}.
\]

Setting the derivative w.r.t. $\pi_{i}$ to zeros and using the KKT
conditions, we obtain:
\begin{align*}
\pi_{i}-w_{i}-\gamma-\beta_{i} & =0,\forall i=s+1,...,m\\
\sum_{i=s+1}^{m}\pi_{i}= & 1\\
\beta_{i}\geq0,\pi_{i}\geq0, & \beta_{i}\pi_{i}=0,\forall i=s+1,...,m.
\end{align*}

If $\pi_{i}>0$, $\beta_{i}=0$, hence $\pi_{i}=w_{i}+\gamma>0$.
Otherwise, if $\pi_{i}=0$, $w_{i}+\gamma=-\beta_{i}\leq0$. Therefore,
$w_{s+1:m}$ has the same order as $\pi_{s+1:m}$ and we can arrange
them as:
\[
\pi_{s+1}\geq\pi_{s+2}\geq...\geq\pi_{\rho}>\pi_{\rho-1}=...=\pi_{m}=0.
\]
\[
u_{s+1}=w_{s+1}\geq u_{s+2}=w_{s+2}\geq....\geq u_{p}=w_{p}\geq u_{\rho-1}=w_{\rho-1}\geq...\geq u_{m}=w_{m}\geq0.
\]

It appears that $1=\sum_{i=s+1}^{m}\pi_{i}=\sum_{i=s+1}^{\rho}\pi_{i}=\sum_{i=s+1}^{\rho}\left(w_{i}+\gamma\right)=\sum_{i=s+1}^{\rho}w_{i}+\left(\rho-s\right)\gamma$.
Hence, we gain $\gamma=\frac{1}{\rho-s}\left[1-\sum_{i=s+1}^{\rho}w_{i}\right]=\frac{1}{\rho-s}\left[1-\sum_{i=s+1}^{\rho}u_{i}\right]$.
We now prove that $\rho=\max\left\{ s+1\leq i\leq m:u_{i}+\frac{1}{i-s}\left(1-\sum_{j=s+1}^{i}u_{j}\right)>0\right\} $.
\end{proof}
\begin{itemize}
\item For $i=\rho$, we have 
\[
u_{\rho}+\frac{1}{\rho-s}\left(1-\sum_{j=s+1}^{\rho}u_{j}\right)=u_{\rho}+\gamma=w_{\rho}+\gamma>0.
\]
\item For $i<\rho$, we have\textbf{
\begin{align*}
u_{i}+\frac{1}{i-s}\left(1-\sum_{j=s+1}^{i}u_{j}\right) & =\frac{1}{i-s}\left((i-s)u_{i}+1-\sum_{j=s+1}^{i}u_{j}\right)\\
= & \frac{1}{i-s}\left[(i-s)w_{i}+\sum_{j=s+1}^{\rho-1}\pi_{j}-\sum_{j=s+1}^{i}w_{j}\right]\\
= & \frac{1}{i-s}\left[(i-s)w_{i}+\sum_{j=i+1}^{\rho-1}\pi_{j}+\sum_{j=s+1}^{i}\left(\pi_{j}-w_{j}\right)\right]\\
= & \frac{1}{i-s}\left[(i-s)\left(w_{i}+\gamma\right)+\sum_{j=i+1}^{\rho-1}\pi_{j}\right]\\
= & \frac{1}{i-s}\left[(i-s)\pi_{i}+\sum_{j=i+1}^{\rho-1}\pi_{j}\right]>0.
\end{align*}
}
\item For $i>\rho$, we have 
\begin{align*}
u_{i}+\frac{1}{i-s}\left(1-\sum_{j=s+1}^{i}u_{j}\right) & =\frac{1}{i-s}\left((i-s)u_{i}+1-\sum_{j=s+1}^{i}u_{j}\right)\\
 & =\frac{1}{i-s}\left((i-s)w_{i}+\sum_{j=s+1}^{\rho-1}\pi_{j}-\sum_{j=s+1}^{i}w_{j}\right)\\
 & =\frac{1}{i-s}\left((i-s)w_{i}+\sum_{j=s+1}^{\rho-1}(\pi_{j}-w_{j})-\sum_{j=\rho}^{i}w_{j}\right)\\
 & =\frac{1}{i-s}\left((i-s)w_{i}+(\rho-s-1)\gamma-\sum_{j=\rho}^{i}w_{j}\right)\\
 & =\frac{1}{i-s}\left((\rho-s-1)(w_{i}+\gamma)+\sum_{j=\rho}^{i}(w_{i}-w_{j})\right)\leq0.
\end{align*}
\end{itemize}
Therefore, $\rho=\max\left\{ s+1\leq i\leq m:u_{i}+\frac{1}{i-s}\left(1-\sum_{j=s+1}^{i}u_{j}\right)>0\right\} $.
Finally, we also have $\pi_{i}=\max\{w_{i}+\gamma,0\},i=s+1,...,m$
and $\pi_{i}=0,i=1,...,s$.

\lemmatwo*
\begin{proof}
Recall $\rho=\max\left\{ s+1\leq i\leq m:u_{i}+\frac{1}{i-s}\left(1-\sum_{j=s+1}^{i}u_{j}\right)>0\right\} $.
Therefore, $\rho=m$ because we have
\[
u_{m}+\frac{1}{m-s}\left(1-\sum_{j=s+1}^{m}u_{j}\right)=w_{m}+\frac{1}{m-s}\left(1-\sum_{j=s+1}^{m}w_{j}\right)=w_{m}+\frac{\sum_{j=1}^{s}w_{j}}{m-s}>0.
\]

It follows that 
\[
\gamma=\frac{1}{m-s}\left(1-\sum_{i=s+1}^{m}u_{i}\right)=\frac{1}{m-s}\left(1-\sum_{i=s+1}^{m}w_{i}\right)\geq0.
\]
\[
\text{proj}_{\mathcal{S}}\left(w\right)_{i}=\begin{cases}
0 & 1\leq i\leq s\\
\max\left\{ w_{i}+\gamma,0\right\} =w_{i}+\gamma & \text{otherwise}
\end{cases}
\]
\begin{align*}
\Omega\left(w\right) & =\sum_{i=1}^{s}w_{i}^{2}+\sum_{i=s+1}^{m}\left(w_{i}-\max\left\{ w_{i}+\gamma,0\right\} \right)^{2}\\
 & =\sum_{i=1}^{s}w_{i}^{2}+\sum_{i=s+1}^{m}\gamma^{2}=\sum_{i=1}^{s}w_{i}^{2}+(m-s)\gamma^{2}\\
 & =\sum_{i=1}^{s}w_{i}^{2}+\frac{1}{m-s}\left(1-\sum_{i=s+1}^{m}w_{i}\right)^{2}.
\end{align*}
\end{proof}

\subsection{Illustrations of How MOO and TA-MOO Work}

\begin{figure}
    \centering
    \includegraphics[width=\textwidth]{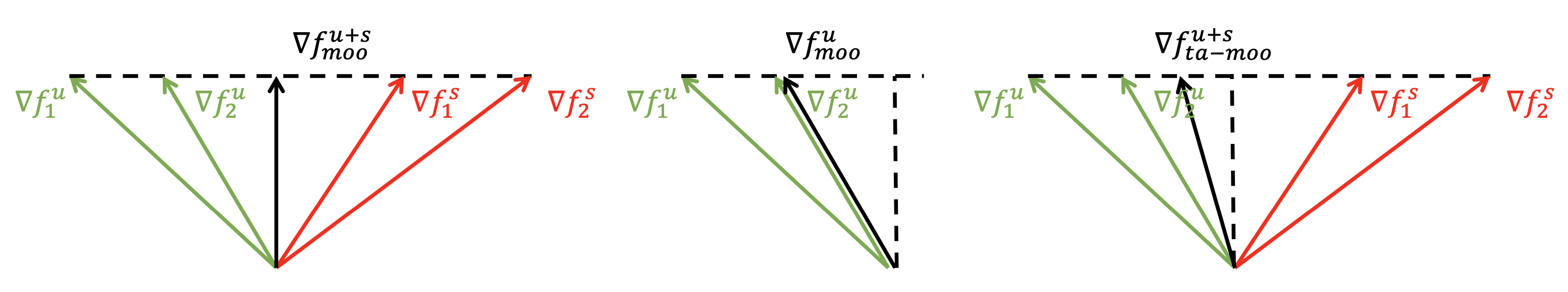}
    \caption{Visualization of standard MOO and TA-MOO solutions in a scenario of 2 goal-achieved tasks (\textcolor{blue}{$\nabla f^s_{1,2}$}) and 2 goal-unachieved tasks (\textcolor{teal}{$\nabla f^u_{1,2}$}). 
    (left) MOO; 
    (middle) MOO on the set of goal-unachieved tasks only; 
    (right) TA-MOO with a solution focuses more on the goal-unachieved tasks.}
    \label{fig:ta-moo-illus}
\end{figure}

Figure \ref{fig:ta-moo-illus} illustrates solutions of MOO and TA-MOO in a scenario of 2 goal-achieved tasks 
(with corresponding gradients \textcolor{blue}{$\nabla f^s_{1,2}$}) and 2 goal-unachieved tasks (with corresponding gradients \textcolor{teal}{$\nabla f^u_{1,2}$}). 
As illustrated in the left figure, with standard MOO method, where all the tasks' gradients have been considered regardless their status and  
the solution associated with the minimal norm is the perpendicular vector as suggested by geometry \citep{sener2018multi}. 
If considering the goal-unachieved tasks only as in the middle case, the MOO solution is the edge case. 
However, this extreme strategy ignores all the goal-achieved tasks which might lead to the instability. 
The Task Oriented regularization strikes a balance between the two aforementioned strategies as illustrated in the right figure. 
The method focuses more on improving the goal-unachieved tasks while spend less effort to maintain the goal-achieved tasks. 
With $\lambda=0$ the TA-MOO optimal solution is equivalent to the standard MOO optimal solution while 
it becomes the MOO solution in the case of the goal-unachieved tasks only when $\lambda \to \infty$.

\section{Experimental settings \label{sec:sup-exp-setting}}

\paragraph{General settings. }

Through our experiments, we use six common architectures including
ResNet18 \citep{he2016deep}, VGG16 \citep{simonyan2014very}, GoogLeNet
\citep{szegedy2015going}, EfficientNet (B0) \citep{tan2019efficientnet}, 
MobileNet \cite{howard2017mobilenets}, and WideResNet (with depth 34 and widen factor 10) \cite{zagoruyko2016wide} 
with the implementation of \url{https://github.com/kuangliu/pytorch-cifar}.
We evaluate on the full testing set (10k) of two benchmark datasets
which are CIFAR10 and CIFAR100 \citep{cifar10}. More specifically,
the two datasets have 50k training images and 10k testing images, respectively,
with the same image resolution of $32\times32\times3$. However, while
the CIFAR10 dataset has 10 classes, the CIFAR100 dataset has 100 classes
and fewer images per class. Therefore, in general, an adversary is
easier to attack a CIFAR100 model than a CIFAR10 one as shown in
Table \ref{tab:models-acc}. We observed that the attack performance
is saturated with standard training models. Therefore, to make the
job of adversaries more challenging, we use Adversarial Training with
PGD-AT \citep{madry2017towards} to robustify the models and use these
robust models as victim models in our experiments. Specifically, we
use the SGD optimizer (momentum 0.9 and weight decay $5\times10^{-4}$)
and Cosine Annealing Scheduler to adjust the learning rate with an
initial value of 0.1 and train a model in 200 epochs as suggested
in the implementation above. We use PGD-AT $L_{\infty}$ \citep{madry2017towards}
with the same setting for both CIFAR10 and CIFAR100 datasets, i.e.,
perturbation limitation $\epsilon=8/255$, $k=20$ steps, and step
size $\eta=2/255$. 

\begin{center}
\begin{table}
\caption{Robustness performance of models in the experiments\label{tab:models-acc}}

\centering{}%
\begin{tabular}{lccccc}
\hline 
 & \multicolumn{2}{c}{CIFAR10} &  & \multicolumn{2}{c}{CIFAR100}\tabularnewline
\cline{2-3} \cline{3-3} \cline{5-6} \cline{6-6} 
 & Nat-Acc & Adv-Acc &  & Nat-Acc & Adv-Acc\tabularnewline
\hline 
ResNet18 & 86.47 & 42.14 &  & 59.64 & 18.62\tabularnewline
VGG16 & 84.24 & 40.88 &  & 55.27 & 16.41\tabularnewline
GoogLeNet & 88.26 & 41.26 &  & 63.10 & 19.16\tabularnewline
EfficientNet & 74.52 & 41.36 &  & 57.67 & 19.90\tabularnewline
MobileNet & 76.52 & 31.12 &  & - & -\tabularnewline
WideResNet & 88.13 & 48.62 &  & - & -\tabularnewline
\hline 
\end{tabular}
\end{table}
\par\end{center}

\paragraph{Method settings. }
In this work, we evaluate all the methods in the untargeted attack setting with $L_\infty$ norm. 
The attack parameters are the same among methods, i.e., number of attack steps 100, attack budget $\epsilon=8/255$ and step size $\eta_\delta=2/255$. 
In our method, we use $K$=10 to update the weight in each step with learning rate $\eta_w=0.005$. Tradeoff parameter $\lambda=100$ in all experiments. 
In MinMax~\citep{wang2021adversarial}, we use the same $\gamma=3$ for all settings and use the authors' implementation \footnote{https://github.com/wangjksjtu/minmax-adv}. 

\paragraph{Attacking ensemble model settings. }

In our experiment, we use an ensemble of four adversarially trained
models: ResNet18, VGG16, GoogLeNet, and EfficientNet. The architecture
is the same for both the CIFAR10 and CIFAR100 datasets except for
the last layer which corresponds with the number of classes in each
dataset. The final output of the ensemble is an average of the probability
outputs (i.e., output of the softmax layer). We use three different
losses as an object for generating adversarial examples including
Cross Entropy (CE) \citep{madry2017towards}, Kullback-Leibler divergence
(KL) \citep{trades}, and CW loss \citep{carlini2017towards}. 

\paragraph{Universal perturbation settings. }

We follow the experimental setup in \cite{wang2021adversarial}, such
that the full test set (10k images) is randomly divided into equal-size
groups (K images per group). The comparison has been conducted on
the CIFAR10 and CIFAR100 datasets, and CW loss. We use adversarial trained ResNet18, VGG16 and EfficientNet as base models. We observed that the $\text{ASR-All}$ was mostly
zero, indicating that it is difficult to generate a general perturbation for all data points. 
Therefore, we use $\text{ASR-Avg}$ to compare the performances of the methods. 

\paragraph{Robust adversarial examples against transformations settings. }

In our experiment, we use 7 common data transformations including
I-Identity, H-Horizontal flip, V-Vertical flip, C-Center crop, B-Adjust
brightness, R-Rotation, and G-Adjust gamma. The parameter setting
for each transformation has been shown in Table \ref{tab:eot-setting}.
In the deterministic setting, a transformation has been fixed with
one specific parameter, e.g., center cropping with a scale of 0.6
or adjusting brightness with a factor of 1.3. While in the stochastic
setting, a transformation has been uniformly sampled from its family,
e.g., center cropping with a random scale in range (0.6, 1.0) or adjusting
brightness with a random factor in range (1.0, 1.3). The experiment
has been conducted on adversarially trained ResNet18 model with the  CW loss. 

\begin{table}
\caption{Data transformation setting. $U$ represents uniform sampling function
and $p$ represents probability to excuse a transformation (e.g.,
flipping). \label{tab:eot-setting}}

\centering{}%
\begin{tabular}{lcc}
 & Deterministic & Stochastic\tabularnewline
\hline 
Identity & Identity & Identity\tabularnewline
Horizontal flip & $p=1$ & $p=0.5$\tabularnewline
Vertical flip & $p=1$ & $p=0.5$\tabularnewline
Center crop & $\text{scale}=0.6$ & $\text{scale}=U(0.6,1.0)$\tabularnewline
Adjust brightness & $\text{factor}=1.3$ & $\text{factor}=U(1.0,1.3)$\tabularnewline
Rotation & $\text{angle}=\text{10°}$ & $\text{angle}=U(-\text{10°},\text{10°})$\tabularnewline
Adjust gamma & $\text{gamma}=1.3$ & $gamma=U(0.7,1.3)$\tabularnewline
\hline 
\end{tabular}
\end{table}

\section{Additional Experiments \label{sec:sup-additional-exp}}

\subsection{Transferability of adversarial examples in the ENS setting \label{subsec:sup-ens-exp}}
We conduct an additional experiment to evaluate the transferability of our adversarial examples. 
We use an ensemble (RME) of three models: ResNet18, MobileNet, and EfficientNet as a source model
and apply different adversaries to generate adversarial examples to this ensemble. 
We then use these adversarial examples to attack other ensemble architectures (target models), for example, 
RMEVW is an ensemble of 5 models including ResNet18, MobileNet, EfficientNet, VGG16 and WideResNet. 
Table \ref{tab:ens-transferability-sup} reports the SAR-All metric of transferred adversarial examples, 
where a higher number indicates a higher success rate of attacking a target model, therefore, 
also implies a higher transferability of adversarial examples. 
The first column (heading RME) shows SAR-All when adversarial examples attack the source model (i.e., the whitebox attack setting). 

\paragraph{The Uniform strategy achieves the lowest transferability. }
It can be observed from Table \ref{tab:ens-transferability-sup} that the Uniform strategy achieves the lowest SAR in the whitebox attack setting. 
This strategy also has the lowest transferability in attacking other ensembles (except an ensemble RVW). 

\paragraph{MinMax's transferability drops on dissimilar target models. }
While MinMax achieves the second-best performance in the whitebox attack setting, its adversarial examples have a 
low transferability when target models are different from the source model. For example, in the target model RVW  
where there is only one member of the target model from the source model (RME) (i.e., R or ResNet18), 
MinMax achieves a 23.75\% success rate which is lower than the Uniform strategy by 1.28\%. 
Similar observation can be observed on target models EVW and MVW, where MinMax outperforms the 
Uniform strategy by just 0.2\% and 0.6\%, respectively.  

\paragraph{TA-MOO achieves the highest transferability on a diverse set of ensembles}. 
Our TA-MOO adversary achieves the highest attacking performance on the whitebox attack setting, 
with a huge gap of 9.24\% success rate over the Uniform strategy. Our method also achieves 
the highest transferability regardless diversity of a target ensemble. 
More specifically, on target models such as REV, MEV, and RMEV, where members in the source ensemble (RME) 
are also in the target ensemble, our TA-MOO significantly outperforms the Uniform strategy, 
with the highest improvement is 5.19\% observed on target model RMEV. 
On the target models EVW and MVW which are less similar to the source model, 
our method still outperforms the Uniform strategy by 1.46\% and 1.65\%. 
The superior performance of our adversary on the transferability shows another benefit 
of using multi-objective optimization in generating adversarial examples. 
By reaching the intersection of all members' adversarial regions, our adversary is capable  
to generate a common vulnerable pattern on an input image shared across architectures, 
therefore, increasing the transferability of adversarial examples. 

\begin{center}
    \begin{table}
    \caption{Evaluation on the Transferability of adversarial examples. Each cell
    (row-ith, column-jth) reports SAR (higher is better) of adversarial
    examples from the same source architecture (RME) with an adversary
    at row-ith to attack an ensemble at column-jth. Each architecture
    has been denoted by symbols such as R: ResNet18, M: MobileNet, E:
    EfficientNet, V: VGG16, W: WideResNet. For examples, RME represents
    for an ensemble of ResNet18, MobileNet and EfficientNet. The highest/second
    highest performance is highlighted in \textbf{Bold}/\uuline{Underline}.
    The table is copied from Table \ref{tab:ens-transferability} in the main paper for reading 
    comprehension purpose.
    \label{tab:ens-transferability-sup}}
    
    \centering{}%
    \begin{tabular}{lcccccccc}
     & RME & RVW & EVW & MVW & REV & MEV & RMEV & RMEVW\tabularnewline
    \midrule 
    Uniform & 31.73 & \uuline{25.03} & 22.13 & 22.73 & 29.50 & 28.44 & 26.95 & 20.50\tabularnewline
    MinMax & \uuline{40.01} & 23.75 & 22.39 & 23.34 & \uuline{32.57} & \uuline{32.75} & \uuline{31.85} & \uuline{21.99}\tabularnewline
    MOO & 35.20 & 24.25 & \uuline{22.94} & \uuline{23.76} & 30.65 & 32.28 & 29.49 & 21.77\tabularnewline
    TA-MOO & \textbf{40.97} & \textbf{25.13} & \textbf{23.59} & \textbf{24.38} & \textbf{33.00} & \textbf{33.05} & \textbf{32.14} & \textbf{23.04}\tabularnewline
    \bottomrule
    \end{tabular}
    \end{table}
\par\end{center}

\subsection{Adversarial Training with TA-MOO} \label{subsec:sup-adv-training}

\paragraph{Setting.}
We conduct adversarial training with adversarial examples generated
by MOO and TA-MOO attacks to verify the quality of these adversarial
examples. We choose an ensemble of 3 MobileNet
architectures (non-diverse set) and ensemble of 3 different architectures
including ResNet18, MobileNet and EfficientNet (diverse set). To evaluate the adversarial robustness,
we compare natural accuracy (NAT) and robust accuracy (ADV) against
PGD-Linf attack of these adversarial training methods (the higher
the better). We also measure the success attack rate (SAR) of adversarial
examples generated by the same PGD-Linf attack on fooling each single
member and all members of the ensemble (the lower the better). 
We use ${k=10,\epsilon=8/255,\eta=2/255}$ for adversarial training and PGD-Linf 
with $k=20, \epsilon=8/255,\eta=2/255$ for robustness evaluation. 
We use SGD optimizer with momentum 0.9 and weight decay 5e-4. 
Initial learning rate is 0.1 with Cosine Annealing scheduler and train on 100 epochs.

\paragraph{Result 1. Reducing transferability. } It can be seen that the SAR-All of MOO-AT and 
TA-MOO-AT are much lower than that on other methods. More specifically, the gap of SAR-All between 
PGD-AT and TA-MOO-AT is (5.33\%) 6.13\% on the (non) diverse setting. 
The lower SAR-All indicating that adversarial examples are harder to transfer among ensemble members on the TA-MOO-AT model than on the PGD-AT model.

\paragraph{Result 2. Producing more robust single members. } The comparison of average SAR shows that adversarial training with TA-MOO 
produces more robust single models than PGD-AT does. More specifically, the average robust accuracy (measured by 100\% - \textit{A-Avg}) 
of TA-MOO-AT is 32.17\%, an improvement of 6.06\% over PGD-AT in the non-diverse setting, while 
there is an improvement of 4.66\% in the diverse setting.

\paragraph{Result 3. Adversarial training with TA-MOO achieves the best robustness. }
More specifically, on the non-divese setting, TA-MOO-AT achives 38.22\% robust accuracy, an improvement 
of 1\% over MinMax-AT and 5.44\% over standard PGD-AT. On the diverse setting, the improvement
over MinMax-AT and PGD-AT are 0.9\% and 4\%, respectively. The root of the improvement is the ability to generate stronger adversarial examples in the
the sense that they can challenge not only the entire ensemble model but also
all single members. These adversarial examples lie in the joint insecure
region of members (i.e., the low confidence region of multiple classes),
therefore, making the decision boundaries more separate. 
As a result, adversarial training with TA-MOO produces more robust 
single models (i.e., lower SAR-Avg) and significantly reduces 
the transferability of adversarial examples among members (i.e., lower SAR-All). 
These two conditions explain the best ensemble adversarial robustness achieved by TA-MOO. 

\begin{center}
\begin{table}
\caption{Robustness evaluation of Adversarial Training methods on the CIFAR10
dataset. RME represents an ensemble of ResNet18 (R), MobileNet (M)
and EfficientNet E), while MobiX3 represents an ensemble of three
MobileNets. NAT and ADV measure the natural accuracy and the robust
accuracy against PGD-Linf attack ($\uparrow$the higher the better).
Other metrics measure the success attack rate (SAR) of adversarial
examples generated by the same PGD-Linf attack on fooling each single
member and all members of the ensemble ($\downarrow$the lower the
better). The highest/second highest \textbf{robustness} is highlighted
in \textbf{Bold}/\uuline{Underline.} The most important metric
is emphasized in blue. \label{tab:ens-adv-training-full}}

\centering{}%
\begin{tabular}{llrrrrrrr}
 & Arch & NAT$\uparrow$ & \textcolor{blue}{ADV}$\uparrow$ & A-All$\downarrow$ & A-Avg$\downarrow$ & R/M1$\downarrow$ & M/M2$\downarrow$ & E/M3$\downarrow$\tabularnewline
\hline 
PGD-AT & MobiX3 & \textbf{80.43} & \textcolor{blue}{32.78} & 54.34 & 73.89 & 76.17 & 74.35 & 71.14\tabularnewline
MinMax-AT & MobiX3 & 79.01 & \textcolor{blue}{\uuline{37.28}} & 50.28 & \textbf{66.77} & \textbf{65.27} & \uuline{70.27} & \textbf{64.78}\tabularnewline
MOO-AT & MobiX3 & \uuline{79.38} & \textcolor{blue}{33.04} & \textbf{46.28} & 74.36 & 71.25 & 74.53 & 77.29\tabularnewline
TA-MOO-AT & MobiX3 & 79.22 & \textbf{\textcolor{blue}{38.22}} & \uuline{48.21} & \uuline{67.83} & \uuline{68.04} & \textbf{67.37} & \uuline{68.07}\tabularnewline
\hline 
PGD-AT & RME & \textbf{86.52} & \textcolor{blue}{37.36} & 49.01 & 69.75 & 65.81 & 75.24 & 68.21\tabularnewline
MinMax-AT & RME & \uuline{83.16} & \textcolor{blue}{\uuline{40.40}} & 46.91 & \uuline{65.73} & \uuline{65.22} & \textbf{68.28} & \uuline{63.70}\tabularnewline
MOO -AT & RME & 82.04 & \textcolor{blue}{37.48} & \uuline{45.24} & 70.11 & 69.00 & 75.43 & 65.90\tabularnewline
TA-MOO-AT & RME & 82.59 & \textbf{\textcolor{blue}{41.32}} & \textbf{43.68} & \textbf{65.09} & \textbf{63.77} & \uuline{68.98} & \textbf{62.51}\tabularnewline
\hline 
\end{tabular}
\end{table}
\par\end{center}

\begin{figure}
    \begin{centering}
    \subfloat[MobiX3]{
    \begin{centering}
    \includegraphics[width=0.5\textwidth]{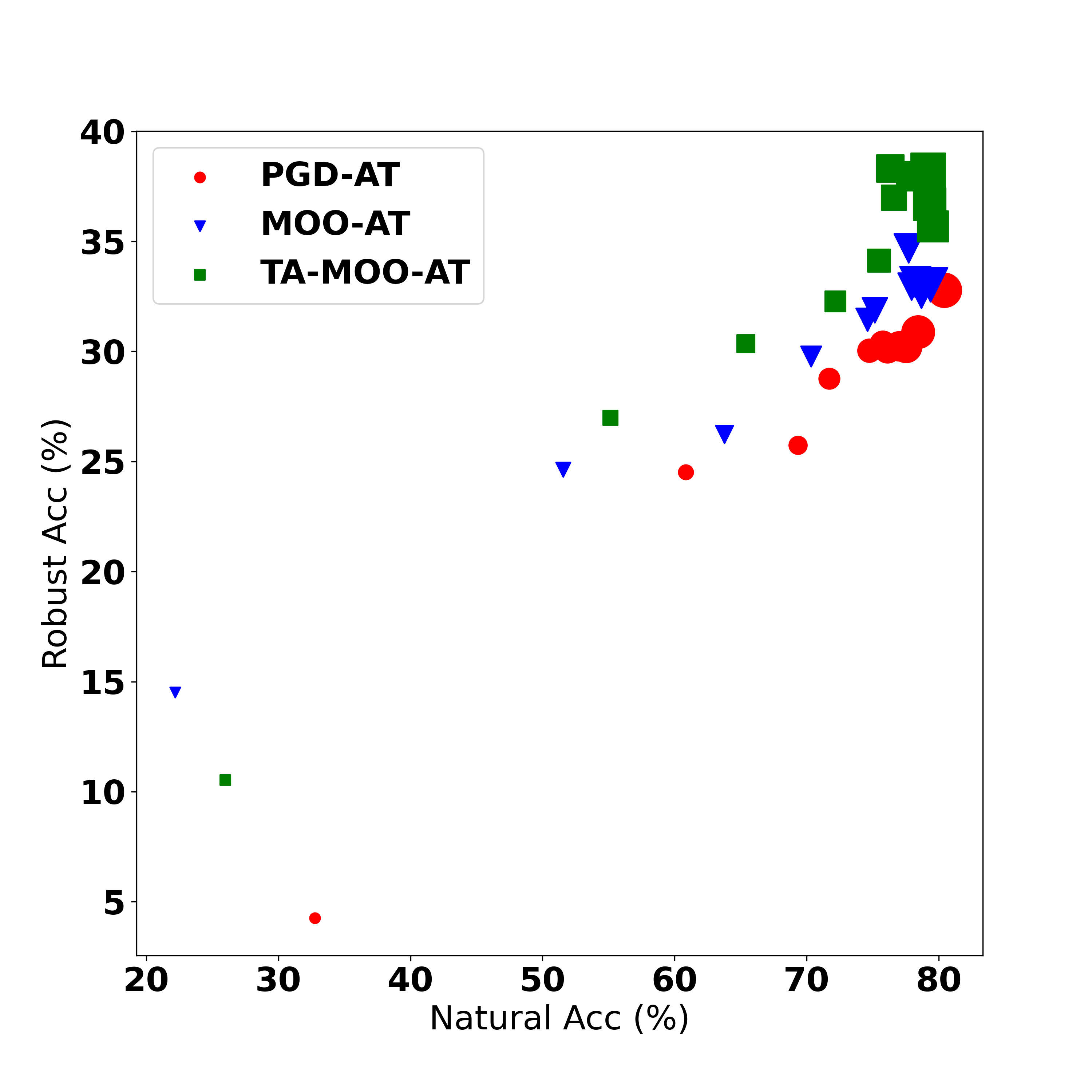}
    \par\end{centering}
    \label{fig:ens-adv-training-progress-mobix3}
    }
    \subfloat[RME]{\begin{centering}
    \includegraphics[width=0.5\textwidth]{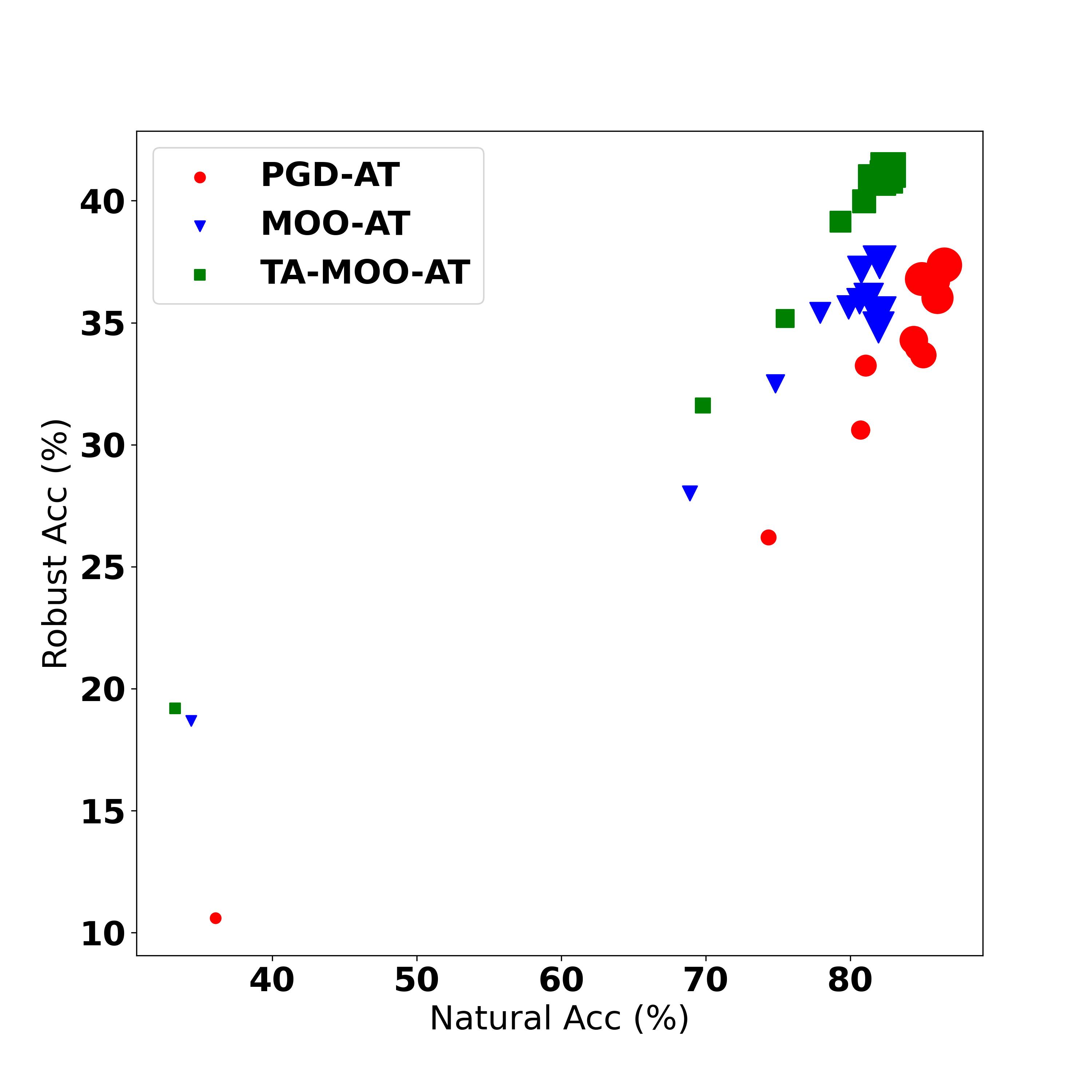}
    \par\end{centering}
    \label{fig:sensitivity-eot}
    }
    \par\end{centering}
    \caption{Comparison progress of three adversarial training methods. The bigger marker size represents the later epoch. Each point represents the natural accuracy and robust accuracy against PGD-Linf attack on the testing set.}
    \label{fig:ens-adv-training-progress-rme}
\end{figure}

\subsection{Universal Perturbation (UNI)\label{subsec:sup-uni-exp}}
\paragraph{Additional experimental results.}
In addition to the experiments on ResNet18 as reported in Table \ref{tab:cf10-cf100-uni}, we would like to provide additional experimental results on two other adversarial trained models VGG16 and EfficientNet as shown in Table \ref{tab:cf10-cf100-uni-full}. It can be seen that, TA-MOO consistently achieves the best attacking performance on ResNet18 and VGG16, on both CIFAR10 and CIFAR100 datasets, with $K \geq 8$.

\begin{center}
\begin{table}
\begin{centering}
\caption{Evaluation of generating Universal Perturbation on the CIFAR10 and
CIFAR100 datasets. R: ResNet18, V: VGG16, E: EfficientNet. \label{tab:cf10-cf100-uni-full}}
\resizebox{\textwidth}{!}{ %
\begin{tabular}{llccccccccccc}
 &  & \multicolumn{5}{c}{CIFAR10} &  & \multicolumn{5}{c}{CIFAR100}\tabularnewline
 &  & K=4 & K=8 & K=12 & K=16 & K=20 &  & K=4 & K=8 & K=12 & K=16 & K=20\tabularnewline
\midrule 
\multirow{4}{*}{R} & Uniform & 37.52 & 30.34 & 27.41 & 25.52 & 24.31 &  & 65.40 & 58.99 & 55.33 & 53.02 & 51.49\tabularnewline
 & MinMax & \textbf{50.13} & 33.68 & 20.46 & 15.74 & 14.73 &  & \textbf{74.73} & 62.29 & 52.05 & 45.26 & 42.33\tabularnewline
 & MOO & 43.80 & \uuline{35.92} & \uuline{31.41} & \uuline{28.75} & \uuline{26.83} &  & 69.35 & \uuline{62.72} & \uuline{57.72} & \uuline{54.12} & \uuline{52.25}\tabularnewline
 & TA-MOO & \uuline{48.00} & \textbf{39.31} & \textbf{34.96} & \textbf{31.84} & \textbf{30.12} &  & \uuline{72.74} & \textbf{68.06} & \textbf{62.33} & \textbf{57.48} & \textbf{54.12}\tabularnewline
\midrule 
\multirow{4}{*}{V} & Uniform & 37.76 & 30.81 & 27.49 & 25.94 & 24.46 &  & 66.87 & 61.49 & 58.53 & 56.29 & 54.98\tabularnewline
 & MinMax & \textbf{47.96} & 30.88 & 20.20 & 16.93 & 16.25 &  & \textbf{78.58} & \uuline{69.14} & 58.85 & 51.81 & 48.09\tabularnewline
 & MOO & 43.04 & \uuline{34.56} & \uuline{30.07} & \uuline{27.43} & \uuline{25.42} &  & 73.46 & 66.51 & \uuline{61.28} & \uuline{57.88} & \uuline{56.09}\tabularnewline
 & TA-MOO & \uuline{46.58} & \textbf{38.33} & \textbf{32.32} & \textbf{29.16} & \textbf{26.56} &  & \uuline{75.57} & \textbf{71.86} & \textbf{67.22} & \textbf{62.99} & \textbf{59.19}\tabularnewline
\midrule 
\multirow{4}{*}{E} & Uniform & 44.86 & \uuline{39.03} & \uuline{36.37} & \uuline{34.65} & \uuline{33.49} &  & 67.55 & \uuline{60.99} & \uuline{57.35} & \textbf{54.84} & \textbf{53.57}\tabularnewline
 & MinMax & 44.47 & 32.96 & 28.86 & 27.01 & 26.47 &  & \uuline{69.69} & 57.99 & 50.93 & 45.59 & 43.87\tabularnewline
 & MOO & \uuline{45.31} & \textbf{39.28} & \textbf{36.44} & \textbf{34.72} & \textbf{33.51} &  & 66.68 & 59.69 & 54.95 & 53.20 & \uuline{51.43}\tabularnewline
 & TA-MOO & \textbf{46.74} & 37.95 & 33.95 & 31.71 & 30.41 &  & \textbf{70.40} & \textbf{63.78} & \textbf{58.17} & \uuline{53.26} & 50.66\tabularnewline
\bottomrule
\end{tabular}}
\par\end{centering}
\centering{}\vspace{-4mm}
\end{table}
\par\end{center}

\paragraph{Why does MOO work?}
As shown in Table \ref{tab:cf10-cf100-uni-full}, MOO consistently achieves better performance than the Uniform strategy (except for the setting with EfficientNet on the CIFAR100 dataset). 
To find out the reason for the improvement, we investigate the gradient norm $\| \nabla_{\delta} f(\delta) \|$ and weight $w$ for the first, and second groups (as an example) and the average over 100 groups of the testset as shown in Table \ref{tab:uni-moouni-grad-strenth}. 
It can be seen that in the first and second groups, there are some tasks that have significantly low gradient strengths than other tasks. The gap of the strongest/weakest gradient strength can be a magnitude of $10^6$ indicating the domination of one task over others. While this issue can cause the failure as in the ENS setting, however, in the UNI setting, the lowest gradient strengths in each group correspond to unsuccessful tasks (unsuccessful adversarial examples) and vice versa. 
Recall that we use the multi-gradient descent algorithm to solve MOO, which in principle assigns a higher weight for a weaker gradient vector. Therefore, in the UNI setting, while the dominating issue still exists, fortunately, the result still fits our desired weighting strategy (i.e., higher weight for an unsuccessful task and vice versa). Moreover, when there are a large number of groups (i.e., 100 groups), the issue of dominating tasks is alleviated. The average gradient strength is more balanced as shown in Table \ref{tab:uni-moouni-grad-strenth}. This explains the improvement of MOO over the Uniform strategy in the UNI setting.

\begin{center}
\begin{table}
\caption{Evaluation of generating Universal Perturbation (K=8) on the CIFAR10 dataset with ResNet18 architecture and MOO method. $\{T_i\}_{i=1}^K$ represents value for each task (i.e., a sample in a group). $w_1/w_2$ represents the weight of the first/second group of K samples, while $w$ represents the the statistic of weight over all groups (mean$\pm$std). $\| \nabla_{\delta_1} f_i(\delta_1)\|$ / $\|  \nabla_{\delta_2} f_i(\delta_2) \| $ represents the gradient norm of the first/second group of K samples, while $\| \nabla_{\delta} f_i(\delta) \|$ represents the statistic of gradient norm over all groups (mean$\pm$std). 
$\mathbb{I}_0 / \mathbb{I}_1$ represents the indicator function for a successful (1) or unsuccessful (0) task, while $\mathbb{I}$ represents the the statistic of successful rate over all groups.} 

\centering{}%
\resizebox{\columnwidth}{!}{%
\begin{tabular}{lcccccccc}
 & $T_{1}$ & $T_{2}$ & $T_{3}$ & $T_{4}$ & $T_{5}$ & $T_{6}$ & $T_{7}$ & $T_{8}$\tabularnewline
\midrule 
$\| \nabla_{\delta_1} f_i(\delta_1) \| $ & 1.15e1 & 3.45e-5 & 1.97e-2 & 1.26e-4 & 1.27e0 & 1.04e-1 & 1.04e1 & 9.91e0\tabularnewline
$w_0$ &  0.0238 & 0.1861  & 0.1859  &  0.1861 & 0.1763 & 0.1862 & 0.0257  & 0.0299 \tabularnewline
$\mathbb{I}_0$ & 1  & 0  & 0  & 0  & 0 & 0 & 1  & 1 \tabularnewline
$\| \nabla_{\delta_2} f_i(\delta_2) \| $ & 9.70e0 & 1.59e1 & 4.32e-4 & 4.27e-4 & 1.25e1 & 6.23e-5 & 2.91e-5 & 6.17e-6\tabularnewline
$w_1$ & 0.0341 & 0.0167 & 0.1854 & 0.1854 & 0.0222 & 0.1854 & 0.1854 & 0.1854\tabularnewline
$\mathbb{I}_1$ & 1  & 1  & 0  & 0  & 1 & 0 & 0  & 0 \tabularnewline
$\| \nabla_{\delta} f_i(\delta) \| $ & 4.93$\pm$6.63 & 4.23$\pm$6.97 & 5.18$\pm$7.42 & 3.84$\pm$5.83 & 4.39$\pm$6.04 & 6.66$\pm$7.64 & 4.82$\pm$7.48 & 5.25$\pm$7.17\tabularnewline
$w$ & 0.12$\pm$0.08 & 0.14$\pm$0.09 & 0.12$\pm$0.08 & 0.13$\pm$0.08 & 0.12$\pm$0.09 & 0.10$\pm$0.08 & 0.14$\pm$0.10 & 0.11$\pm$0.08\tabularnewline
$\mathbb{I}$ &  0.38$\pm$0.49 & 0.28$\pm$0.46 & 0.36$\pm$0.48 & 0.32$\pm$0.48 & 0.38$\pm$0.48 & 0.48$\pm$0.50 & 0.32$\pm$0.47 & 0.40$\pm$0.49 \tabularnewline
\bottomrule
\end{tabular}
}
\label{tab:uni-moouni-grad-strenth}
\end{table}
\par\end{center}

\subsection{Robust Adversarial Examples against Transformations (EoT)\label{subsec:sup-aug-exp}}

We observed that in EoT with the stochastic setting, adjusting
gamma sometimes has the overflow issue resulting in an infinite gradient.
Recall that our method using MGDA to solve MOO which relies on the stability of gradient strengths. Therefore, in the case of having infinite gradients, learning weight $w$ is unstable, resulting to lower performance in both MOO and TA-MOO. 

To overcome the overflow issue, we allocate memory to cache the valid gradient of each task in the previous iteration and replace the infinite value in the current iteration with the valid one in the memory. The storage only requires a tensor with the same shape as the gradient (i.e., as the exact size of the input), therefore, it does not increase the computation resource significantly. 
As shown in Table \ref{tab:eot-det-sto}, this simple technique helps to improve performance of TA-MOO by 5.3\% on both the CIFAR10 and CIFAR100 datasets. It also helps to improve performance of MOO by  0.8\% and 4.8\%, respectively. Finally, after overcoming the gradient issue, the TA-MOO achieves the best performance on the CIFAR100 dataset and the second best performance on the CIFAR10 dataset (0.4\% lower in ASR-All but 0.8\% higher in ASR-Avg when comparing to MinMax). This result provides additional evidence of the advantage of our method.   

\begin{center}
    \begin{table}
    \caption{Robust adversarial examples against transformations evaluation. The
    highest/second highest performance is highlighted in \textbf{Bold}/\uuline{Underline}.
    MOO$^{\star}$ and TA-MOO$^{\star}$ represent version with memory
    to overcome the infinite gradient issue in the stochastic setting.
    \label{tab:eot-det-sto}}
    
    \centering{}%
    \begin{tabular}{llcccc}
     &  & \multicolumn{2}{c}{Deterministic} & \multicolumn{2}{c}{Stochastic}\tabularnewline
     &  & A-All & A-Avg & A-All & A-Avg\tabularnewline
    \midrule 
    \multirow{6}{*}{C10} & Uniform & \textcolor{blue}{25.98} & \textbf{55.33} & \textcolor{blue}{31.47} & \textbf{50.55}\tabularnewline
     & MinMax & \textcolor{blue}{\uuline{30.54}} & 52.20 & \textbf{\textcolor{blue}{33.35}} & 49.44\tabularnewline
     & MOO & \textcolor{blue}{21.25} & 49.81 & \textcolor{blue}{26.97} & 43.84\tabularnewline
     & TA-MOO & \textbf{\textcolor{blue}{31.10}} & \uuline{55.26} & \textcolor{blue}{28.26} & 45.67\tabularnewline
     & MOO$^{\star}$ & - & - & \textcolor{blue}{27.79} & 45.91\tabularnewline
     & TA-MOO$^{\star}$ & - & - & \textcolor{blue}{\uuline{32.96}} & \uuline{50.27}\tabularnewline
    \midrule 
    \multirow{6}{*}{C100} & Uniform & \textcolor{blue}{56.19} & \uuline{76.23} & \textcolor{blue}{59.89} & \uuline{73.73}\tabularnewline
     & MinMax & \textcolor{blue}{\uuline{59.75}} & 75.72 & \textcolor{blue}{\uuline{61.30}} & 73.59\tabularnewline
     & MOO & \textcolor{blue}{53.17} & 74.21 & \textcolor{blue}{54.96} & 69.26\tabularnewline
     & TA-MOO & \textbf{\textcolor{blue}{60.88}} & \textbf{76.71} & \textcolor{blue}{56.23} & 69.91\tabularnewline
     & MOO$^{\star}$ & - & - & \textcolor{blue}{58.79} & 72.81\tabularnewline
     & TA-MOO$^{\star}$ & - & - & \textbf{\textcolor{blue}{61.54}} & \textbf{74.07}\tabularnewline
    \bottomrule
    \end{tabular}
    \end{table}
\par\end{center}

\begin{center}
    \begin{table}
    \centering{}\caption{Robust adversarial examples against transformations evaluation. The
    highest/second highest performance is highlighted in \textbf{Bold}/\uuline{Underline}. 
    The most important metric is emphasized in blue color. 
    MOO$^{\star}$ and TA-MOO$^{\star}$ represent version with memory
    to overcome the infinite gradient issue in the stochastic setting.
    I: Identity, H: Horizontal flip, V: Vertical flip, C: Center crop,
    G: Adjust gamma, B: Adjust brightness, R: Rotation.\label{tab:eot-det-sto-all}}
    \begin{tabular}{llccccccccc}
     &  & \textcolor{blue}{A-All} & A-Avg & I & H & V & C & G & B & R\tabularnewline
    \midrule 
    \multirow{4}{*}{D-C10} & Uniform & \textcolor{blue}{25.98} & \textbf{55.33} & \textbf{44.85} & 41.58 & 82.90 & \textbf{72.56} & \textbf{45.92} & \textbf{49.59} & \textbf{49.93}\tabularnewline
     & MinMax & \textcolor{blue}{\uuline{30.54}} & 52.20 & 43.31 & \uuline{41.59} & 78.80 & 64.83 & 44.38 & 46.53 & 45.97\tabularnewline
     & MOO & \textcolor{blue}{21.25} & 49.81 & 36.23 & 33.93 & \textbf{87.47} & 71.05 & 37.68 & 40.21 & 42.12\tabularnewline
     & TA-MOO & \textbf{\textcolor{blue}{31.10}} & \uuline{55.26} & \uuline{44.15} & \textbf{41.86} & \uuline{85.19} & \uuline{71.86} & \uuline{45.53} & \uuline{48.70} & \uuline{49.54}\tabularnewline
    \midrule 
    \multirow{4}{*}{D-C100} & Uniform & \textcolor{blue}{56.19} & \uuline{76.23} & \textbf{70.43} & 69.01 & 87.66 & \uuline{87.36} & 71.40 & \uuline{74.25} & \uuline{73.47}\tabularnewline
     & MinMax & \textcolor{blue}{\uuline{59.75}} & 75.72 & \uuline{70.13} & \uuline{69.26} & 87.45 & 86.03 & \uuline{71.54} & 73.30 & 72.32\tabularnewline
     & MOO & \textcolor{blue}{53.17} & 74.21 & 66.96 & 65.68 & \textbf{89.16} & 87.03 & 68.49 & 71.11 & 71.06\tabularnewline
     & TA-MOO & \textbf{\textcolor{blue}{60.88}} & \textbf{76.71} & \textbf{70.43} & \textbf{69.37} & \uuline{89.11} & \textbf{87.95} & \textbf{71.70} & \textbf{74.73} & \textbf{73.69}\tabularnewline
    \midrule 
    \multirow{6}{*}{S-C10} & Uniform & \textcolor{blue}{31.47} & \textbf{50.55} & \textbf{48.58} & \uuline{44.70} & \textbf{65.52} & 51.14 & \textbf{47.43} & \textbf{48.76} & \uuline{47.70}\tabularnewline
     & MinMax & \textbf{\textcolor{blue}{33.35}} & 49.44 & 47.35 & 44.45 & 62.78 & \uuline{51.75} & 46.32 & 47.13 & 46.34\tabularnewline
     & MOO & \textcolor{blue}{26.97} & 43.84 & 40.62 & 38.45 & 57.65 & 48.55 & 40.41 & 40.71 & 40.47\tabularnewline
     & TA-MOO & \textcolor{blue}{28.26} & 45.67 & 42.80 & 39.66 & 61.98 & 47.92 & 41.80 & 43.01 & 42.54\tabularnewline
     & MOO$^{\star}$ & \textcolor{blue}{27.79} & 45.91 & 42.43 & 39.65 & 62.11 & 51.44 & 41.62 & 42.21 & 41.92\tabularnewline
     & TA-MOO$^{\star}$ & \textcolor{blue}{\uuline{32.96}} & \uuline{50.27} & \uuline{48.18} & \textbf{45.26} & \uuline{62.97} & \textbf{52.49} & \uuline{47.03} & \uuline{48.22} & \textbf{47.76}\tabularnewline
    \midrule 
    \multirow{6}{*}{S-C100} & Uniform & \textcolor{blue}{59.89} & \uuline{73.73} & \textbf{73.19} & \textbf{71.15} & 79.73 & 74.81 & \uuline{72.05} & \uuline{73.10} & \textbf{72.10}\tabularnewline
     & MinMax & \textcolor{blue}{\uuline{61.30}} & 73.59 & 72.44 & 70.55 & 80.04 & \uuline{75.55} & 71.99 & 72.49 & \textbf{72.10}\tabularnewline
     & MOO & \textcolor{blue}{54.96} & 69.26 & 67.62 & 66.11 & 75.88 & 72.72 & 66.87 & 68.11 & 67.49\tabularnewline
     & TA-MOO & \textcolor{blue}{56.23} & 69.91 & 68.52 & 66.92 & 76.70 & 72.71 & 67.57 & 68.97 & 67.97\tabularnewline
     & MOO$^{\star}$ & \textcolor{blue}{58.79} & 72.81 & 71.58 & 69.08 & \uuline{80.17} & 75.01 & 70.78 & 71.71 & 71.33\tabularnewline
     & TA-MOO$^{\star}$ & \textbf{\textcolor{blue}{61.54}} & \textbf{74.07} & \uuline{72.95} & \uuline{70.95} & \textbf{80.94} & \textbf{76.22} & \textbf{72.22} & \textbf{73.21} & \uuline{72.00}\tabularnewline
    \bottomrule
    \end{tabular}
    \end{table}
\par\end{center}

\subsection{Generating Speed Comparison and Experiments' Stability \label{subsec:sup-speed-comparison}}

\paragraph{Generating Speed Comparison. }
Table \ref{tab:compute-time} shows the average time to generate one
adversarial example in each setting. The results are measured on the
CIFAR10 dataset with ResNet18 architecture in the Ensemble of Transformations
(EoT) and Universal Perturbation (Uni) settings. We use 1 Titan RTX
24GB for the EoT experiment and 4 Tesla V100 16GB each for the other experiments.
It is worth mentioning that our primary focus in this paper is showing 
the advantage of MOO and the Task-Oriented regularization in generating adversarial examples. Therefore, we did not try to optimize our implementation in terms of generating time. 

\begin{table}
\caption{Average time per sample for generating adversarial example. All experiments
are measured on the CIFAR10 dataset, EoT and Uni are with ResNet18
architecture.\label{tab:compute-time}}

\centering{}%
\begin{tabular}{ccccc}
 & Ensemble (K=4) & EoT (K=7) & Uni@K=12 & Uni@K=20\tabularnewline
\hline 
Uniform & 640ms & 350ms & 1850ms & 3030ms\tabularnewline
MinMax & 1540ms & 610ms & 1210ms & 2080ms\tabularnewline
MOO & 1770ms & 1130ms & 5600ms & 9280ms\tabularnewline
TA-MOO & 1960ms & 1200ms & 5870ms & 9500ms\tabularnewline
\hline 
\end{tabular}
\end{table}

\paragraph{Experiments' Stability. }
We conduct an experiment with 5 different random seeds to generate adversarial examples for the ENS setting 
to evaluate the stability of experimental results on choosing of random seed. The experiment is on the CIFAR10 dataset, 
with an ensemble of 4 architectures including ResNet18, VGG16, GoogLeNet, and EfficientNet. 
We report mean and variation values in Table \ref{tab:sup-stability-study}. 
It can be observed that there is a slight variation in attack performances across
methods. The variation is small enough compared to the gap between methods (i.e., the biggest
variation is 0.32\% in SAR-All while the smallest gap is 2.51\% between MOO and the Uniform approach),
therefore, making the comparison still reliable. 

\begin{table}
    \caption{Stability of experiments' evaluation on different random seeds. Experiment
    on the ENS setting, with an ensemble of 4 models: Resnet18, VGG16,
    GoogleNet and EfficientNet. \label{tab:sup-stability-study}}
    
    \begin{centering}
    \resizebox{\textwidth}{!}{
    \begin{tabular}{lcccccc}
     & A-All & A-Avg & R & V & G & E\tabularnewline
    \hline 
    Uniform & 28.12 $\pm$ 0.09 & 48.29 $\pm$ 0.05 & 48.81 $\pm$ 0.08 & 49.06 $\pm$ 0.08 & 48.27 $\pm$ 0.10 & 47.06 $\pm$ 0.03\tabularnewline
    MOO & 25.61 $\pm$ 0.36 & 45.13 $\pm$ 0.30 & 39.84 $\pm$ 0.62  & 47.29 $\pm$ 0.36 & 37.51 $\pm$ 0.36 & 55.90 $\pm$ 0.17\tabularnewline
    TA-MOO & 37.56 $\pm$ 0.32 & 51.15 $\pm$ 0.21 & 49.37 $\pm$ 0.15  & 52.80 $\pm$ 0.45 & 48.98 $\pm$ 0.25 & 53.24 $\pm$ 0.13\tabularnewline
    \hline 
    \end{tabular}}
    \par\end{centering}
\end{table}

\subsection{Sensitivity to Hyper-parameters} \label{sec:alb-sensitivity}
In this section we provide an analytical experiment on the sensitivity of our TA-MOO method to the tradeoff $\lambda$. 
The study has been conducted with the ENS setting with CE loss and the EoT setting with deterministic transformations using ResNet18 architecture. 
All experiments are on the CIFAR10 dataset. The value of $\lambda$ is changed from 1 to 1000. 
It can be observed from Figure \ref{fig:sensitivity-ens} (the ENS setting) that 
(i) increasing $\lambda$ reduces the performance of dominated task (i.e., ASR on the EfficientNet decreases from 54.49\% at $\lambda=1$ to 53.40\% at $\lambda=100$) while increases performances of other tasks. 
In overall, it significantly increases the ASR-All performance of the entire ensemble from 29.14\% at $\lambda=1$ to 38.01\% at $\lambda=100$. 
(ii) However, over-high $\lambda$ (i.e., $\lambda > 200$) leads to the drop of performance in all tasks, resulting in a lower overall performance.   

A similar observation can be seen in the EoT setting in Figure \ref{fig:sensitivity-eot}. 
The attack performance on the dominated task (V-Vertical flipping) decreases from 86.11\% at $\lambda=50$ to 83.67\% at $\lambda=200$. 
In contract, in the same range of $\lambda$ the overall performance increases from 32.85\% to 34.36\%. 
The performances of all tasks decrease when using too large $\lambda$ (i.e., $\lambda > 200$). 
Based on the result of this study, we choose $\lambda=100$ in all the other experiments.  

\begin{figure}
    \begin{centering}
    \subfloat[ENS setting]{
    \begin{centering}
    \includegraphics[width=0.5\textwidth]{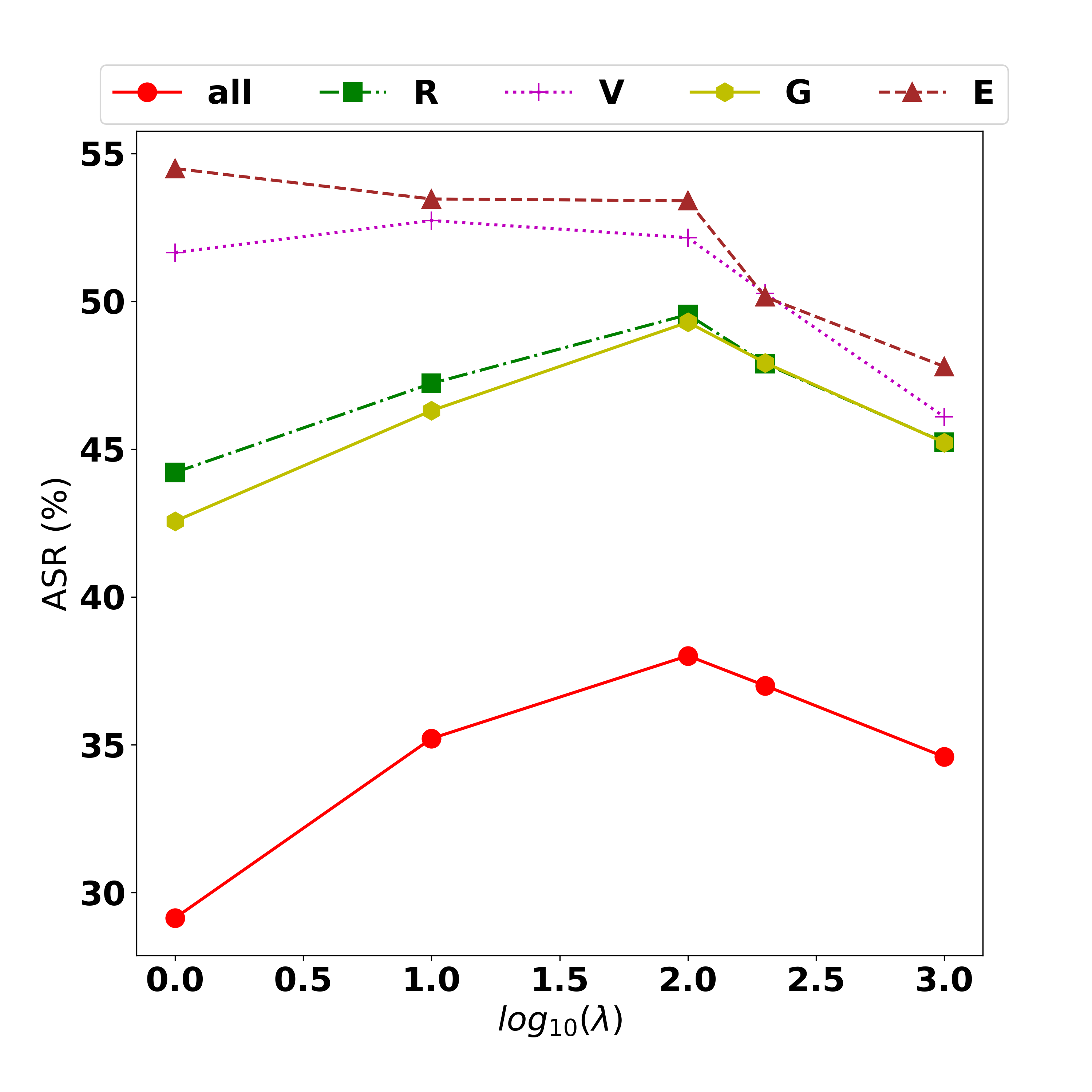}
    \par\end{centering}
    \label{fig:sensitivity-ens}
    }
    \subfloat[EoT setting]{\begin{centering}
    \includegraphics[width=0.5\textwidth]{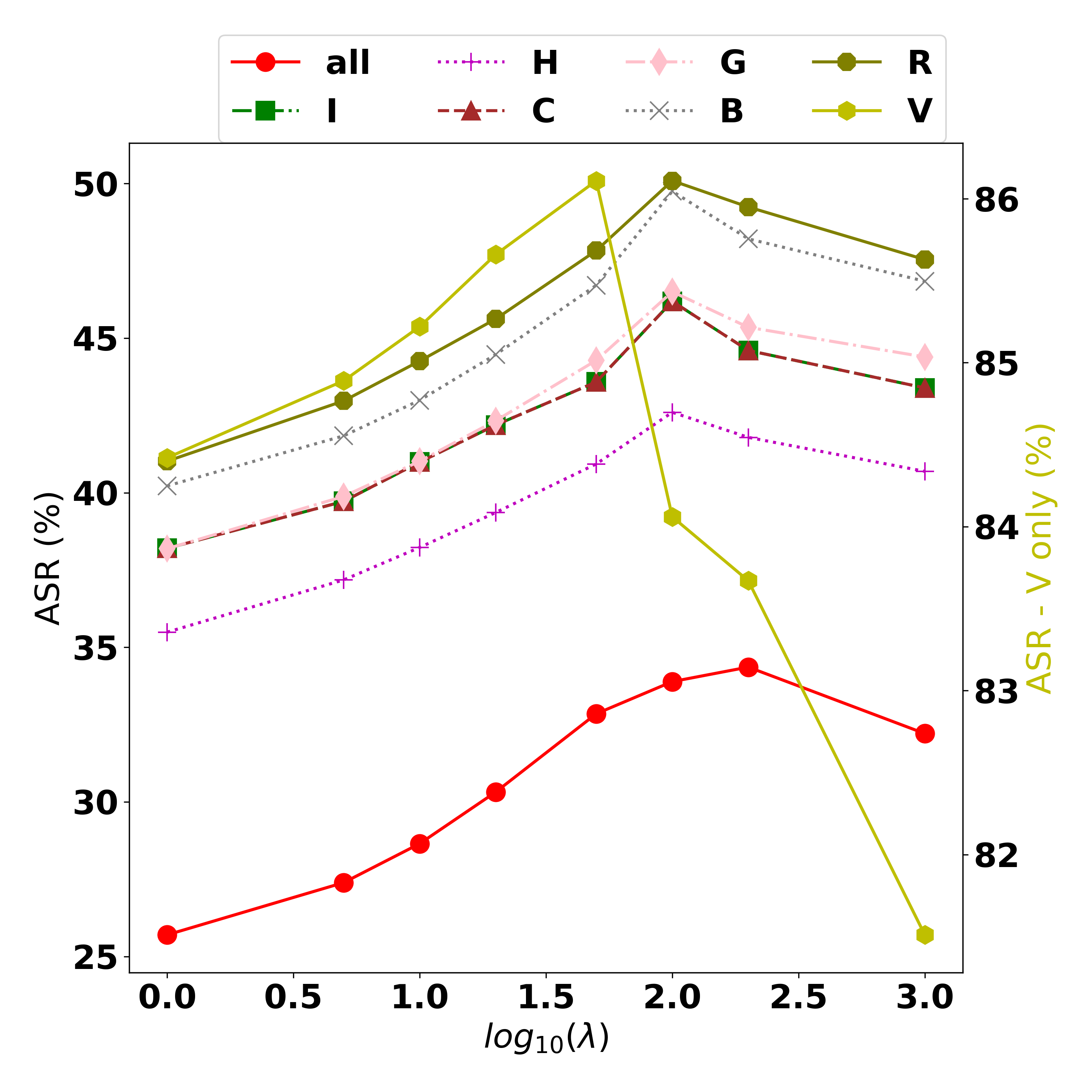}
    \par\end{centering}
    \label{fig:sensitivity-eot}
    }
    \par\end{centering}
    \caption{Sensitivity to the parameter $\lambda$.}
    \label{fig:sensitivity-lambda}
\end{figure}

\subsection{Comparison with Standard Attacks} \label{sec:alb-compare-attacks}
We conducted an additional comparison on the ENS setting to further
confirm the effectiveness of our method over standard adversarial
attacks (which consider an entire ensemble as a single model). More
specifically, we compare with AutoAttack \citep{croce2020reliable}, Brendel-Bethge attack (BB) \citep{brendel2019accurate},
Carlini-Wagner attack (CW) \citep{carlini2017towards}, and PGD attack \citep{madry2017towards}. 
For AutoAttack, we use the standard version which includes 4 different attacks. 
For BB attack, we initialized with the PGD attack with 20 steps. 
For CW attack, we set the confidence factor to 1.0. 
We evaluate these attacks on 2 ensemble settings, a diverse (D) ensemble set with
4 different architectures (ResNet18, VGG16, GoogLeNet, and EfficientNet)
and a non-diverse (ND) ensemble set with 4 ResNet18 architectures.

It can be seen from the Table \ref{tab:compare-attacks} that our TA-MOO attack consistently
achieves the best attack performance, with a significant gap compared
to the best standard attack. More specifically, our TA-MOO method
achieves 38.01\% (SAR-All metric) on the diverse ensemble set, while
the second best attack is AutoAttack with 30.71\% (a gap of 7.3\%).
On the non-diverse set, the gap between our TA-MOO and AutoAttack
is still notably large at 4\%. These standard attacks consider an entire ensemble 
as a single model, i.e., aim to optimize a single objective given a single ensemble output. 
Therefore, they cannot guarantee a successful attack on each member. 

\begin{center}
    \begin{table}
    \centering{}\caption{Attacking Ensemble model with a diverse set D=\{R-ResNet18, V-VGG16,
    G-GoogLeNet, E-EfficientNet\} and non-diverse set ND=\{4 ResNets\}.
    Experiment on the CIFAR10 dataset with cross-entropy objective loss.
    The most important metric is emphasized in blue.\label{tab:compare-attacks}}
    \begin{tabular}{llrrrrrr}
     &  & \textcolor{blue}{A-All} & A-Avg & R/R1 & V/R2 & G/R3 & E/R4\tabularnewline
    \hline 
    \multirow{6}{*}{D} & PGD & \textcolor{blue}{28.21} & \uuline{48.34} & \uuline{48.89} & \uuline{49.08} & \uuline{48.38} & 47.03\tabularnewline
     & CW & \textcolor{blue}{6.10} & 16.63 & 13.53 & 15.76 & 11.74 & 25.47\tabularnewline
     & B\&B & \textcolor{blue}{6.67} & 38.03 & 37.95 & 38.92 & 35.58 & 39.68\tabularnewline
     & AutoAttack & \textcolor{blue}{\uuline{30.71}} & 45.49 & 48.32 & 45.83 & 47.25 & 40.56\tabularnewline
     & MOO & \textcolor{blue}{25.16} & 44.76 & 39.06 & 46.83 & 37.05 & \textbf{56.11}\tabularnewline
     & TA-MOO & \textbf{\textcolor{blue}{38.01}} & \textbf{51.10} & \textbf{49.55} & \textbf{52.15} & \textbf{49.29} & \uuline{53.40}\tabularnewline
    \hline 
    \multirow{6}{*}{ND} & PGD & \textcolor{blue}{28.17} & 48.75 & 51.94 & 45.55 & 54.15 & 43.34\tabularnewline
     & CW & \textcolor{blue}{4.71} & 13.86 & 14.92 & 12.71 & 17.51 & 10.31\tabularnewline
     & B\&B & \textcolor{blue}{5.29} & 40.51 & 49.06 & 35.19 & 48.63 & 29.16\tabularnewline
     & AutoAttack & \textcolor{blue}{\uuline{37.00}} & 49.32 & 51.07 & 48.58 & 51.08 & 46.55\tabularnewline
     & MOO & \textcolor{blue}{32.50} & \uuline{52.21} & \uuline{53.25} & \uuline{49.05} & \uuline{56.80} & \uuline{49.76}\tabularnewline
     & TA-MOO & \textbf{\textcolor{blue}{41.01}} & \textbf{57.33} & \textbf{58.88} & \textbf{55.32} & \textbf{60.81} & \textbf{54.29}\tabularnewline
    \hline 
    \end{tabular}
    \end{table}
\par\end{center}

\subsection{Attacking the ImageNet dataset} \label{subsec:attack-imagenet}

\paragraph{Experimental Setting. }

We conduct experiments on the ENS setting using the adversarial pre-trained
models on the RobustBench \citep{croce2021robustbench}. We use two sets of an ensemble
to verify the importance of our task-oriented strategy. The first
set is the robust ensemble (RE) set including 3 robust models: ResNet18
(model ID: Salman2020Do\_R18 \citep{salman2020adversarially}, robust accuracy 25.32\%), ResNet50 (model
ID: Salman2020Do\_R50 \citep{salman2020adversarially}, robust accuracy 34.96\%) and ResNet50 (model
ID: Wong2020Fast \citep{wong2019fast}, robust accuracy 26.24\%). The second set is the
less-robust ensemble (LE) which includes 3 models: ResNet18 (model
ID: Salman2020Do\_R18), ResNet50 (model ID: Salman2020Do\_R50) and
the standard training ResNet50 (model ID: Standard\_R50, robust accuracy
0\%). We use both targeted attack and untargeted attack settings,
with $\epsilon=4/255$ , and $\eta=1/255$ with 20 steps. We use 5000 images of the
validation set to evaluate.

\paragraph{Experimental Results.}

We report experimental results with different settings in Table
\ref{tab:ens-imagenet}, where RE/LE/TAR/UNTAR represents Robust Ensemble/Less-Robust
Ensemble/Targeted Attack/Untargeted Attack, respectively. It can be
seen that, in the robust ensemble setting (RE-TAR and RE-UNTAR), our
MOO achieves a similar performance compared to the baseline, while
TA-MOO has a further improvement over MOO. The gap of SAR-All between
TA-MOO and the uniform weighting strategy is 0.1\% in the targeted
attack setting (RE-TAR), while that in the untargeted attack setting
is 1.2\%. In the less-robust ensemble setting (LE-TAR and LE-UNTAR),
the improvement of our methods over the baseline is higher than in
the robust ensemble setting. With the gap of SAR-All between TA-MOO
and the uniform strategy is 0.38\% with the targeted attack setting
(LE-TAR), while the gap in the untargeted setting (LE-UNTAR) is 15.22\%
a significantly higher. While it is acknowledged that the targeted
attack is a more common protocol in attacking the ImageNet dataset \citep{athalye2018obfuscated}, however, we believe that our significant improvement on the
untargeted attack is still worth noting.

\begin{center}
    \begin{table}
    \centering{}\caption{Evaluation attacking performance on the ImageNet dataset. RE/LE/TAR/UNTAR
    represents Robust Ensemble/Less-Robust Ensemble/Targeted Attack/Untargeted
    Attack, respectively. R18/R50/STD represents robust ResNet18, robust
    ResNet50 and standard ResNet50 pre-trained model, respectively. The
    most important metric is emphasized in blue.\label{tab:ens-imagenet} }
    \begin{tabular}{llrrrrr}
     &  & \textcolor{blue}{A-All} & A-Avg & R18/R18 & R50/R50 & R50/STD\tabularnewline
    \hline 
    RE-TAR & Uniform & \textcolor{blue}{29.58} & 39.38 & 42.50 & 32.22 & 43.42\tabularnewline
     & MOO & \textcolor{blue}{29.66} & \textbf{39.73} & 42.86 & \textbf{32.32} & 44.00\tabularnewline
     & TA-MOO & \textbf{\textcolor{blue}{29.68}} & \textbf{39.73} & \textbf{42.90} & 32.26 & \textbf{44.02}\tabularnewline
    \hline 
    LE-TAR & Uniform & \textcolor{blue}{30.30} & 58.14 & 42.36 & 32.06 & \textbf{100.0}\tabularnewline
     & MOO & \textcolor{blue}{30.66} & \textbf{58.37} & \textbf{42.70} & \textbf{32.48} & 99.94\tabularnewline
     & TA-MOO & \textbf{\textcolor{blue}{30.68}} & 58.25 & 42.54 & 32.36 & 99.86\tabularnewline
    \hline 
    RE-UNTAR & Uniform & \textcolor{blue}{48.58} & 60.11 & 64.22 & 51.72 & 64.38\tabularnewline
     & MOO & \textcolor{blue}{48.68} & \textbf{60.20} & \textbf{64.30} & 51.82 & \textbf{64.48}\tabularnewline
     & TA-MOO & \textbf{\textcolor{blue}{49.80}} & 59.71 & 63.80 & \textbf{52.38} & 62.94\tabularnewline
    \hline 
    LE-UNTAR & Uniform & \textcolor{blue}{34.24} & 61.01 & 46.98 & 36.28 & 99.78\tabularnewline
     & MOO & \textcolor{blue}{44.76} & 68.29 & 58.42 & 46.64 & 99.80\tabularnewline
     & TA-MOO & \textbf{\textcolor{blue}{49.46}} & \textbf{70.74} & \textbf{61.26} & \textbf{51.14} & \textbf{99.82}\tabularnewline
    \hline 
    \end{tabular}
    \end{table}
\par\end{center}

We conduct an additional experiment on the EoT setting
with the ImageNet dataset and report result in Table \ref{tab:eot-imagenet}. In this experiment, we use the robust pretrained
ResNet18 model (model ID: Salman2020Do\_R18) as the victim model.
We use the standard attack setting, i.e., targeted attack with $\epsilon=4/255,\eta=1/255$
with 20 steps. It can be seen that both MOO and TA-MOO could obtain
a better attack performance than the uniform strategy. It is a worth
noting that, in the experiment on the CIFAR10/CIFAR100 datasets (i.e.,
Table \ref{tab:eot-main-res} in the main paper) the dominating issue of the vertical filliping
exists and prevents MOO to obtain a better performance. In the ImageNet
dataset, the dominating issue is less serious, therefore, explains
the improvement of MOO and corroborates our hypothesis on the issue of dominating task.

\begin{center}
    \begin{table}
    \caption{Evaluation on the EoT setting with the ImageNet dataset. The most
    important metric is emphasized in blue.\label{tab:eot-imagenet}}
    
    \centering{}%
    \begin{tabular}{lccccccccc}
     & \textcolor{blue}{A-All} & A-Avg & I & H & V & C & G & B & R\tabularnewline
    \hline 
    Uniform & \textcolor{blue}{31.52} & 46.59 & 41.12 & \textbf{40.98} & 67.42 & 41.60 & 43.26 & 41.82 & 49.96\tabularnewline
    MOO & \textcolor{blue}{31.92} & 47.19 & 41.92 & 41.78 & 67.64 & \textbf{42.10} & 43.66 & 42.74 & 50.48\tabularnewline
    TA-MOO & \textbf{\textcolor{blue}{32.00}} & \textbf{47.21} & \textbf{41.94} & 41.80 & \textbf{67.66} & 42.06 & \textbf{43.70} & \textbf{42.80} & \textbf{50.52}\tabularnewline
    \hline 
    \end{tabular}
    \end{table}
\par\end{center}

\section{Additional Discussions}\label{sec:sup-discussions}

\begin{table}
\centering{}\caption{Evaluation of Attacking Ensemble model on the CIFAR10 (C10) and CIFAR100 (C100) datasets. The highest/second highest performance is highlighted in \textbf{Bold}/\uuline{Underline}. 
The table is copied from Table \ref{tab:ens-main-res} in the main paper for reading comprehension purpose.}
\label{tab:ens-main-res-supp}
\begin{tabular}{llcccccc}
\toprule 
 &  & \multicolumn{2}{c}{CW} & \multicolumn{2}{c}{CE} & \multicolumn{2}{c}{KL}\tabularnewline
 &  & \textcolor{blue}{A-All} & A-Avg & \textcolor{blue}{A-All} & A-Avg & \textcolor{blue}{A-All} & A-Avg\tabularnewline
\midrule 
\multirow{4}{*}{C10} & Uniform & \textcolor{blue}{26.37} & \textbf{41.13} & \textcolor{blue}{28.21} & 48.34 & \textcolor{blue}{17.44} & \uuline{32.85}\tabularnewline
 & MinMax & \textcolor{blue}{\uuline{27.53}} & \uuline{41.20} & \textcolor{blue}{\uuline{35.75}} & \textbf{51.56} & \textcolor{blue}{\uuline{19.97}} & \textbf{33.13}\tabularnewline
 & MOO & \textcolor{blue}{18.87} & 34.24 & \textcolor{blue}{25.16} & 44.76 & \textcolor{blue}{15.69} & 29.54\tabularnewline
 & TA-MOO & \textbf{\textcolor{blue}{30.65}} & 40.41 & \textbf{\textcolor{blue}{38.01}} & \uuline{51.10} & \textbf{\textcolor{blue}{20.56}} & 31.42\tabularnewline
\midrule
\multirow{4}{*}{C100} & Uniform & \textcolor{blue}{52.82} & \textbf{67.39} & \textcolor{blue}{55.86} & 72.62 & \textcolor{blue}{38.57} & \textbf{54.88}\tabularnewline
 & MinMax & \textcolor{blue}{\uuline{54.96}} & 66.92 & \textcolor{blue}{\uuline{63.70}} & \uuline{75.44} & \textcolor{blue}{\uuline{40.67}} & \uuline{53.83}\tabularnewline
 & MOO & \textcolor{blue}{51.16} & 65.87 & \textcolor{blue}{58.17} & 73.19 & \textcolor{blue}{39.18} & 53.44\tabularnewline
 & TA-MOO & \textbf{\textcolor{blue}{55.73}} & \uuline{67.02} & \textbf{\textcolor{blue}{64.89}} & \textbf{75.85} & \textbf{\textcolor{blue}{41.97}} & 53.76\tabularnewline
\bottomrule
\end{tabular}
\end{table}

\begin{table}
\centering{}\caption{Attacking Ensemble model with a diverse set D=\{R-ResNet18, V-VGG16, G-GoogLeNet, E-EfficientNet\} and non-diverse
set ND=\{4 ResNets\}. $w$ represents the final $w$ of MOO (mean $\pm$ std). $\| \nabla_{\delta} f_i(\delta) \|$ represents the gradient norm of each model (mean $\pm$ std). The table is copied from Table \ref{tab:ENS-non-diverse} in the main paper for reading comprehension purpose.}
\label{tab:ENS-non-diverse-sup}
\begin{tabular}{llcccccc}
 &  & \textcolor{blue}{A-All} & A-Avg & R/R1 & V/R2 & G/R3 & E/R4\tabularnewline
\midrule 
\multirow{3}{*}{D} 
& $\| \nabla_{\delta} f_i(\delta) \| $ & - & - & 7.15 $\pm$ 6.87 & 4.29 $\pm$ 4.64 & 7.35 $\pm$ 7.21 & 0.98 $\pm$ 0.72\tabularnewline
& $w$ & - & - & 0.15 $\pm$ 0.14 & 0.17 $\pm$0.13 & 0.15 $\pm$ 0.14 & 0.53 $\pm$ 0.29\tabularnewline
& Uniform & \textcolor{blue}{28.21} & 48.34 & 48.89 & 49.08 & 48.38 & 47.03\tabularnewline
 & MOO & \textcolor{blue}{25.16} & 44.76 & 39.06 & 46.83 & 37.05 & 56.11\tabularnewline
 & TA-MOO & \textcolor{blue}{38.01} & 51.10 & 49.55 & 52.15 & 49.29 & 53.40\tabularnewline
 \midrule 
\multirow{3}{*}{ND} 
& $\| \nabla_{\delta} f_i(\delta) \| $ & - & - & 8.41 $\pm$ 8.22 & 6.68$\pm$ 6.95 & 7.36 $\pm$ 6.03 & 5.67 $\pm$ 6.09\tabularnewline
& $w$ & - & - & 0.23 $\pm$ 0.21 & 0.24$\pm$0.17 & 0.23 $\pm$ 0.19 & 0.30 $\pm$ 0.21\tabularnewline
& Uniform & \textcolor{blue}{28.17} & 48.75 & 51.94 & 45.55 & 54.15 & 43.34\tabularnewline
 & MOO & \textcolor{blue}{32.50} & 52.21 & 53.25 & 49.05 & 56.80 & 49.76\tabularnewline
 & TA-MOO & \textcolor{blue}{41.01} & 57.33 & 58.88 & 55.32 & 60.81 & 54.29\tabularnewline
\bottomrule
\end{tabular}
\end{table}

\subsection{When does MOO Work? \label{subsec:sup-discussion-dominating-issue}}
\paragraph{The dominating issue.}
On  one hand, there is the dominating issue that happens in all the three settings. 
The issue can be recognized by the gap of attack performance among tasks. For example, 
in Table \ref{tab:ENS-non-diverse-sup} (i.e., the ENS setting with the diverse ensemble and MOO method), 
the gap between highest ASR (at EfficientNet) and lowest ASR (at GoogLeNet) is 19\%. 
In the EoT setting, the problem is even worse: The largest gap observed is 53.6\% as shown 
in Table \ref{tab:eot-det-sto-all} (the highest ASR is 88.19\% with Vertical flipping and the lowest ASR is 34.54\% with Horizontal flipping in with MOO - D-C10 setting). 
The dominating issue is also be recognized by the observation that a significant small gradient strength of one task on comparison with 
other tasks' strength. For example, in Table \ref{tab:ENS-non-diverse-sup} it can be seen that the gradient 
strength corresponding to the EfficientNet architecture (mean value is 0.98) is much lower than 
those of other architectures (mean values are at least 4.29). As the result, the weight corresponding to the EfficientNet architecture is much higher than those of others. 

The root of the dominating issue can be the natural of the setting (i.e., as shown in Table \ref{tab:eot-det-sto-all} with the EoT setting, when the 
domination of the Vertical flipping task can be observed in all methods) or because of the MOO solver which is discussed in Section \ref{sec:sup-gradient-des-discuss} 

\paragraph{Overcoming the dominating issue.}
On the other hand, if overcoming this issue, MOO can outperform the Uniform strategy. 
For example, on attacking the non-diverse ensemble model (i.e., 4 ResNets) MOO surpasses the Uniform strategy by 4.3\% and 3.5\% in the ASR-All and ASR-Avg metrics, respectively. On generating universal perturbations, MOO outperforms the Uniform strategy in most of the settings. 
As discussed in Section \ref{subsec:sup-aug-exp}, a simple memory caching trick can helps to overcome the infinite gradient issue and significantly boosts the performance of MOO or TA-MOO. Therefore, we believe that developing a technique to lessen the dominating issue might be a potential extension to further improve the performance. 

\paragraph{Balancing among goal-unachived tasks.}
We observed in the EoT setting, the dominating issue is strictly serious when gradients of some tasks are much weaker/stronger than others.
It is because of the natural of the transformation operations, therefore, this issue happens regardless status of the tasks.
In the set of goal-unachieved tasks' gradients can exist a dominated one, resulting to a much higher weight of the dominated task.  
Therefore, in order to strike a more balance among goal-unachieved tasks, we apply an additional regularization which minimizes the entropy 
of goal-unachieved weights $\mathcal{H}(w)= \sum_{i=s+1}^m - w_i \log w_i$. 
If all tasks have been achieved (i.e., $s=m$) then the additional regularization will be ignored.  
This additional regularization helps to improve further 2\% in the EoT setting.

\subsection{Importance of the Task-Oriented Regularization.\label{subsec:sup-discussion-task-oriented-reg}}

In this discussion, we would like to provide more experimental results in the ENS and EoT settings to further emphasize the contribution of the Task-Oriented regularization. 
Figure \ref{fig:ens-radar-tamoo} shows the ASR of each individual task in the ENS setting with three losses and the EoT setting with ResNet18 architecture and deterministic transformations.
As shown in Figure \ref{fig:ens-radar-tamoo-ens}, in the ENS setting, the MOO adversary produces a much higher ASR on the EfficientNet architecture than other architectures with any losses.
In contrast, the TA-MOO adversary has a lower ASR on the EfficientNet architecture but a much higher ASR on other architectures.   
Similar observation can be seen in Figure \ref{fig:ens-radar-tamoo-eot} such that the ASR corresponding to the V-flipping of MOO is slightly higher than 
that of TA-MOO, however, the ASR on other transformations of MOO is much lower than those of TA-MOO. 

\begin{figure}
    \begin{centering}
    \subfloat[ENS setting]{
    \begin{centering}
        \includegraphics[width=0.45\textwidth]{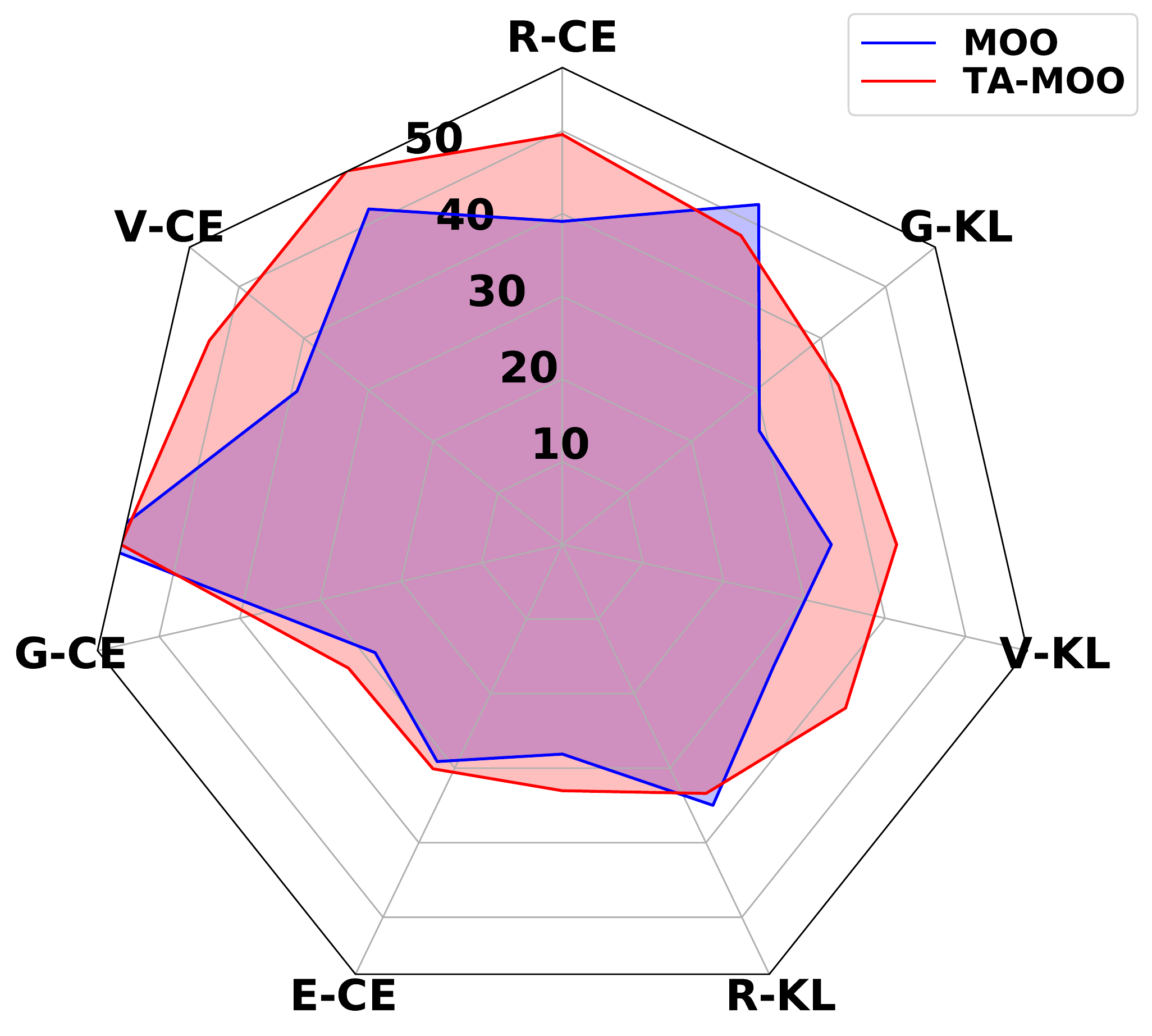}
    \par\end{centering}
    \label{fig:ens-radar-tamoo-ens}
    }
    \subfloat[EoT setting]{\begin{centering}
        \includegraphics[width=0.5\textwidth]{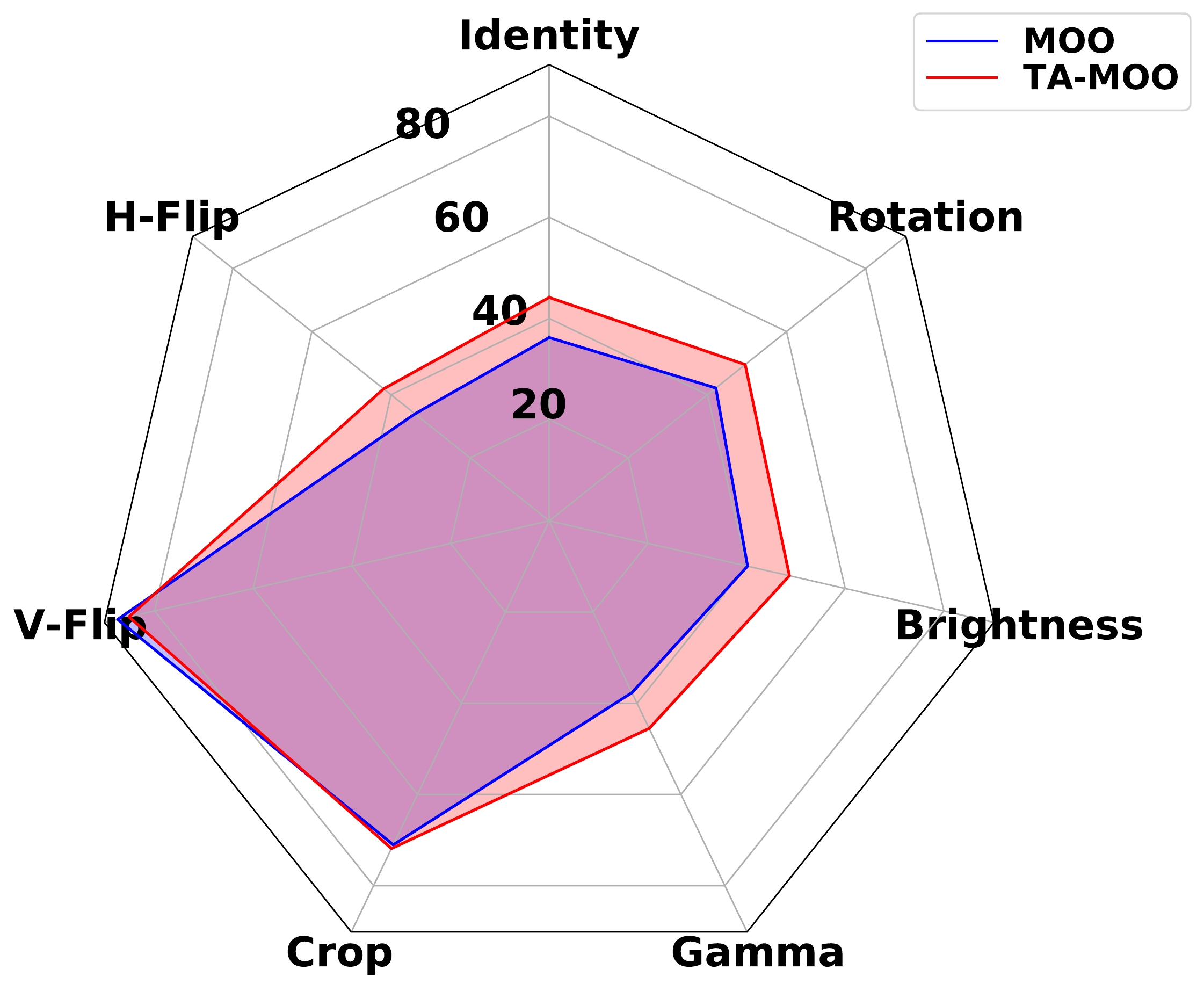}
    \par\end{centering}
    \label{fig:ens-radar-tamoo-eot}
    }
    \par\end{centering}
    \caption{Comparison on the ASR of each individual task. R: ResNet18, V: VGG16, G: GoogLeNet, E: EfficientNet. 
    CE: Cross-entropy loss, KL: Kullback-Leibler divergence, CW: Carnili-Wagner loss}
    \label{fig:ens-radar-tamoo}
\end{figure}

\subsection{More Efficient MOO Solvers} \label{sec:sup-gradient-des-discuss}

\paragraph{Discussions on the weighted-sum method.}
One of the most common approaches to solve the MOO problem is the scalarizing method, which formulates a single-objective optimization (SOO) such that the optimal solutions to the SOO problem are Pareto optimal solutions to the MOO problem.
While this line of approach (e.g., weighted-sum method) is suitable for end-to-end learning such as deep learning, there are several acknowledged weaknesses: 
(i) the choice of utility function has a large impact on the computational complexity of the resulted SOO problem \citep{bjornson2014multiobjective, bjornson2013optimal}; 
(ii) a small change in weights may results in big changes in the combined objective \citep{caballero1997algorithmic}, and vice versa, a huge different weights may produce nearly similar result \citep{coello1999comprehensive}; 
(iii) it does not work well in the case of a non-convex objective space \citep{deb2011multi}. 

One of the most common replacement for the weighted-sum method is the $\epsilon$ constraint method which is applicable to either convex or non-convex problem. Applying a more efficient MOO solver might be one of the potential extensions of this work. 

\paragraph{Discussions on the gradient descent solver.}

Inspired by \cite{sener2018multi}, in this paper we use multi-gradient descent algorithm \citep{deb2011multi} as an MOO solver 
which casts the multi-objective problem to a single-objective problem. While \cite{sener2018multi} used Frank-Wolfe 
algorithm to project the weight into the desired simplex, we use parameterization with softmax instead. Although this technique 
is much faster than Frank-Wolfe algorithm, it has some weaknesses that will be addressed in our future work.
More specifically, the GD solver with softmax parameterization cannot handle well the edge case which is the root of the dominating issue.
The snippet code \ref{snippet_code} provides a minimal example of quadratic optimization problem as similar in MGDA, where the goal is to find 
$w^* = \underset{w \in \simplex_w}{\text{argmin}}  \sum_{i=1}^5 \| w_i \text{g}_i \|_2^2$. 
The solver is the Gradient Solver with softmax parameterization. With $\text{input}_1$ where none of elements dominates others, the solver works quite reasonable 
with the weights corresponding to 4 first elements are equal and less than the last one (corresponding to bigger strength). 
With $\text{input}_2$ where $g_5 \gg g_1$, the solver still works well where $w_1=1$ corresponding to the minimal strength $g_1=0.1$. 
However, with $\text{input}_3$, the solver fails to find a good solution (which should be $w=[1,0,0,0]$ given that input). 
It is a worth noting that the main goal of this paper is to show the application of Multi-objective Optimization for generating adversarial examples and the impact of 
the Task-Oriented regularization. 
Therefore, while the issue of the gradient descent solver is well recognized, we did not take effort to try with a better solver.   

\begin{lstlisting}[language=Python, caption=Python example of the Gradient Solver with softmax parameterization]
import torch 
import torch.nn.functional as F 
import torch.optim as optim 

input_1 = [0.1, 0.1, 0.1, 0.1, 0.2] # normal case 
input_2 = [0.01, 0.1, 0.1, 0.1, 2e3] # normal case  
input_3 = [0.001, 0.002, 0.002, 0.002, 2e3] # dominating issue

init_alpha = [0.2, 0.2, 0.2, 0.2, 0.2]
g = torch.tensor(input_3)
alpha = torch.tensor(init_alpha, requires_grad=True)
opt = optim.SGD([alpha], lr=1.0)

for step in range(20):
    w = F.softmax(alpha, dim=0)
    loss = torch.square(torch.sum(w * g))
    opt.zero_grad()
    loss.backward()
    opt.step()
    print('step={}, w={}'.format(step, w.detach().numpy()))

# Result with input_1 
# step=19, w=[0.20344244 0.20344244 0.20344244 0.20344244 0.18623024]
# Result with input_2
# step=19, w=[9.999982e-01 5.582609e-07 5.582609e-07 5.582609e-07 0.]
# Result with input_3
# step=19, w=[0.28042343 0.23985887 0.23985887 0.23985887 0.]
\end{lstlisting}\label{snippet_code}

\subsection{Correlation between the Objective Loss and Attack Performance.} \label{subsec:correlation-loss-attack-performance} 

It is broadly accepted that to fool a model, a feasible approach is maximizing the objective loss (i.e., CE, KL, or CW loss), and the higher the loss, the higher the attack success rate. While it is true with the same architecture, we found that it does not hold when comparing different architectures. 
Figure \ref{fig:ens-cf10-diverse} shows the adversarial loss and the attack success rate for each model in the ENS setting. With the CW loss as the adversarial objective, it can be observed that there is a positive correlation between the loss value and the ASR, i.e., the higher the loss, the higher the ASR. 
For example, with the same adversarial examples, the adversarial loss on EfficientNet is the highest and so is ASR. 
However, there is no clear correlation observed when using CE and KL losses.
Therefore, the higher weighted loss does not directly imply a higher success rate for attacking an ensemble of different architectures. 
The MinMax method \citep{wang2021adversarial} which solely weighs the tasks' losses, therefore, does not always achieve a good performance in all the tasks. 

\begin{figure}
\begin{centering}
\subfloat[CW]{
\begin{centering}
\includegraphics[width=0.8\textwidth]{images/ens/MinMaxEns_cw.png}
\par\end{centering}
}

\subfloat[CE]{\begin{centering}
\includegraphics[width=0.8\textwidth]{images/ens/MinMaxEns_ce.png}
\par\end{centering}
}

\subfloat[KL]{
\begin{centering}
\includegraphics[width=0.8\textwidth]{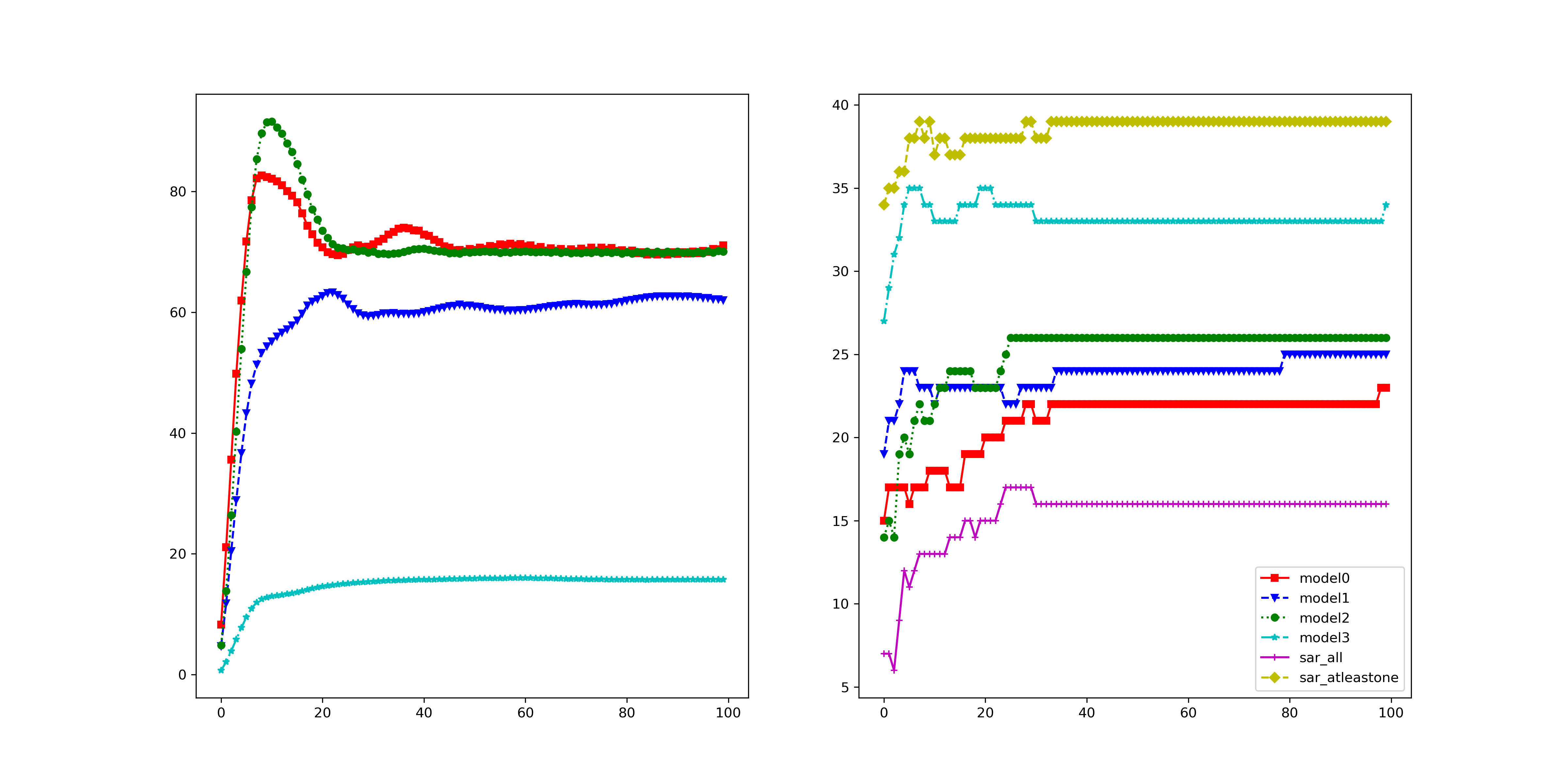}
\par\end{centering}
}

\par\end{centering}
\caption{Loss (left fig) and ASR (right fig) of each task over all attack iterations with the MinMax method. model0/1/2/3 represents R/V/G/E architecture, respectively. }
\label{fig:ens-cf10-diverse}
\end{figure}

\subsection{Conflicting between gradients in the adversarial generation task} \label{subsec:conflicting}

In multi-task learning setting, conflicting between gradient is the common issue 
to tackle with. More specifically, the gradients with respect to the (shared) model 
parameter of task $f_i$ and task $f_j$ can have a negative correlation (i.e., cosine similarity 
between $\nabla_\theta f_i (\theta, \delta) $ and $\nabla_\theta f_j (\theta, \delta)$ is negative). 
However, in the adversarial generation task, we consider the gradient 
with respect to the input (e.g., $\nabla_\delta f(\theta, \delta)$) to update the 
adversarial examples. As we explore through empirical experiments, the issue that we 
to deal with is not the gradient confliction problem but the gradient domination problem. 
These gradients with respect to the inputs can have a positive correlation but also 
have a huge difference in their strengths. 
In this specific challenge, the standard MOO which solely relies on the gradient 
strengths to calculated the weight for each task is strongly sensitive to the 
gradient domination problem and in some cases cannot lead to a good solution 
as discussed in Appendix \ref{subsec:sup-discussion-dominating-issue}

To further support our hypothesis, we would like to provide a measurement on 
the cosine similarity between gradients on different ensemble members on the ENS 
setting in Table \ref{tab:cosine-gradients}. 
Each cell (row-ith, column-jth) of the Table reports the cosine similarity between 
gradient $\nabla_{\delta}f_{i}(\delta)$ of model ith and gradient $\nabla_{\delta}f_{j}(\delta)$ 
of model jth (w.r.t. the same input $\delta$). 
It can be seen that the gradients between different architectures has the positive correlation 
instead of negative correlation. On the other hand, as shown in the last row, the gradient norm $\|\nabla_{\delta}f_{i}(\delta)\|$ 
varies widely among architectures. While this observation is in line with the widely accepted 
phenomenon about the transferability of adversarial examples, it also does support our motivation to derive 
the TA-MOO method to improve the standard MOO.

\begin{table}

	\caption{Correlation between gradients of ensemble members on ENS setting.
	Each cell (row-ith, column-jth) reports the cosine similarity (mean
	$\pm$ std) between gradient $\nabla_{\delta}f_{i}(\delta)$ of model
	ith and gradient $\nabla_{\delta}f_{j}(\delta)$ of model jth (w.r.t.
	the same input $\delta$). The last row $\|\nabla_{\delta}f_{i}(\delta)\|$
	reports the gradient norm of each model. R: ResNet18, V: VGG16, E:
	EfficientNet, G: G-GoogLeNet. \label{tab:cosine-gradients}}
	
	\begin{centering}
	\begin{tabular}{c|cccc}
	 & R & V & G & E\tabularnewline
	\hline 
	R & 1.00$\pm$0.00 & 0.34$\pm$0.15 & 0.44$\pm$0.17 & 0.35$\pm$0.19\tabularnewline
	V & 0.34$\pm$0.15 & 1.00$\pm$0.00 & 0.36$\pm$0.19 & 0.41$\pm$0.22\tabularnewline
	G & 0.44$\pm$0.17 & 0.36$\pm$0.19 & 1.00$\pm$0.00 & 0.41$\pm$0.18\tabularnewline
	E & 0.35$\pm$0.19 & 0.41$\pm$0.22 & 0.41$\pm$0.18 & 1.00$\pm$0.00\tabularnewline
	\hline 
	$\|\nabla_{\delta}f_{i}(\delta)\|$ & 7.15 $\pm$ 6.87 & 4.29 $\pm$ 4.64 & 7.35 $\pm$ 7.21 & 0.98 $\pm$ 0.72\tabularnewline
	\end{tabular}
	\par\end{centering}
\end{table}

\subsection{Discussion on the Convergence of our methods} \label{subsec:convergence}

In multi-task learning, the gradient of each task is calculated with respect to the (shared) model parameter 
(e.g., $\nabla_\theta f(\theta, \delta)$). 
Therefore, to quantify the convergence of a multi-task learning method, we can measure the gradient 
norm of the comment gradient direction to quantify the convergence of the model. 
The gradient norm is expected to be a very small value when the model reaches to the Pareto optimality points.
However, in adversarial generation problem, the gradient of each task is calculated with respect to 
the input (e.g., $\nabla_\delta f(\theta, \delta)$). 
Therefore, unlike in the multi-task learning, there is a different behavior of gradient in 
the adversarial generation task. 
To verify our hypothesis, we measure the gradient norm of the gradient over all attack iterations 
and visualize in Figure \ref{fig:norm_grad_common}. 
It can be seen that the gradient norm of all attacks tends to converge to a large value. 
It is a worth noting that we use projected gradient descent with $l_\infty$ in all attacks. 
Therefore, in each attack iteration, the amount to update is not the gradient 
$\nabla_\delta f(\theta, \delta)$ but the sign of it scaling 
with a step size $\eta_\delta$. However, there is still an interesting observation such that MOO and TA-MOO 
attack have a much lower gradient norm than other attacks.

\begin{center}
    \begin{figure}
        \begin{centering}
            \includegraphics[width=0.5\textwidth]{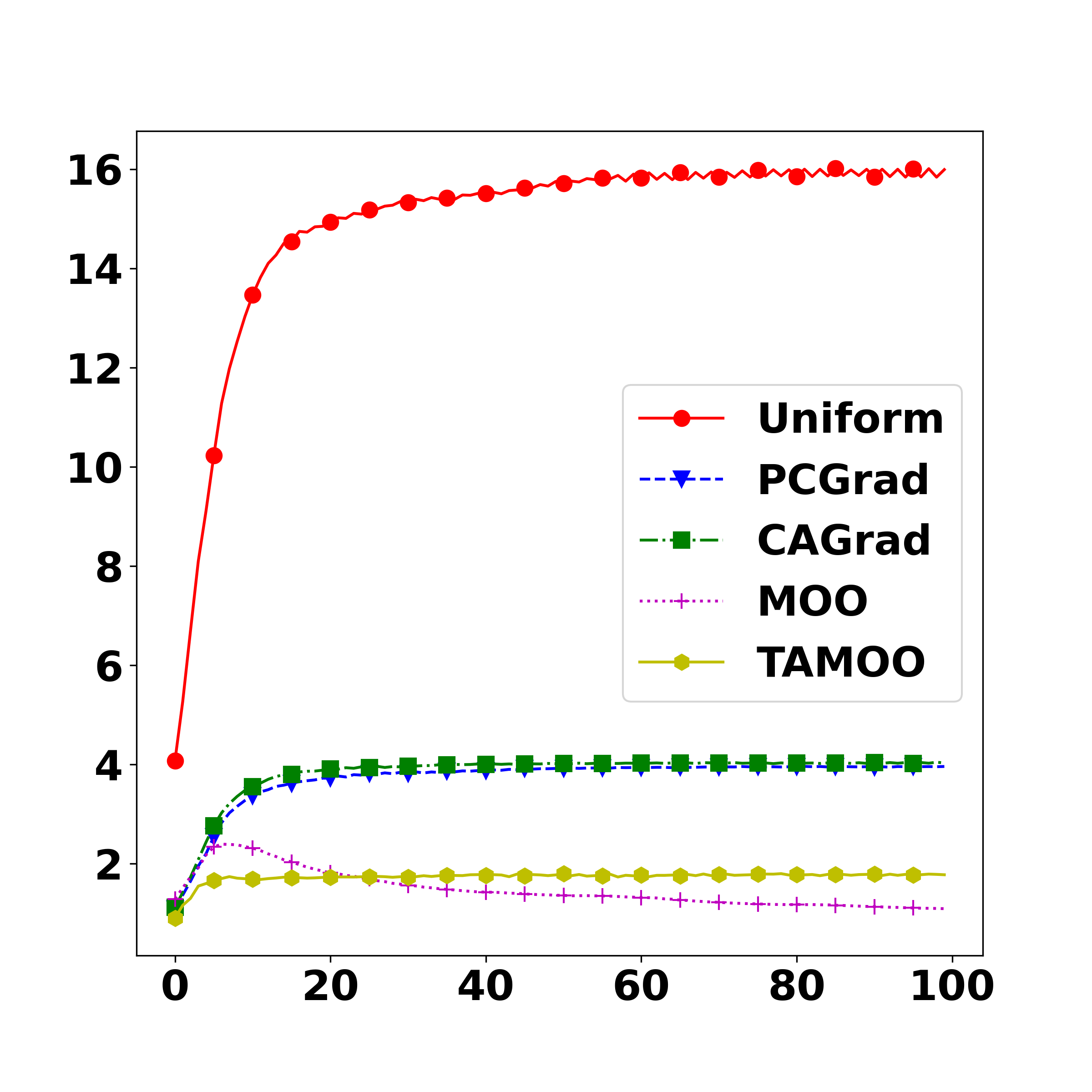}
            \caption{Norm of the gradient $\nabla_\delta f(\delta)$ over all attack iterations. Measure on the diverse set 
            of the ENS setting, with CE loss. \label{fig:norm_grad_common}}		
        \par\end{centering}
    \end{figure}
\par\end{center}

We would like to propose a simple alternative approach to quantify the convergence of our 
method in the adversarial generation setting.  
More specifically, we leverage the advantage of the adversarial generation task such that 
we can access to the label to audit whether the task is successful or not. 
Therefore, we simply measure the loss and the success attack rate over all attack iterations 
as shown in Figure \ref{fig:convergence}.

First, we would like to recall the definition of the Pareto optimality. 
Given $m$ objective function $f(\delta) \triangleq [f_1(\delta), ..., f_m(\delta)]$, the Pareto optimality 
$\delta^*$ of the multi-objective optimization $\delta^* = \underset{\delta}{\text{argmax}} f(\delta)$ 
if there is no feasible solution $\delta'$ such that is strictly better than $\delta^*$ in some tasks 
(i.e., $f_i(\delta') > f_i(\delta^*)$ for some $i$) while equally good as $\delta^*$ in all other tasks 
(i.e., $f_j(\delta') = f_j(\delta^*), j \neq i$).
Bear this definition in mind, it can be seen from the loss progress of MOO attack in Figure \ref{fig:convergence-MOO} that 
(i) from iteration 1st to around iteration 10th all the losses are increased quickly showing 
that the method optimize efficiently; 
(ii) after iteration 10th, the loss w.r.t. the EfficientNet model (i.e., model3 in the legend) continually increases 
while other losses continually decrease. Therefore, any solution after iteration 10th do not dominate 
each other indicating that the method reaches the Pareto front.

On the other hand, it can be seen from Figure \ref{fig:convergence-TA-MOO} that the loss progress of our TA-MOO is more stable. TA-MOO also can 
optimize to the optimal point efficiently as MOO does, however, after reaching the peak, the losses in 
all tasks are more stable than those in MOO. This observation indicates that the solutions 
after the peak point are also in the Pareto front but are more concentrated than those in MOO. 
It can explain the stability of the success attack rate in TA-MOO in Figure \ref{fig:convergence-TA-MOO}. 
Comparing across both MOO and TA-MOO at their last iteration shows that while the loss w.r.t. the EfficientNet model (model3) in MOO is a bit higher than that in TA-MOO, these other losses w.r.t. V/G/E models in MOO is lower than those in TA-MOO. 
This observation indicates that in term of losses, the solutions of 
MOO and TA-MOO do not dominate each other. 
However, the solution of TA-MOO is more stable and leads better final attacking performance.

\begin{figure}
	\begin{centering}
	\subfloat[MOO \label{fig:convergence-MOO}]{
	\begin{centering}
	\includegraphics[width=0.5\textwidth]{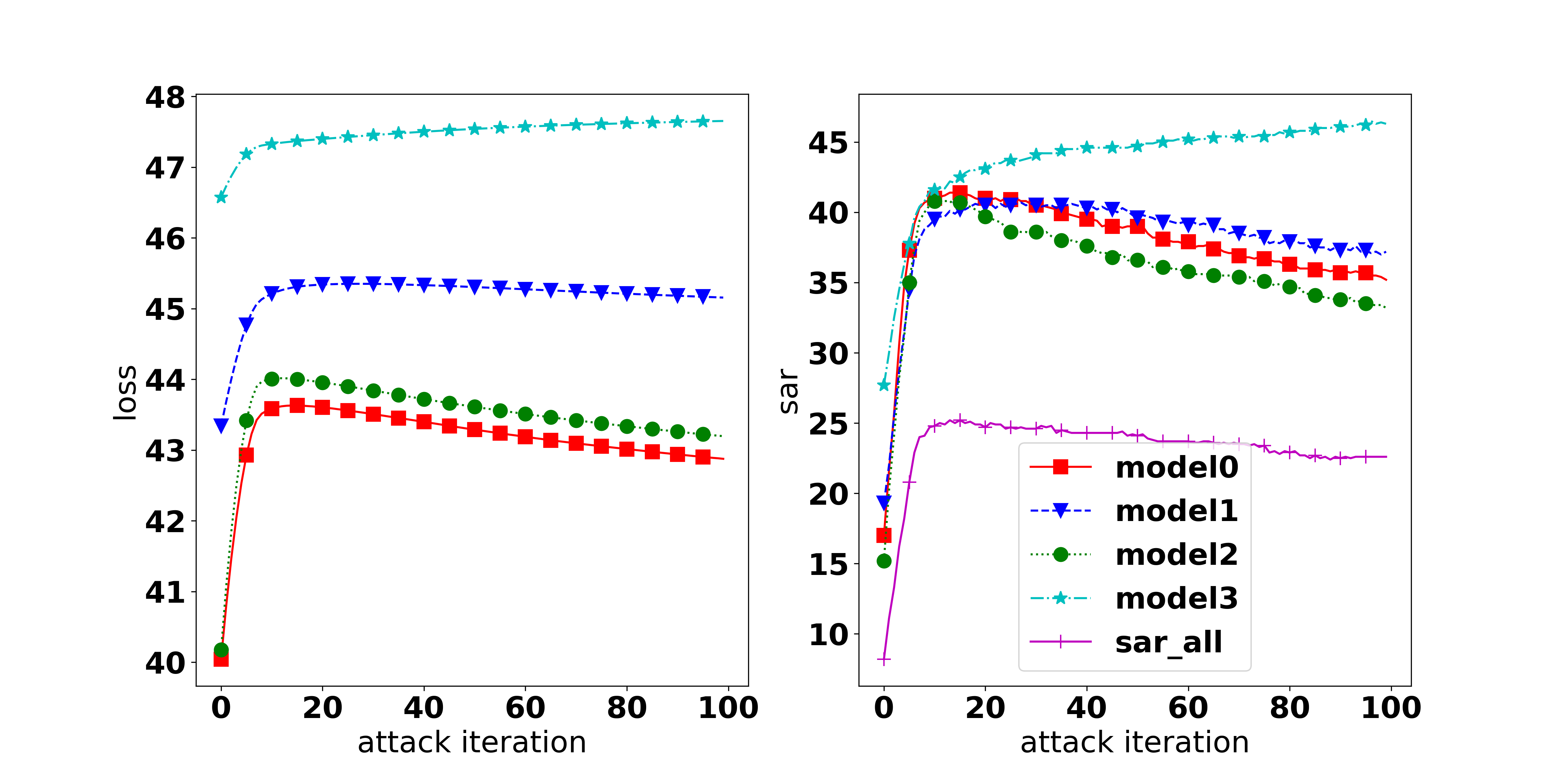}
	\par\end{centering}
	}
	\subfloat[TA-MOO \label{fig:convergence-TA-MOO}]{\begin{centering}
	\includegraphics[width=0.5\textwidth]{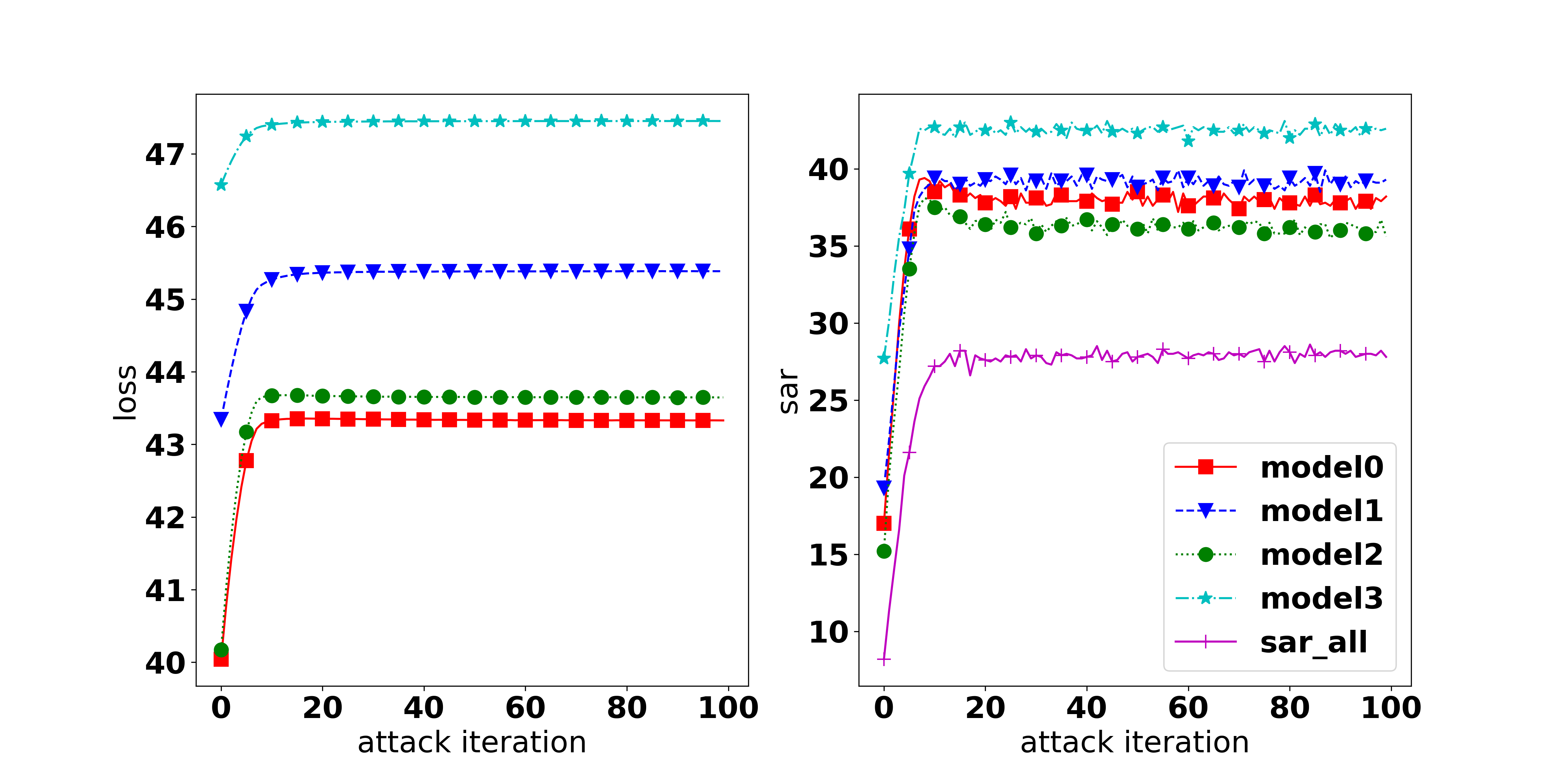}
	\par\end{centering}
	}
	
	\par\end{centering}
	\caption{Loss (left fig) and SAR (right fig) of each task over all attack iterations. model0/1/2/3 represents R/V/G/E architecture, respectively. 
	The CW loss is used as the adversaries's objective function.}
	\label{fig:convergence}
\end{figure}

\subsection{Additional Experiments with Different initializations for MOO} \label{subsec:optimal-init-moo}

In our method, the default initialization for the weight $w$ is $1/m$ equally for all tasks. 
Therefore, one raising valid concern is that \textit{Might better initialization can help to boost the 
performance?}. To answer this question, we first find the optimal initital weight by using the 
weight at the last iteration when running MOO and TA-MOO attacks with the default initialization. 
For example, as shown in Figure \ref{fig:weight-sar-MOO} for the ENS setting with diverse architectures, 
the average weight that MOO assigns for model R/V/G/E converging to 0.15/0.17/0.15/0.53 (\textit{set A}), respectively. 
The average weights' distribution learned by TA-MOO is 0.19/0.25/0.19/0.37 (\textit{set B}), respectively. 
It is a worth noting that, we consider each set of weights for each data sample separately, 
and the above weights are just the average over entire testing set (e.g., 10K sample), while the 
full statistic (mean $\pm$ std) of weights can be seen in Table 2. 
In order to make the experiment to be more comprehensive with diverse initializations, 
we use two additional sets including set C=[0.22, 0.23, 0.22, 0.33] and set D=[0.24, 0.25, 0.24, 0.27].

Given these above four weights sets A/B/C/D, 
we then init the standard MOO with one of these above sets and adjust the learning rate $\eta_w$ 
with three options {5e-3, 5e-5, 1e-8} and report results in Table \ref{tab:ENS-diverse-diff-init-4sets}. 
The complete attacking progress can be seen in Figure \ref{fig:weight-sar}. 
It can be seen from Table \ref{tab:ENS-diverse-diff-init-4sets} that better initialization 
does help to improve the performance of the standard MOO. 
The best setting is the initialization with set D and $\eta_w=\text{5e-3}$ achieves 29.53\% in A-All metric, 
a 4.37\% improvement over the default MOO initialization. 
It can be seen from the evolution of the weights in Figure \ref{fig:weight-sar-MOO-A-5e-3} that 
even initializing with the converged weights (i.e., set A) from the pre-running attack, 
the weight of each task does not stand still but converges to a different value. 
It is another different behavior in adversarial generation task compared to the multi-task learning problem. 
On the other hand, despite of the extensive tuning, the performance of MOO is still far below the TA-MOO approach, 
with the gap of 8.48\% in A-All metric.

\begin{figure}
	\begin{centering}
	\subfloat[MOO \label{fig:weight-sar-MOO}]{
	\begin{centering}
	\includegraphics[width=0.5\textwidth]{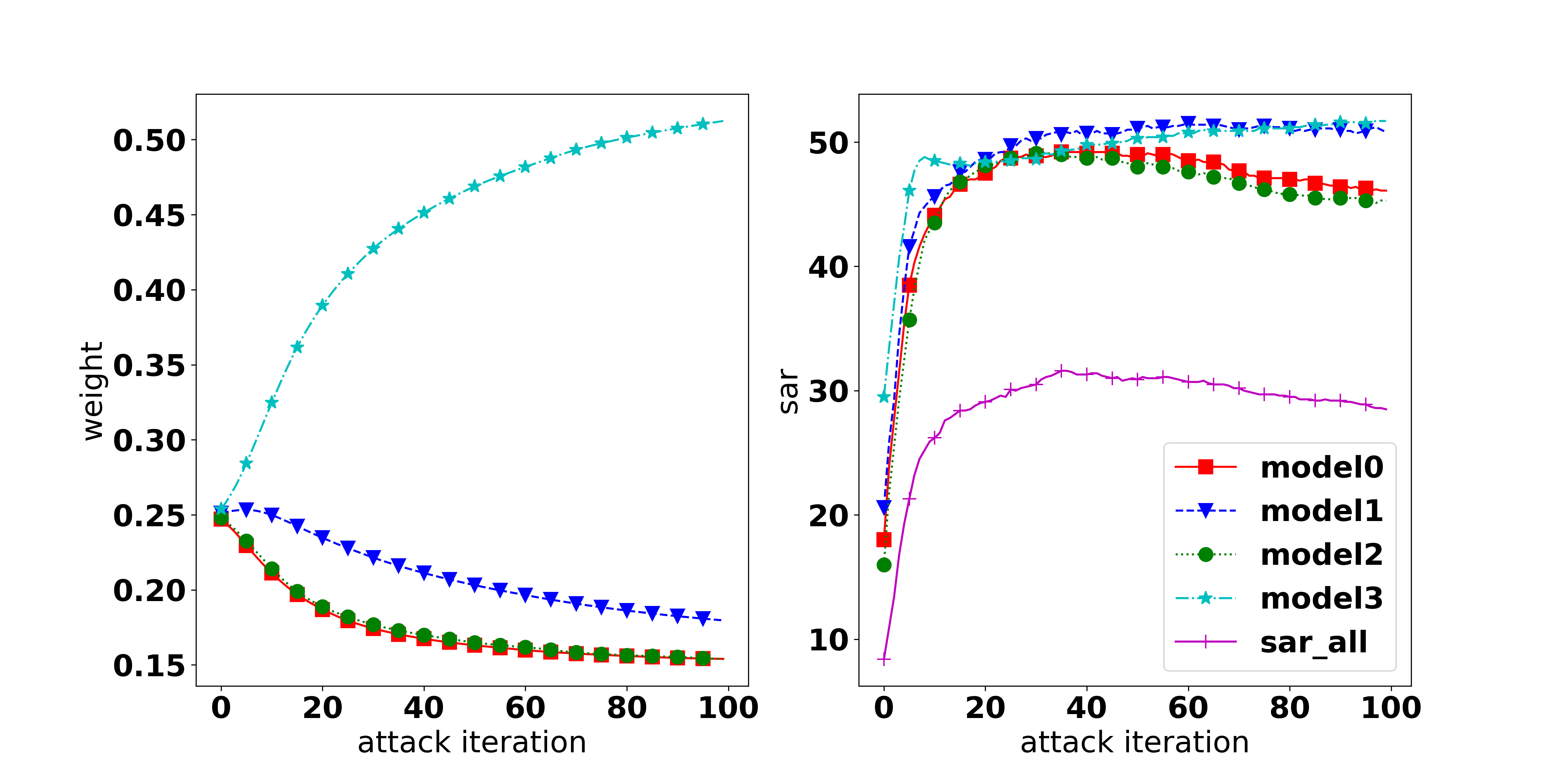}
	\par\end{centering}
	}
	\subfloat[TA-MOO \label{fig:weight-sar-TA-MOO}]{\begin{centering}
	\includegraphics[width=0.5\textwidth]{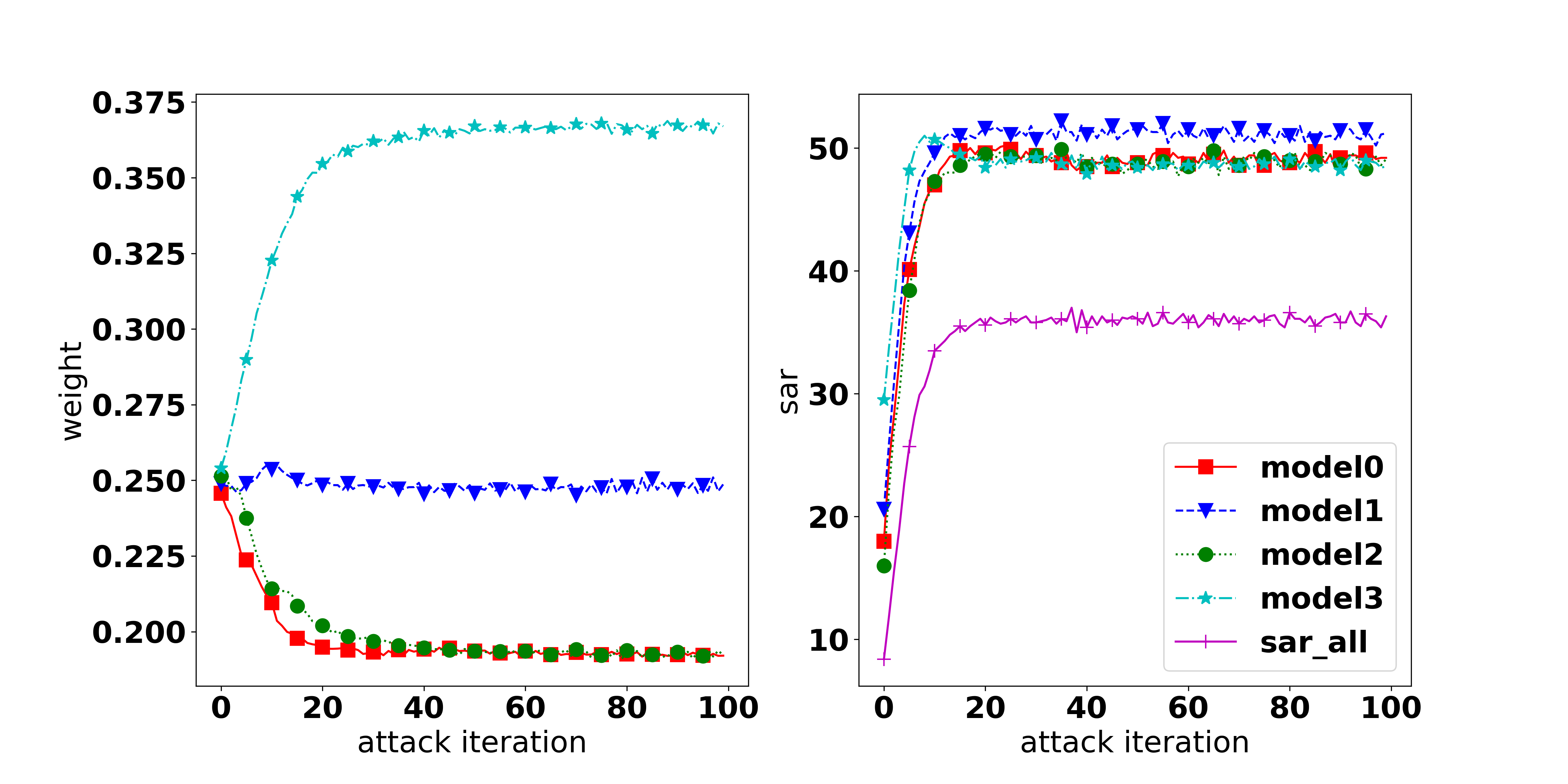}
	\par\end{centering}
	}

    \subfloat[$MOO^A$ with $\eta_w=\text{5e-3}$ \label{fig:weight-sar-MOO-A-5e-3}]{\begin{centering}
        \includegraphics[width=0.5\textwidth]{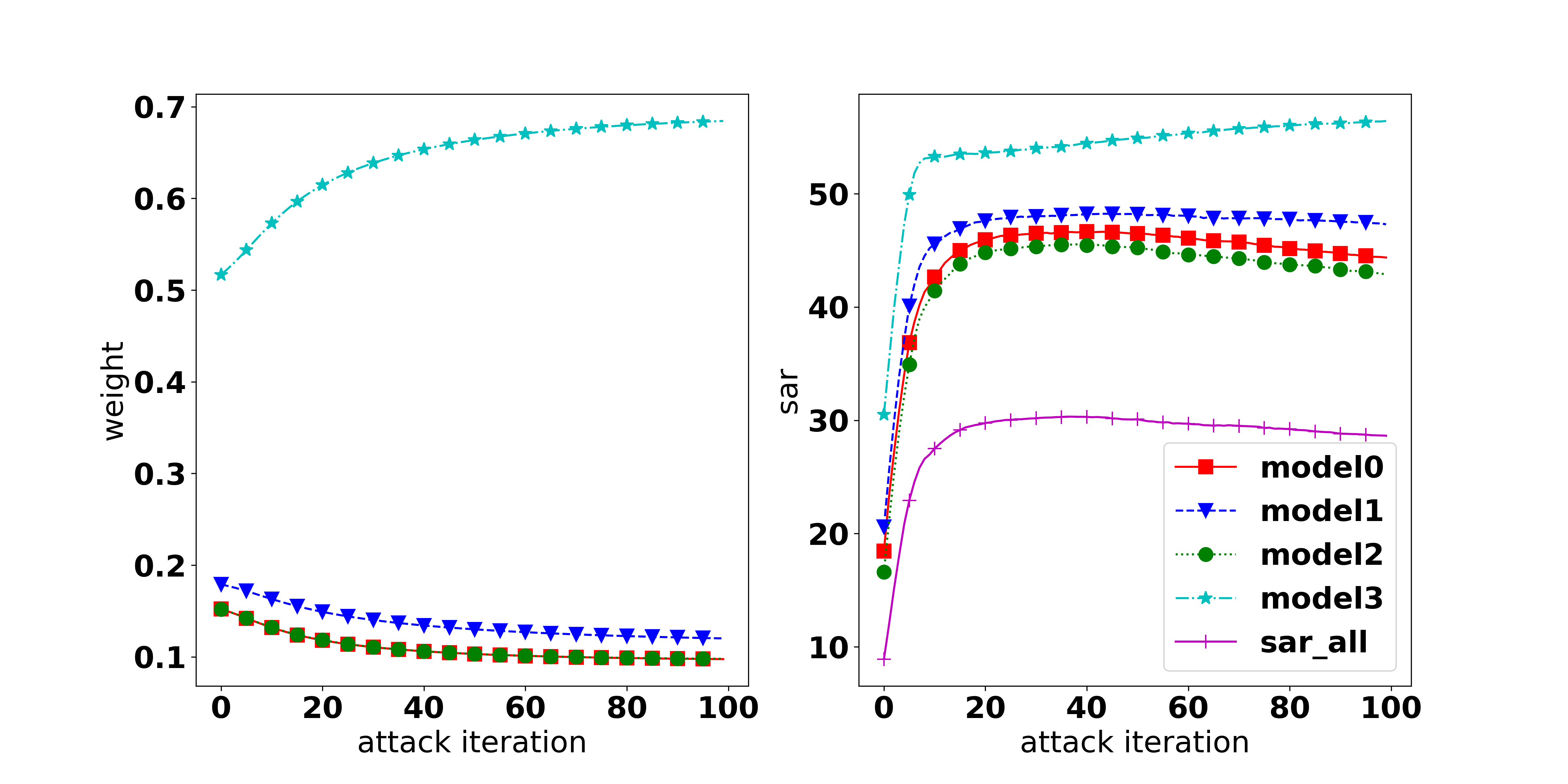}
        \par\end{centering}
    }
    \subfloat[$MOO^B$ with $\eta_w=\text{5e-3}$ \label{fig:weight-sar-MOO-B-5e-3}]{\begin{centering}
        \includegraphics[width=0.5\textwidth]{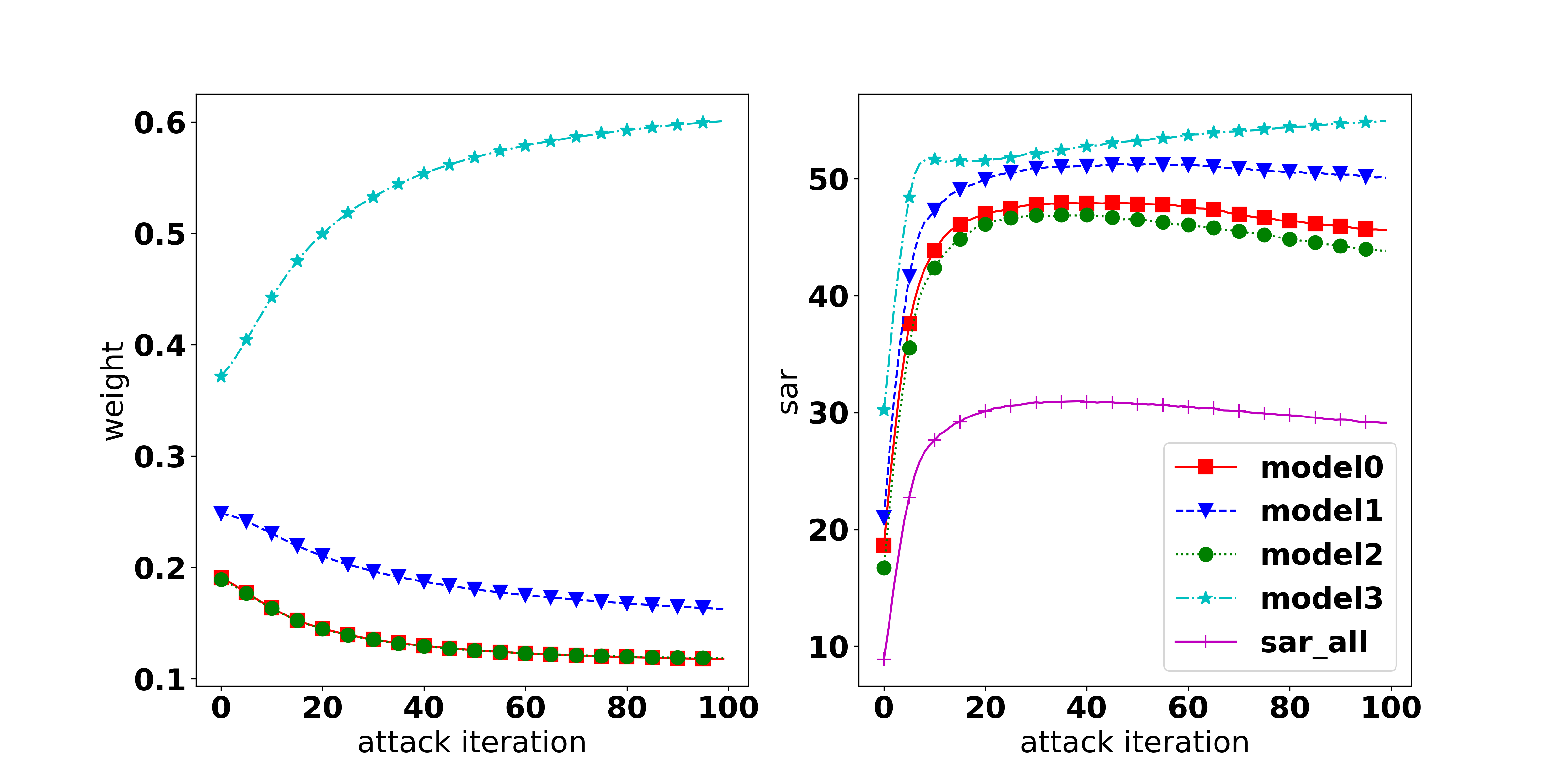}
        \par\end{centering}
    }

	\subfloat[$MOO^C$ with $\eta_w=\text{5e-3}$ \label{fig:weight-sar-MOO-C-5e-3}]{\begin{centering}
        \includegraphics[width=0.5\textwidth]{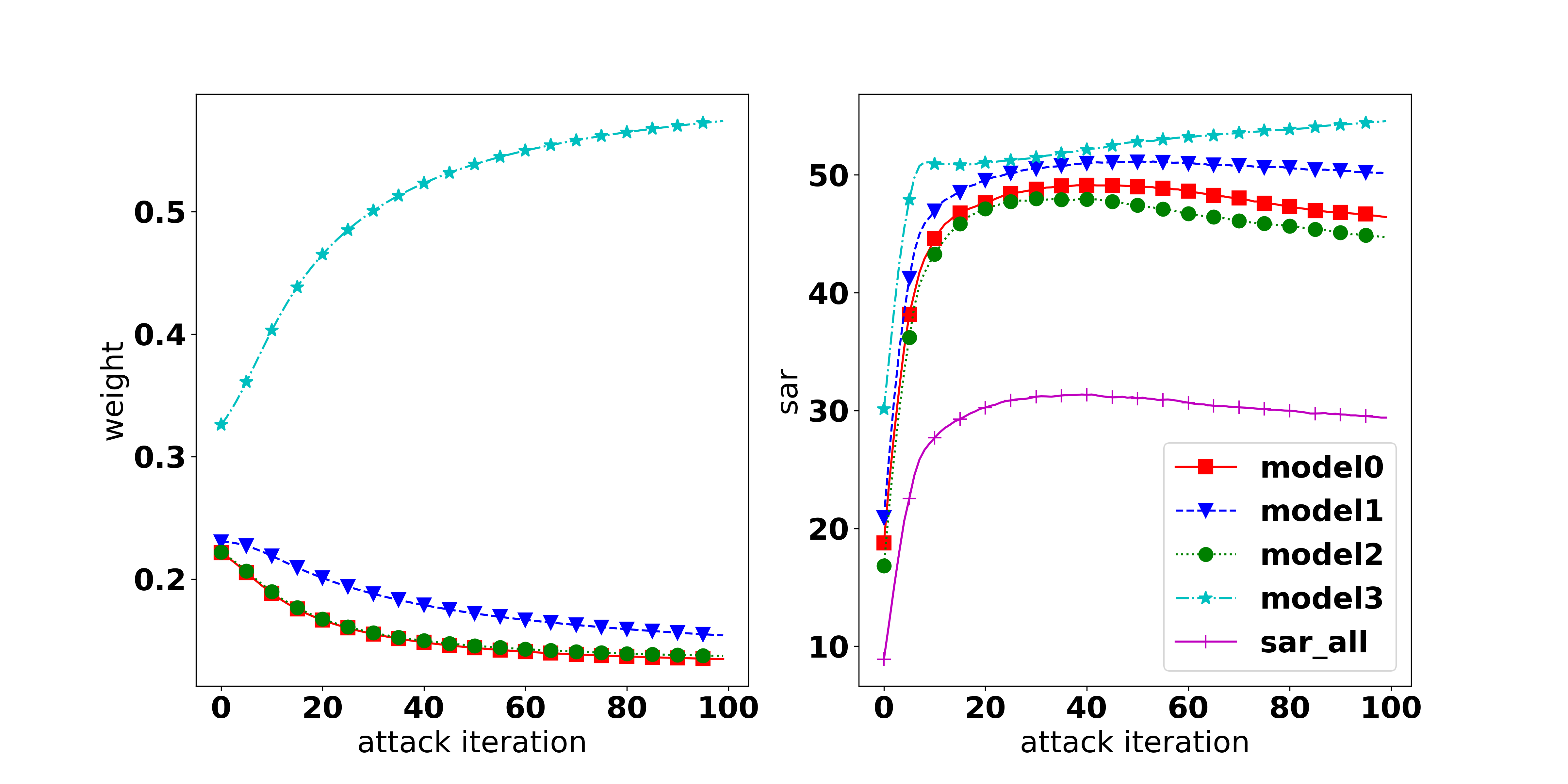}
        \par\end{centering}
    }
    \subfloat[$MOO^D$ with $\eta_w=\text{5e-3}$ \label{fig:weight-sar-MOO-D-5e-3}]{\begin{centering}
        \includegraphics[width=0.5\textwidth]{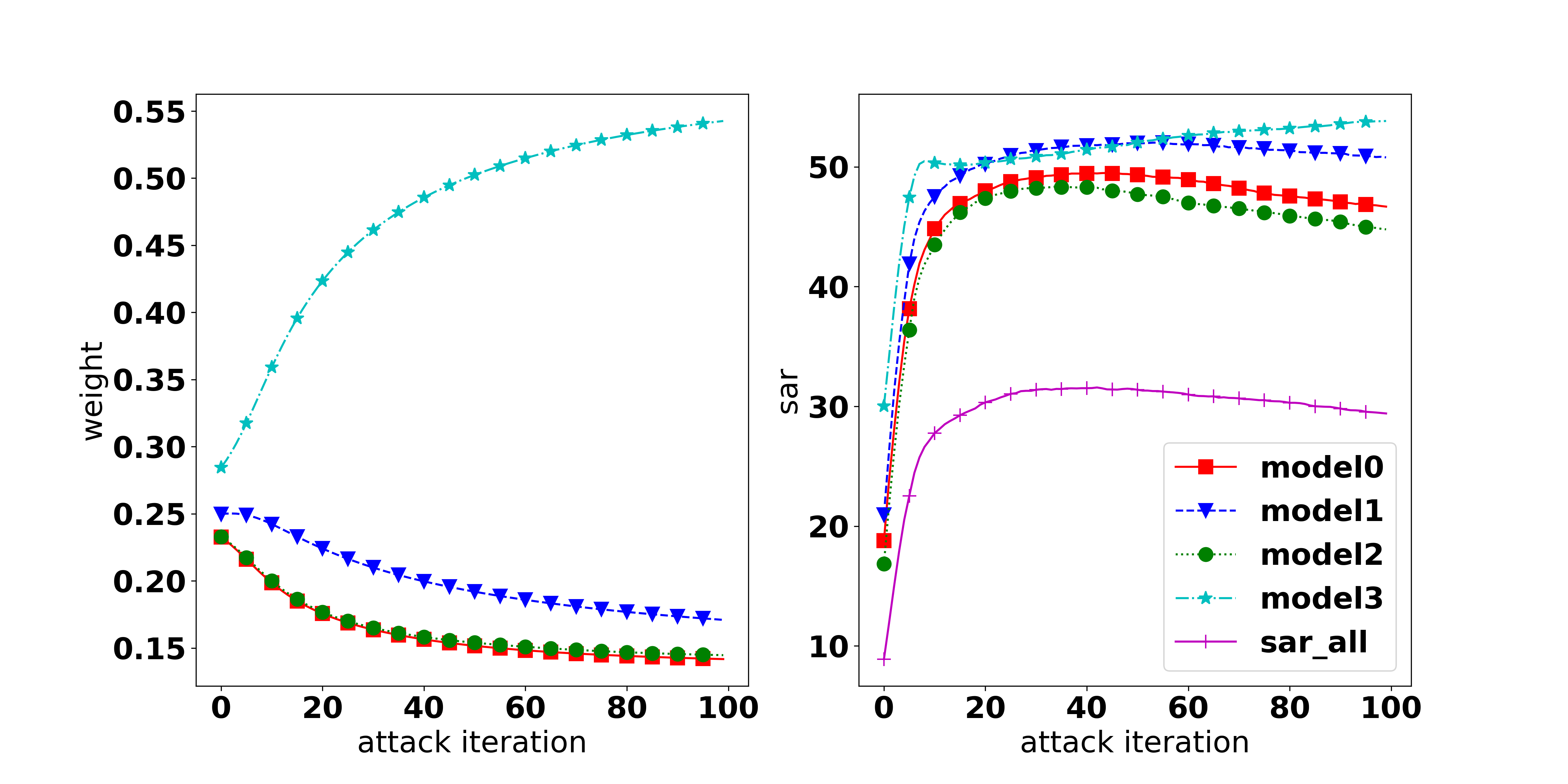}
        \par\end{centering}
    }

	\par\end{centering}
	\caption{Weight (left fig) and SAR (right fig) of each task over all attack iterations. 
    model0/1/2/3 represents R/V/G/E architecture, respectively.}
	\label{fig:weight-sar}
\end{figure}

\begin{center}
	\begin{table}
	\begin{centering}
	\caption{Attacking Ensemble model with a diverse set D=\{R-ResNet18, V-VGG16,
	G-GoogLeNet, E-EfficientNet\}. $\text{MOO}^{A/B/C/D}$ is MOO with
	initial weights from set A/B/C/D, respectively. $\eta_{w}$ denotes
	the learning rate to update for the weight $w$. \label{tab:ENS-diverse-diff-init-4sets}}
	\par\end{centering}
	\centering{}%
	\begin{tabular}{lccc}
	 & $\eta_{w}=\text{5e-3}$ & $\eta_{w}=\text{5e-5}$ & $\eta_{w}=\text{1e-8}$\tabularnewline
	\midrule 
	$\text{MOO}^{A}$ & 28.64 & 29.18 & 29.12\tabularnewline
	$\text{MOO}^{B}$ & 29.13 & 28.75 & 28.65\tabularnewline
	$\text{MOO}^{C}$ & 29.38 & 28.46 & 28.33\tabularnewline
	$\text{MOO}^{D}$ & 29.53 & 28.37 & 28.18\tabularnewline
	\midrule 
	MOO & 25.16 & - & -\tabularnewline
	TA-MOO & 38.01 & - & -\tabularnewline
	\bottomrule
	\end{tabular}
	\end{table}
\par\end{center}

\end{document}